\documentclass[11pt]{article}
\usepackage[margin=1in]{geometry}

\usepackage{lipsum}
\usepackage{amsfonts}
\usepackage{graphicx}
\usepackage{epstopdf}
\usepackage{algorithmic}
\usepackage{nicefrac}       
\usepackage{microtype}      
\usepackage{xcolor}         
\usepackage{xfrac}
\usepackage{booktabs}
\usepackage{amsmath}
\usepackage{amssymb}
\usepackage{mathtools}
\usepackage[colorlinks=true, linkcolor=blue!50!black, citecolor=green!50!black]{hyperref}
\usepackage[capitalize,noabbrev]{cleveref}
\usepackage{subcaption}
\usepackage{aligned-overset} 
\usepackage{orcidlink}
\usepackage{amsthm} 
\usepackage[square,numbers]{natbib}
\usepackage{aliascnt}
\usepackage{titlesec}

\ifpdf
  \DeclareGraphicsExtensions{.eps,.pdf,.png,.jpg}
\else
  \DeclareGraphicsExtensions{.eps}
\fi

\usepackage{enumitem}
\setlist[enumerate]{leftmargin=.5in}
\setlist[itemize]{leftmargin=.5in}

\theoremstyle{plain}
\newtheorem{theorem}{Theorem}[section]

\newaliascnt{proposition}{theorem}
\newtheorem{proposition}[proposition]{Proposition}
\aliascntresetthe{proposition}

\newaliascnt{lemma}{theorem}
\newtheorem{lemma}[lemma]{Lemma}
\aliascntresetthe{lemma}

\newaliascnt{corollary}{theorem}
\newtheorem{corollary}[corollary]{Corollary}
\aliascntresetthe{corollary}

\theoremstyle{definition}
\newaliascnt{definition}{theorem}

\aliascntresetthe{definition}

\newaliascnt{assumption}{theorem}
\newtheorem{assumption}[assumption]{Assumption}
\aliascntresetthe{assumption}

\theoremstyle{remark}
\newaliascnt{remark}{theorem}
\newtheorem{remark}[remark]{Remark}
\aliascntresetthe{remark}

\titleformat{\paragraph}[runin]
  {\normalfont\itshape}
  {}
  {0pt}
  {}
  [.]

\titleformat{\subparagraph}[runin]
  {\normalfont\itshape}
  {}
  {1pt}
  {}
  [.]

\renewenvironment{proof}[1][\proofname]{%
  \par\noindent\textbf{\textit{#1.}}
}{%
  \hfill$\square$\par
}

\newcommand{\Rbb}{\mathbb{R}}

\newcommand{\cA}{\mathcal{A}}
\newcommand{\cB}{\mathcal{B}}
\newcommand{\cC}{\mathcal{C}}

\newcommand{\cL}{\mathcal{L}}
\newcommand{\cN}{\mathcal{N}}

\newcommand{\rank}{\mathrm{rank}}

\newcommand{\T}{\mathsf{T}}

\newcommand{\diag}{\mathbf{diag}}
\newcommand{\grad}{\nabla}
\newcommand{\iid}{\overset{\mathrm{i.i.d.}}{\sim}}

\DeclareMathOperator{\vect}{\mathbf{vec}}

\DeclareMathOperator{\sr}{\mathrm{sr}}

\newcommand{\R}{\mathbb{R}}
\newcommand{\tauub}{\tau_{\mathsf{ub}}}

\DeclarePairedDelimiter{\norm}{\lVert}{\rVert}
\DeclarePairedDelimiter{\opnorm}{\lVert}{\rVert_{\mathsf{op}}}
\DeclarePairedDelimiter{\frobnorm}{\lVert}{\rVert_{\mathsf{F}}}
\DeclarePairedDelimiter{\ip}{\langle}{\rangle}
\DeclarePairedDelimiter{\abs}{\lvert}{\rvert}
\DeclarePairedDelimiter{\set}{\{}{\}}
\DeclarePairedDelimiter{\ceil}{\lceil}{\rceil}
\newcommand{\range}{\mathrm{range}}
\newcommand{\expec}[1]{\mathbb{E}\left[ {#1} \right]}
\newcommand{\prob}[1]{\mathbb{P}\left\{ {#1} \right\}}
\newcommand{\cprod}[1]{C_{\mathsf{prod}}^{(\setminus {#1})}} 
\newcommand{\call}{C_{\mathsf{prod}}}
\newcommand{\callbound}{\frac{2\eta \lambda}{m}}
\newcommand{\din}{m}
\newcommand{\dout}{d}
\newcommand{\dhid}{d_{w}}

\newcommand{\tsm}{\mathsf{small}}
\newcommand{\tlr}{\mathsf{r}}

\newcommand{\Zlr}{Z_{\mathsf{r}}}
\newcommand{\Xlr}{X_{\mathsf{r}}}
\newcommand{\Xsmall}{X_{\mathsf{small}}}
\newcommand{\Philr}{\Phi_{\mathsf{r}}}
\newcommand{\Phismall}{\Phi_{\mathsf{small}}}
\newcommand{\Klr}{\kappa(X_{\mathsf{r}})}
\newcommand{\KX}{\kappa(X)}
\newcommand{\Usmall}{U_{\mathsf{small}}}
\newcommand{\Ulr}{U_{\mathsf{r}}}
\newcommand{\Plr}{P_{\mathsf{r}}}
\newcommand{\Psmall}{P_{\mathsf{small}}}
\newcommand{\Elr}{E}
\newcommand{\myidx}{p}

\newcommand{\bmx}[1]{\begin{bmatrix} #1 \end{bmatrix}}

\newcommand{\Worc}{W_\mathrm{oracle}}

\newcommand*{\email}[1]{\href{mailto:#1}{\nolinkurl{#1}}} 

\title{Solving Inverse Problems with Deep Linear Neural Networks: Global Convergence Guarantees for Gradient Descent with Weight Decay}

\author{Hannah Laus\thanks{Contributed equally to this work. Department of Mathematics, Technical University of Munich, Munich Center for Machine Learning (MCML), Munich, Germany (\email{hannah.laus@tum.de})}
\and Suzanna Parkinson\orcidlink{0000-0003-1863-1620} \thanks{Contributed equally to this work. Committee on Computational and Applied Mathematics, University of Chicago, Chicago, IL (\email{sueparkinson@uchicago.edu})}\and Vasileios Charisopoulos\orcidlink{0000-0002-3717-0236} \thanks{Department of Electrical \& Computer Engineering, University of Washington, Seattle, WA} \and Felix Krahmer\orcidlink{0000-0002-1959-5548}\thanks{Department of Mathematics, Technical University of Munich, Munich Center for Machine Learning (MCML), Munich, Germany} \and  Rebecca Willett\orcidlink{0000-0002-8109-7582} \thanks{Departments of Statistics and Computer Science, University of Chicago, Chicago, IL}}

\usepackage{amsopn}

\usepackage{xr}

\begin{document}
\maketitle

\begin{abstract}
  	Machine learning methods are commonly used to solve inverse problems, wherein an unknown signal must be estimated from few indirect measurements generated via a known acquisition procedure. In particular, neural networks perform well empirically but have limited theoretical guarantees. In this work, we study an underdetermined linear inverse problem that admits several possible solution operators that map measurements to estimates of the target signal.
    A standard remedy (e.g., in compressed sensing) for establishing the uniqueness of the solution mapping is to assume the existence of a latent low-dimensional structure in the source signal. 
    We ask the following question: \textit{do deep linear neural networks adapt to unknown low-dimensional structure when trained by gradient descent with weight decay regularization?} We prove that mildly overparameterized deep linear networks trained in this manner converge to an approximate solution mapping that accurately solves the inverse problem while implicitly encoding latent subspace structure. 
    We show rigorously that deep linear networks trained with weight decay automatically adapt to latent subspace structure in the data under practical stepsize and weight initialization schemes. Our work highlights that regularization and overparameterization improve generalization, while overparameterization also accelerates convergence during training.
\end{abstract}

\section{Introduction}
\label{sec:submission}
Machine learning approaches, especially those based on deep neural networks, have risen to prominence for solving a broad class of inverse problems. In particular, deep learning approaches constitute the state of the art for various inverse problems arising in medical imaging (e.g., MRI or CT) \cite{lustig2008compressed, mccann2017convolutional,sriram2020end}, image denoising \cite{gonzalez2002digital, elad2023image}, and image inpainting \cite{bertalmio2000image, quan2024deep}. 
Despite its impressive performance for inverse problems, almost all the theoretical underpinnings of deep learning focus on regression or classification problems; see \cite{scarlett2022theoretical} for a summary of the theoretical results for deep neural networks for inverse problems.
On the other hand, there is a strong need for theory: understanding the behavior of deep neural networks is crucial when they are deployed in critical applications such as medical imaging.

When training a neural network to solve an underdetermined inverse problem on finitely many training samples, there are many possible solution mappings the neural network could learn. We seek to understand the role of weight decay regularization and latent low-dimensional structure in the training data in determining which solution is selected when training via gradient descent. In general, we expect that a solution mapping that leverages the latent low-dimensional structure is less likely to overfit to the training data and hence be less sensitive to small perturbations in test samples. We show an example of how weight decay can yield smaller errors across a range of noise levels in \cref{fig:wd-robustness-linear,fig:wd-robustness-nonlinear} and later show this is correlated with adaptation to low-dimensional structure in the data.

In this paper, we identify a simple yet informative setting, on the aforementioned model of a high-dimensional signal lying in an unknown low-dimensional subspace.
Indeed, \cref{fig:linear} shows that noise robustness considerably improves in the presence of an appropriate amount of $\ell_2$-regularization (also known as \emph{weight decay}),
a standard strategy in machine learning designed to promote simple parameter configurations~\cite{krogh1991simple,bos1996using}.
For the purposes of analysis, we address the case where the training data is  a collection of target signals $x^i \in \mathbb{R}^d$, each of which can be fed through the forward model $A \in  \mathbb{R}^{d \times m}$ to generate a corresponding {measurement} vector $y^i = A x^i \in \mathbb{R}^m$ for $i = 1,...,n$. More specifically, let
\begin{equation}
\begin{aligned}
    X &= \bmx{x^{1} & \dots & x^{n}} \in \Rbb^{\dout \times n}, \quad
    Y = \bmx{y^{1} & \dots & y^{n}} \in \Rbb^{\din \times n}, \\
    \text{where} \quad 
    y^{i} &= Ax^{i} \;\text{and}\; x^{i} \in \mathrm{range}(R) \;\; \text{for all $i$.}
\end{aligned}
\label{eq:dataset-measurements}
\end{equation}
Here, $A \in \Rbb^{\din \times \dout}$ is a fixed measurement operator with $\din \leq \dout$ and $R \in \Rbb^{\dout \times s}$
is an unknown matrix with orthogonal columns that span a low-dimensional subspace (i.e., $s \ll \dout$). 
Training a neural network with weight matrices $W_1,...,W_L$ to solve this inverse problem (i.e., input an measurement vector $y$ into the network and output an estimate of the corresponding target signal $x$) can be accomplished by minimizing squared error loss plus a weight decay regularizer:
\begin{align}\label{eq:l2regprob1}
  \min_{W_1, \dots, W_{L}} \sum_{i=1}^n \|f_{\{W_{\ell}\}_{\ell=1}^L}(y^i) - x^i\|^2 + \lambda \sum_{\ell=1}^L \frobnorm{W_{\ell}}^2.
\end{align}
Here, $f_{\{W_{\ell}\}_{\ell=1}^L}$ is a depth-$L$ neural network with weight matrices $W_1, \dots, W_L$ for some $L \in \mathbb{N}$. 
As long as $A$ is injective on $\range(R)$,
it is not hard to show that, for small $\lambda$, the \textit{global} minimizer of this non-convex problem yields a robust solution -- namely, it
has the following two properties:
\begin{enumerate}
    \item It is accurate (with error vanishing in the limit as $\lambda \rightarrow 0$) on the image of the signal subspace $\range(R)$ under the measurement operator $A$. That is, it accurately reconstructs signals $x \in \range(R)$ from their measurements $y \in \range(AR)$.
	\item It is zero on the orthogonal complement of the image of the signal subspace. That is, perturbations $\epsilon \in \range(AR)^\perp$ orthogonal to the image of the signal subspace are eliminated (see \ref{lem:robustsolutionofoptproblem}).
\end{enumerate}
\begin{figure*}[ht]
	\centering
    \begin{subfigure}[b]{0.48\linewidth}
        \includegraphics[width=\textwidth]{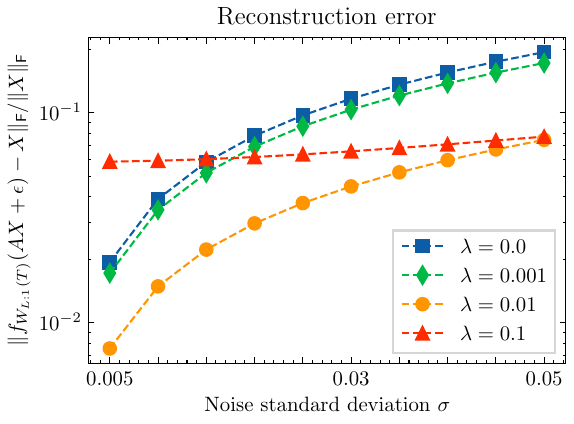}
        \caption{Linear network}
        \label{fig:wd-robustness-linear}
    \end{subfigure}
	\begin{subfigure}[b]{0.48\linewidth}
        \includegraphics[width=\textwidth]{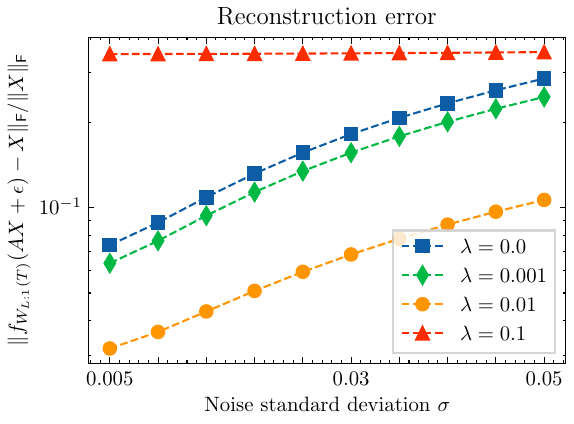}
        \caption{ReLU network with additional linear layers}
        \label{fig:wd-robustness-nonlinear}
    \end{subfigure}
	\caption{
    Effect of weight decay on robustness against Gaussian noise, $\epsilon \sim \mathcal{N},(0,\sigma^2)$ at test time.
    Experiments use signals of dimension $\dout = 256$ lying on a single subspace
    of dimension $s = 16$ (\cref{fig:wd-robustness-linear}) or a union of $k=3$ subspaces of size $s=4$ each (\cref{fig:wd-robustness-nonlinear}) and $m = 128$ measurements; all networks have $L = 5$ layers and hidden layer width $\dhid = 4096$.
    In both cases, adding sufficient $\ell_2$-regularization facilitates adaptation to the low-dimensional structure and thus significantly improves robustness, but too much regularization results in a poor fit to the data and hence poor robustness. For a detailed description of the model in~\cref{fig:wd-robustness-nonlinear}, see~\cref{sec:appendix numerical description}.}
	\label{fig:linear}
\end{figure*}

The remaining issue is that, due to the non-convexity of the problem~\eqref{eq:l2regprob1}, no algorithms with global convergence guarantees are available to date to the best of our knowledge.
As a proxy, practitioners typically apply gradient descent or stochastic gradient descent to the regularized objective~\eqref{eq:l2regprob1}. However,
whether this approach produces a good approximation to the desired global minimizer (or \emph{any} set of weight matrices that shares the aforementioned properties)
remains unclear.

In this paper, we provide an answer for fully connected deep linear neural networks $f_{\{W_{\ell}\}_{\ell=1}^L}(Y) = W_L \cdots W_1 Y$ \textit{trained by gradient descent} on the regularized objective \eqref{eq:l2regprob1}.
Our contributions can be summarized as follows.
\begin{enumerate}
	\item \label{item:reconstruction error} We show that, in the inverse problem setting of \eqref{eq:l2regprob1}, gradient descent converges to an approximate solution that accurately reconstructs 
    signals from their measurements with error vanishing in the limit as $\lambda \rightarrow 0$. (This is formalized in~\cref{eq:thm-regression-error} in~\cref{theorem:main-informal}.) 
	\item \label{item:low rank part} We show that the component of the learned inverse mapping acting on the orthogonal complement of the image of the signal subspace $\range(R)$ under the measurement operator $A$ is small after a finite number of iterations (see~\cref{eq:thm-generalization-error} in~\cref{theorem:main-informal}).
    That is, the effect of perturbations $\epsilon \in \range(AR)^\perp$ orthogonal to the image of the signal subspace is diminished.
    \item We show that while the convergence rates for \cref{item:low rank part,item:reconstruction error} in this list depend on the condition number of $X$ in general,
    they can be significantly improved if $X$ is ``close'' to a well-conditioned matrix. Further, the inverse mapping is learned more quickly in  directions corresponding to the top singular vectors of $X$ (see~\cref{cor:ill-conditioned}). 
	\item We show that optimizing the regularized objective~\eqref{eq:l2regprob1} leads to a provably more robust solution than in the non-regularized case (see~\cref{cor:robustness}).
    \item We propose a way to choose the regularization parameter $\lambda$ in order to balance the bias-variance tradeoff for robustness to Gaussian noise at test-time (see Remark~\ref{remark:choice of lambda}). 
\end{enumerate}

While this work only directly applies to a simplified linear inverse problem setting (i.e., fully-linear network with data lying on a subspace), one observes analogous benefits of weight decay for nonlinear data geometry if nonlinear neural networks are employed; see \cref{fig:wd-robustness-nonlinear}.
At the same time, existing work provides empirical evidence that more general neural networks adapt to other kinds of latent low-dimensional structure in real-world data when trained with gradient descent
\cite{huh2021low,brahma2016why,brown2022union}. 

We see our work as an important step towards understanding such nonlinear settings. Namely, 
we rigorously prove that neural networks trained using gradient descent with weight decay can automatically adapt to structure in data. 
Although we generally expect more complex low-dimensional structure in real-world data than we explore here, our theoretical results here are a stepping stone towards progress into nonlinear networks trained on data with more complex latent structure.

\subsection{Related work}
\label{sec:subsec:related-work}
\paragraph{Benefits of weight decay for generalization}
It is believed that for understanding generalization properties of neural networks, “the size of the weights is more important than the size of the network”~\cite{bartlett1996valid}.
This idea has been studied in several works~\cite{neyshabur2014search,neyshabur2015norm,wei2019regularization,daniely2019generalization,golowich2020size,parkinson2024depth},
and is especially notable in light of modern machine learning that operates in highly overparameterized regimes~\cite{zhang2021understanding}.
Regularizing the $\ell_2$-norm of the parameters (i.e., weight decay) to encourage small-norm weight matrices is common practice in neural network training and has been empirically observed to improve generalization~\cite{krogh1991simple,bos1996using,zhangthree,d2023we}.
Multiple works have addressed the properties of global minimizers of the $\ell_2$-regularized loss and of minimal-norm interpolants of the data~\cite{savarese2019infinite,ongie2019function,ma2022barron,parhi2022near,boursier2023penalising}. Several works have found that such networks adapt to low-dimensional structure~\cite{bach2017breaking,parkinson2023linear,Jacot23,kobayashi2024weight}.
In particular, minimal-norm linear deep neural networks are known to induce low-rank mappings~\cite{shang2020unified,dai2021}.

\paragraph{Convergence of gradient descent for deep linear networks}
Several works study the dynamics of gradient descent for training deep linear neural networks in general regression tasks under different assumptions.
For example, Du and Hu~\cite{du2019width} show that gradient descent starting from a random Gaussian initialization will converge at a linear rate to a global minimizer of the \emph{unregularized} loss ($\lambda = 0$) as long as the hidden layer width
scales linearly in the input dimension and depth; a closely related work by Hu et al.~\cite{HXP20} demonstrates that the hidden layer width no longer needs to scale
with the network depth when weights are initialized according to an orthogonal
scheme.
Similarly, Arora et al.~\cite{arora2019convergence} study convergence of gradient descent when (i) weight matrices
at initialization are approximately balanced and (ii) the problem instance satisfies
a ``deficiency margin'' property ruling out certain rank-deficient solutions -- a condition later removed
by the analysis of Nguegnang et al.~\cite{nguegnang2024convergence}.
On the other hand, Xu et al.~\cite{xu2023linear} show that gradient descent converges to a global
minimum for linear neural networks with two layers and mild overparameterization \emph{without}
any assumptions on the initialization; however, their proof does not readily extend to
neural networks of arbitrary depth $L$.
Shamir~\cite{shamir2019exponential} studies gradient descent on deep linear networks when the dimension and hidden width are both equal to one.
Other results include Kawaguchi et al.~\cite{kawaguchi2016deep} and Laurent and von Brecht~\cite{LvB18}, who
show that under certain assumptions, all local minima are global.
Finally, a number of works focus on gradient flow for deep linear neural networks~\cite{gidel2019implicit,JT19,eftekhari2020training,BRTW21,pesme2021implicit,JGS+21}, the continuous-time analog of gradient descent.

All the works mentioned so far study gradient descent or gradient flow
without any explicit regularization.
In contrast, Arora et al.~\cite{ACH18} study the $\ell_2$-regularized objective for deep linear networks but do not focus on the effects of the regularization in the analysis. Instead, they show that depth has a preconditioning effect that accelerates convergence. However, their analysis for the discrete-time setting relies on
near-zero initialization and small stepsizes. Lewkowycz and Gur-Ari~\cite{lewkowycz2020training} study the regularization effect of
weight decay for \emph{infinitely wide} neural networks with positively homogeneous activations, finding that model performance peaks at approximately
$\lambda^{-1}$ iterations -- a finding also supported by our analysis (see~\cref{theorem:main-informal} and Remark~\ref{remark:horizon-length}). However, their theoretical analysis only covers
gradient flow updates. Yaras et al.~\cite{YWH+23,yaras2024compressible}, inspired
by the LoRA technique~\cite{hu2021lora}, show that gradient descent updates of deep linear networks traverse a ``small'' subspace when
the input data lies on a low-dimensional structure. Unfortunately, their proofs
(i) rely on an ``orthogonal initialization'' scheme and (ii) do not provide any
guarantees on the accuracy of the solution learned by gradient descent.
Finally, Wang and Jacot~\cite{wang2024implicit} study the implicit bias of (stochastic) gradient
descent for deep linear networks. They show that SGD with sufficiently small
weight decay initially converges to a solution that overestimates the rank of
the true solution mapping, but SGD will find a low-rank solution
with positive probability given a sufficiently large number of epochs (proportional to $O(\eta^{-1} \lambda^{-1})$). However, their work does not assess the impact of this phenomenon on generalization performance. 

\subsection{Notation and basic constructions}
\label{sec:subsec:notation}
We briefly introduce the notation used in the paper.
We write $\frobnorm{A} := \sqrt{\mathsf{Tr}(A^{\T} A)}$ for the \emph{Frobenius norm} of a matrix $A \in \Rbb^{m \times d}$
and $\opnorm{A} := \sup_{x: \norm{x} = 1} \norm{Ax}$ for its \emph{spectral norm}. Moreover,
we let $A^{\dag}$ denote the \emph{Moore-Penrose} pseudoinverse of $A$.
We write $\sigma_{\min}(A)$ for the smallest \emph{nonzero}
singular value of $A$,
$\kappa(A) := \opnorm{A} \opnorm{A^{\dag}}$ for its
\emph{condition number} and
$\sr(A) := \nicefrac{\frobnorm{A}^2}{\opnorm{A}^2}$ for the so-called \emph{stable rank} of A.
The vectorization operator $\vect$ transforms
a matrix $A \in \Rbb^{m \times d}$ into a vector $\mathbf{vec}(A) \in \Rbb^{md}$
in column-major order.
We let $A \otimes B$ denote the \emph{Kronecker product} between matrices $A$ and $B$; for
compatible $A$, $X$ and $B$, the Kronecker product and $\vect$ operator satisfy
$
	\vect(AXB^{\T}) = (B \otimes A) \cdot \vect(X). \label{eq:kronecker-product-and-vectorization}
$
Given a projection matrix $P$ (i.e., a symmetric, idempotent matrix),
we write $P^{\perp} := I - P$ for the projection matrix onto the orthogonal complement of $\range(P)$.
Finally, given scalars $A$ and $B$, we
write $A \lesssim B$ to indicate that there is a dimension-independent constant
$c > 0$ such that $A \leq c B$; the precise value of $c$ may change between occurrences.

\section{Main result}
\label{sec:main-result}
In this section, we present our main result. Recall that we are interested in solving~\eqref{eq:l2regprob1},
for the special case where $f_{\{W_{\ell}\}_{\ell=1}^L}$ is a deep linear network, using gradient
descent~\eqref{eq:gradient-descent}: given a step-size parameter
$\eta > 0$, data $(X, Y)$, initial weights $\set{W_{\ell}(0)}_{\ell = 1}^{L}$ and a number of steps $T$, we iterate
\begin{equation}
    \tag{\texttt{GD}}
    W_{\ell}(t+1) := W_{\ell}(t) - \eta 
    \grad \mathcal{L}(\set{W_{\ell}(t)}_{\ell = 1}^{L}; (X, Y)),
    \quad \text{for $t = 0, 1, \dots, T-1$.}
    \label{eq:gradient-descent}
\end{equation}
Concretely, we want to minimize the following loss function:
\begin{equation}
	\cL(W_{L:1}; (X, Y)) :=
	\frac{1}{2} \frobnorm{W_{L} \cdots W_{1} Y - X}^2 + \frac{\lambda}{2} \sum_{\ell=1}^{L} \frobnorm{W_{\ell}}^2,
	\label{eq:loss-function}
\end{equation}
where  $W_{j:i}:= \prod_{\ell = j}^{i} W_{\ell}$ is the partial product of weight matrices. 
We consider weight matrices of the following sizes:
\begin{equation}
    W_{1} \in \Rbb^{\dhid \times \din}, \;\;
    W_{2}, \dots, W_{L-1} \in \Rbb^{\dhid \times \dhid}, \;\; \text{and} \;\;
    W_{L} \in \Rbb^{\dout \times \dhid}.
\end{equation}
Having fixed the architecture, we introduce two mild assumptions under
which our results hold.
\begin{assumption}[Restricted Isometry Property]
	\label{assumption:rip}
	The measurement matrix $A$ from~\eqref{eq:dataset-measurements}
	satisfies the following:
	 there exists $\delta > 0$ such that,
	for all vectors $x \in \range(R)$,
	\begin{equation}
		(1 - \delta) \norm{x}^2 \leq \norm{Ax}^2 \leq (1 + \delta) \norm{x}^2.
		\label{eq:rip}
	\end{equation}
     In particular, we assume $\delta=\frac{1}{10}$ for ease of presentation in establishing our bounds;
     one could choose any other small value. 
\end{assumption}
Assumption \ref{assumption:rip} is standard
in the compressed sensing literature~\cite{foucart2013invitation}, as it is
a sufficient condition that enables the solution of high-dimensional linear inverse
problems from few measurements. In our context,
Assumption \ref{assumption:rip} essentially states that the
training data has been sampled from inverse problems that are identifiable. Other assumptions on $A$ that yield similar bounds, e.g., on the largest and smallest singular values of $AR$, could also be used.
Note that for Assumption \ref{assumption:rip} to hold, we must have that $m \ge s$.
Our next assumption relates to the network initialization:
\begin{assumption}[Initialization]
	\label{assumption:initialization}
	The weight matrices $W_1, \dots, W_{L}$ at initialization are sampled from a scaled
	(``fan-in'') normal distribution:
	\begin{equation}
		[W_{\ell}(0)]_{ij} \iid  \cN\left(0, \frac{1}{d_{\ell}}\right) \text{ with } d_{\ell} = \begin{cases}
			\din, & \ell = 1,           \\
			\dhid,           & \ell = 2, \dots, L.
		\end{cases}
		\label{eq:fan-in-initialization}
	\end{equation}
\end{assumption}

The initialization scheme in Assumption \ref{assumption:initialization} enjoys widespread
adoption\footnote{It was introduced in~\cite{HZRS15} as a heuristic for stabilizing neural network training
and corresponds to the \texttt{torch.nn.init.kaiming\_normal\_} initialization method in PyTorch, which differs from the default Pytorch initialization scheme for MLPs only in its choice of Gaussian instead of uniform random numbers (see \texttt{torch.nn.init.kaiming\_uniform\_}) \cite{paszke2017automatic}.} \cite{goodfellow2016deep}, whereas initialization assumptions made in prior work are often nonstandard or bespoke~\cite{arora2019convergence,nguegnang2024convergence,ACH18,yaras2024compressible}. 
For large-width networks, Gaussian universality implies that the behavior of the initialization in Assumption \ref{assumption:initialization} is essentially equivalent to PyTorch's default \cite{Martinsson_Tropp_2020}. 

To motivate our main theorem, we now describe the reference estimator against which we compare the trained neural network. 
Consider an ``oracle" operator defined as the minimum-Frobenius-norm linear operator that interpolates the training data:
\begin{equation}\label{eq:oracle}
    \Worc = \arg\min_W \|W\|_F^2 \text{ subject to } WY=X.
\end{equation}
The unique solution to the minimization problem in \eqref{eq:oracle} is  given by $\Worc = XY^\dagger = R(AR)^\dagger$ (see \cref{lem:orcal-pseudo-inverse}), and one can show that $\Worc$ is robust to Gaussian noise in the measurements at test time, see \cref{lem:oracle robustness} for the exact formulation. 
Note that in our setting, since $X$ is rank-$s$, the oracle operator $\Worc$ is also rank-$s$. This suggests that an inverse mapping that takes advantage of the low-dimensional structure in the data will be more performant than one that does not.

Although knowledge of the structure in our data would suggest the use of the oracle estimator $\Worc$, practitioners routinely train neural networks to solve inverse problems without \textit{a priori} knowledge of what structure their data might exhibit.
This motivates us to consider whether training a neural network via gradient descent will result in a learned mapping that is close to an oracle solution.
Because we can precisely describe the oracle operator $\Worc$ in our setting, we can analyze the distance between the oracle and the learned neural network mapping. Indeed, we have the following lemma.
\begin{lemma}\label{lem:dist from oracle}
Given a linear operator $W \in \R^{d \times m}$, the operator norm distance between $W$ and the oracle mapping $\Worc \in \R^{d \times m}$ defined in \eqref{eq:oracle} will satisfy
    \begin{equation}\label{eq:oracle dist}
        \opnorm{ W - \Worc} \lesssim \frac{\frobnorm{ W Y-X}}{\sigma_{\min}(X)}
        + \opnorm{ W P_{\range(Y)}^{\perp}}
    \end{equation}
where $P_{\range(Y)}$ denotes the orthogonal projection onto the range of $Y$.
\end{lemma}
\begin{proof}
    Using the characterization of $\Worc$ as $\Worc = XY^\dagger$ proven in \cref{lem:orcal-pseudo-inverse}, we apply the triangle inequality to see that
    \begin{align*}
        \opnorm{ W-XY^{\dag}} &=  \opnorm{ WYY^{\dag}-XY^{\dag}+  W(I-YY^{\dag})}\\
        & \leq \frobnorm{ WY-X} \opnorm{Y^{\dag}}+ \opnorm{ WP_{\range(Y)}^{\perp}} \\
        & \lesssim \frac{\frobnorm{ WY-X}}{\sigma_{\min}(X)}
        + \opnorm{ WP_{\range(Y)}^{\perp}}
    \end{align*}
    where the final inequality comes from the fact that Assumption \ref{assumption:rip} implies that 
    $\opnorm{Y^{\dag}} = \sfrac{1}{\sigma_{\min}(Y)} \lesssim \sfrac{1}{\sigma_{\min}(X)}$ (see \cref{lem:RIP bounds on sv} for details).
\end{proof}

We emphasize that, once the training data is generated, the forward model $A$ does not influence the neural network architecture, training loss, or any other aspect of the training procedure we consider.

We now present our main result, which provides bounds on the terms $\frobnorm{ W Y-X}$ and $\opnorm{ W P_{\range(Y)}^{\perp}}$ from \eqref{eq:oracle dist}. Therefore, this result allows us to bound the distance between the learned neural network mapping and the oracle operator $\Worc$. 
In \cref{sec:mainresultandproof-appendix}, we provide a generalization (\cref{thm:mainresult-formal}) that encompasses both \cref{theorem:main-informal} and \cref{cor:ill-conditioned} (below). As such, the full proof of \cref{theorem:main-informal} appears in \cref{sec:subsec:simplifying num samples,sec:subsec:Lemmas used for the proof,sec:subsec:Properties at initialization,sec:subsec:Step 1: Rapid early convergence,sec:subsec:Step 2: he error stays small,sec:subsec:Step 3: Convergence off the subspace}.
\begin{theorem}
	\label{theorem:main-informal}
	Let Assumptions~\ref{assumption:rip} and \ref{assumption:initialization} hold and set
	the step size $\eta$ and weight decay parameter $\lambda$ as
	\begin{equation}
		\eta := \frac{\din}{L \cdot \sigma_{\max}^2(X)}, \quad
		\lambda := \gamma \sigma_{\min}^2(X) \sqrt{\frac{\din}{\dout}} \text{ with }  0 < \gamma \leq     
        \min\left\{1,10^{-7}\cdot\frac{\sqrt{d}}{L\sqrt{m}} \right\},
		\label{eq:informal-theorem-stepsize-and-wd}
	\end{equation}
	where $\gamma$ is a user-specified accuracy parameter. Moreover, define
	the times
	\begin{subequations}\label{eq:times def}
		\begin{align}
			\tau & = \inf\set*{
				t \in \mathbb{N} \mid
				\frobnorm{W_{L:1}(t)Y - X} \leq
				\frac{C_{1} \gamma \frobnorm{X}}{L}
			},             \label{eq:tau def}                                                            \\
			T    & = \ceil*{\frac{2 L \kappa^2(X) \log(\dhid)}{\gamma} \sqrt{\frac{\dout}{\din}}} \label{eq:T def}
		\end{align}
	\end{subequations}
    where $C_{1}>0$ is a universal constant.
	Then, as long as the hidden layer width satisfies
	\begin{equation}\label{eq:width assumption}
		\dhid = \widetilde \Omega\left( \dout \cdot \sr(X) \cdot \mathrm{poly}(L, \kappa(X)) \right),   
	\end{equation}
	gradient descent~\eqref{eq:gradient-descent}
	produces iterates that satisfy
	\begin{align}
		\frobnorm{W_{L:1}(t+1)Y - X}
		                                          & \leq
		\begin{cases}
			\left(1 - \frac{1}{32 \kappa^2(X)}\right) \frobnorm{W_{L:1}(t)Y - X}, & t < \tau;         \\
			C_{2} \gamma \frobnorm{X},                                           & \tau \le t \leq T
		\end{cases} \label{eq:thm-regression-error} \\
		\opnorm{W_{L:1}(T) P_{\range(Y)}^{\perp}} & \leq \left( \frac{1}{\dhid} \right)^{C_{3}},
		\label{eq:thm-generalization-error}
	\end{align}
	with probability at least $1-exp\left(-\Omega(\dout)\right)$ over the random initialization where $C_{2}, C_{3} > 0$ are universal constants.
\end{theorem}
\cref{theorem:main-informal} suggests that the left hand side of \eqref{eq:thm-regression-error}, which we call the ``reconstruction error," is proportional to the weight decay parameter $\lambda$ (and can be made arbitrarily small by tuning $\lambda$, cf.~\eqref{eq:thm-regression-error}), while the component of the learned mapping acting on the orthogonal complement of the signal subspace can be controlled by increasing the hidden width $\dhid$~\eqref{eq:thm-generalization-error}; this is essential for robustness to noisy test data (see~\cref{sec:subsec:robustness}). Additionally, \cref{theorem:main-informal} highlights two distinct phases of convergence for gradient descent: during the first $\tau$ iterations, the reconstruction error converges linearly up to a threshold specified in~\eqref{eq:tau def} as $O(\sfrac{\gamma\frobnorm{X}}{L})$. Upon reaching this threshold, the behavior changes: while the reconstruction error can increase mildly from iteration $\tau$ to $T$, the ``off-subspace'' component shrinks to the level shown in~\eqref{eq:thm-generalization-error}. The time horizon $T$ required to achieve this behavior grows only logarithmically with the hidden layer width, but is highly sensitive to the targeted reconstruction accuracy --- and therefore the regularization parameter $\lambda$.

\begin{remark}
    \label{remark:horizon-length}
    Notice that  $T = O(\nicefrac{1}{\eta \lambda})$; this is consistent with the
	results in prior work~\cite{lewkowycz2020training,wang2024implicit}. Lewkowycz and Gur-Ari \cite{lewkowycz2020training} observe empirically that SGD without momentum
	attains maximum performance at roughly $O(\nicefrac{1}{\eta \lambda})$ iterations,
	while Wang and Jacot~\cite[Theorem B.2]{wang2024implicit}
	suggests that stochastic gradient descent requires a similar number of iterations to find a low-rank solution --- albeit one that might be a poor data fit.
\end{remark}

\begin{remark}
	\label{remark:small-eta}
	\cref{theorem:main-informal} remains valid when the step size $\eta$ is chosen to be smaller than the value specified in~\eqref{eq:informal-theorem-stepsize-and-wd}, albeit at the expense of an increased number of iterations $T$. A generalization of the theorem to smaller stepsizes $\eta$ can be found in \cref{thm:mainresult-formal}.
\end{remark}

It is natural to ask whether the two phases of convergence outlined in~\cref{theorem:main-informal} are observed in practice or rather an artifact of our proof technique. To that end, we perform the following numerical experiment: fixing the number of layers $L = 3$, weight decay $\lambda = 10^{-3}$
and problem dimensions $(\din, \dout, s) = (128, 256, 4)$, we plot the (normalized) reconstruction error $\frobnorm{W_{L:1}Y - X} / \frobnorm{X}$ as well as the ``off-subspace'' error $\opnorm{W_{L:1}P_{\perp}}$ of gradient descent with step-size $\eta = k \cdot \nicefrac{m}{L \sigma_{\max}^2(X)}$, for a multiplicative pre-factor $k$ ranging between $0.01$ and $5$ (larger prefactors cause the iterates to diverge).
The results, shown in~\cref{fig:stepsize-sweep}, suggest that the two-phase behavior is indeed realistic; for all $k$, the reconstruction error rapidly declines to similar levels before eventually rebounding to a level proportional to $\lambda$, while the off-subspace error converges at a linear, albeit slower, rate towards $0$. 
As part of the proof of \cref{theorem:main-informal}, we bound the length of the first phase as
\begin{equation}
    \label{eq:tau-ub}
    \tau \le \tauub := \frac{64 \din}{\eta L \sigma_{\min}^2(\Xlr)}
        \log\left(
          \frac{
          L \sigma_{\min}^2(\Xlr)}{\lambda}  
        \right)
 \end{equation}
(see \cref{corollary:length-of-step-1}). 
In \cref{fig:stepsize-sweep-a}, we indicate the bound $\tauub$ for the case of $k=1$ using a dashed vertical line; our theory correctly predicts that the reconstruction error will increase mildly by iteration $\tauub$.
Evidently, our estimate is off by a constant factor from the empirically observed timing of the rebound.
Note that the reconstruction error for the ``slowest'' configuration (i.e., $k = 0.01$) rebounds at a later iteration and
thus appears to be monotonically decreasing in the left panel, which only spans the first $10,000$ iterations to visually distinguish between different configurations.
In \cref{fig:stepsize-sweep-b}, we see that the off-subspace error decays much more slowly than the reconstruction error, but continues to decay with more iterations, as predicted by our theory.
We refer the reader to~\cref{sec:appendix numerical description}
for a comprehensive description of the experimental setup and additional
experiments examining the impact of other parameters.

\begin{figure*}[ht]
	\centering
    \begin{subfigure}[b]{0.48\linewidth}
        \includegraphics[width=\textwidth]{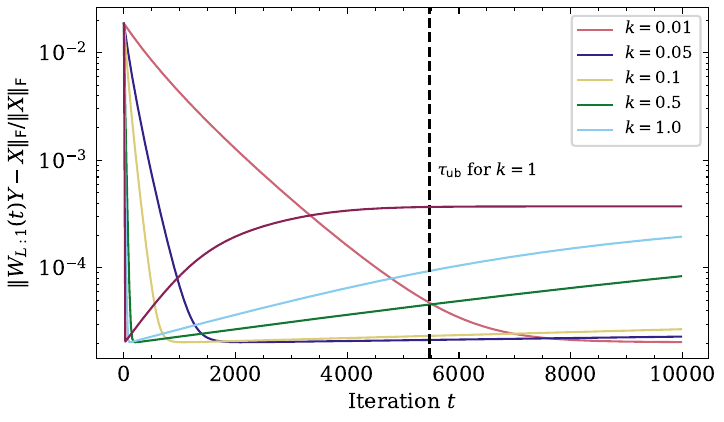}
        \caption{Reconstruction error}
        \label{fig:stepsize-sweep-a}
    \end{subfigure}
	\begin{subfigure}[b]{0.48\linewidth}
        \includegraphics[width=\textwidth]{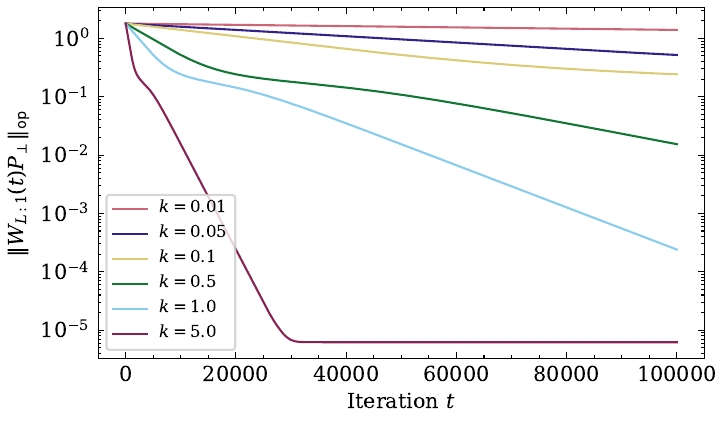}
        \caption{Off-subspace error}
        \label{fig:stepsize-sweep-b}
    \end{subfigure}
	\caption{
     Reconstruction and off-subspace errors for gradient descent across different step sizes $\eta := k \cdot \nicefrac{\din}{L \cdot \sigma_{\max}^2(X)}$. We observe iterates diverging when the multiplicative pre-factor $k$ is larger than $5$. Note that \cref{fig:stepsize-sweep-a} spans the first $10,000$ iterations. The vertical dashed line in \cref{fig:stepsize-sweep-a} indicates the value of $\tauub$ as prescribed by \eqref{eq:tau-ub} for the case of $k=1$; our theory correctly predicts that the reconstruction error will rebound by iteration $\tauub$. 
    In \cref{fig:stepsize-sweep-b}, we see that the off-subspace error slowly decays as the learned network adapts to the latent low-dimensional structure in $X$.
    }
	\label{fig:stepsize-sweep}
\end{figure*}
Readers will note that instantiating~\cref{theorem:main-informal} for ill-conditioned $X$ warrants
prohibitively large width $\dhid$, small regularization parameter $\lambda$, and a large number of iterations $T$.
To that end, we study the behavior of~\eqref{eq:gradient-descent}
with weight decay when $X$ can be written as $\Xlr + \Xsmall$, where $\Xlr$ is low-rank and well-conditioned and $\frobnorm{\Xsmall} \ll \frobnorm{\Xlr}$, by considering its projection
onto the principal singular space and its complement, akin to the Eckart-Young theorem~\cite{Eckart_Young_1936}:
\begin{equation}
    \label{eq:lowrank-plus-small-main}
    X =
    \underbrace{\sum_{i=1}^{r} \sigma_i(X) u_i v_i^\T}_{\Xlr}  + \underbrace{\sum_{i=r+1}^{s} \sigma_i(X) u_i v_i^\T}_{\Xsmall}, \quad
    \text{with} \;\;
    \sigma_{r+1}(X) \gg \sigma_{r}(X).
\end{equation}
In this setting, we can show that gradient descent
will learn an approximate inverse mapping that reconstructs $X_{\tlr}$ at a
rate that only depends on the condition number of $X_{\tlr}$; similarly, the hidden
layer width need no longer scale with the condition number of the full matrix.
As discussed above, the most general form of \cref{cor:ill-conditioned} and its proof appear in \cref{sec:mainresultandproof-appendix}.
\begin{corollary}\label{cor:ill-conditioned}
    Let Assumptions~\ref{assumption:rip} and~\ref{assumption:initialization} hold,
    and fix $r \le s$ to decompose $X$ as in~\eqref{eq:lowrank-plus-small-main}.
    Define the step size $\eta$, weight decay parameter $\lambda$ as
    \begin{equation*}
		\eta := \frac{\din}{L \cdot \sigma_{\max}^2(\Xlr)}, \quad
		\lambda := \gamma \sigma_{\min}^2(\Xlr) \sqrt{\frac{\din}{\dout}} \text{ with }  0 < \gamma \leq     
        \min\left\{1,10^{-7}\cdot\frac{\sqrt{d}}{L\sqrt{m}} \right\},
		\label{eq:informal-theorem-stepsize-and-wd_r}
	\end{equation*} and  time $T$ and width $\dhid$ as
    \begin{equation*}
        T    = \ceil*{\frac{2 L \kappa^2(\Xlr) \log(\dhid)}{\gamma} \sqrt{\frac{\dout}{\din}}}, \quad \dhid = \widetilde \Omega\left( \dout \cdot \sr(\Xlr) \cdot \mathrm{poly}(L, \kappa(\Xlr)) \right).
    \end{equation*}
    Finally, define the stopping time $\tau$ as 
    \begin{equation}
        \tau = \inf\set*{
				t \in \mathbb{N} \mid
				\frobnorm{W_{L:1}(t) AX_{\tlr} - X_{\tlr}} \leq
				\frac{C_{1} \gamma \frobnorm{\Xlr}}{L}
			}             \label{eq:corollary tau def}
    \end{equation}
    where $C_{1}>0$ is a dimension-independent constant.
    Then, as long as $X$ fulfills
    \begin{equation}\label{eq:Xsmall sufficiently small}
        \sigma_{r+1}(X) \lesssim \gamma \cdot \frac{\sigma_{r}(X)}{L^2 \kappa^3(\Xlr) \log(\dhid)} \sqrt{\frac{m}{d(s-r)}},
    \end{equation}
	gradient descent~\eqref{eq:gradient-descent}
	produces iterates that satisfy
	\begin{align}
		\frobnorm{W_{L:1}(t+1)A\Xlr - \Xlr}
		                                          & \leq
		\begin{cases}
			\left(1 - \frac{1}{64 \kappa^2(\Xlr)}\right) \frobnorm{W_{L:1}(t)A\Xlr - \Xlr}, & t < \tau;         \\
			C_{2} \gamma \frobnorm{\Xlr},                                           & \tau \le t \leq T
		\end{cases} \label{eq:cor-regression-error} \\
		\opnorm{W_{L:1}(T) P_{\range(Y)}^{\perp}} & \leq \left( \frac{1}{\dhid} \right)^{C_{3}},
		\label{eq:cor-generalization-error}
	\end{align}
	where $C_{2}, C_{3} > 0$ are universal constants, with probability at least $1-\exp\left(-\Omega(\dout)\right)$ over the random initialization
\end{corollary}
\begin{remark}
    It is possible that there exist multiple choices of $r \le s$ in the decomposition \eqref{eq:lowrank-plus-small-main} that allow the conditions in \eqref{eq:Xsmall sufficiently small} to hold. In that case, applying \cref{cor:ill-conditioned} with different $r$ reveals that the reconstruction error of $W_{L:1}$ converges faster in the direction of the top singular vectors of $X$ because $\kappa(\Xlr)$ is increasing with $r$.
\end{remark}

\section{Robustness to noisy test data}
\label{sec:subsec:robustness}
Training a network on the regularized objective with gradient descent leads to
a learned mapping that closely approximates the oracle mapping $\Worc$.
As a result, training with weight decay parameter $\lambda>0$ guarantees robustness to noise within a bounded number of iterations, while training with $\lambda =0$ allows the training error to be driven arbitrarily close to zero while the error on a noisy test sample remains bounded away from zero for any number of iterations. 
The following corollary formalizes this by considering a test instance with noisy measurements.
\begin{corollary}
	\label{cor:robustness}
	Let $(W_1(T), \dots, W_{L}(T))$ be the weight matrices of a deep linear
	network trained for $T$ iterations in the setting of~\cref{theorem:main-informal} (i.e., $\lambda > 0$, $T$ as in \eqref{eq:T def}, and noise-free training data as in \eqref{eq:dataset-measurements}).
	Consider a test data point $(x, y)$ satisfying $y=Ax+\epsilon$, where $x \in \range(R)$ and
	$\epsilon \sim \mathcal{N}(0, \sigma^2 I_m)$.
	Then, the output of the network $W_{L:1}(T)y$ satisfies
	\begin{equation}\label{eq:robustnesslambda}
		\norm{W_{L:1}(T)y - x} \lesssim
		\left(\gamma \kappa(X) \sqrt{\sr(X)}  + \frac{1}{\dhid^{C_3}}\right) \norm{y}
		+ \sigma \sqrt{s}.
	\end{equation}
    with probability 
    of at least $1 - \exp\left(-\Omega(\dout)\right) - \exp\left(-\Omega(s^2)\right)$ with respect to both the weight initialization and the noise.

	Conversely, let $(W^{\lambda=0}_1(t),...,W^{\lambda=0}_L(t))$
	be the weight matrices of a deep linear network trained in the setting of~\cref{theorem:main-informal} except that $\lambda = 0$. 
    Then, for any $\beta >0$, there exists an iteration $T'$ such that, for all $t >T'$, 
    the training reconstruction error satisfies $\frobnorm{W_{L:1}^{\lambda = 0}(t)Y-X} \leq \beta \frobnorm{X}$ 
    while for all $\alpha > 0$, the test error satisfies
	\begin{align}\label{eq:robustnessnolambda}
		\norm{W^{\lambda=0}_{L:1}(t)y-x} &\gtrsim
		\sigma \left(
            \sqrt{\frac{\dout(\din - s)}{\din}} 
            - (1+\sqrt{\log(\sfrac{1}{\alpha}})) 
            \sqrt{\frac{\dout\kappa^4(X)\sr(X)}{\din}}
            - \sqrt{s}
        \right)\\
        \notag & \quad -
		\beta \kappa(X) \sqrt{\sr(X)} \norm{y}.
	\end{align}
    with probability at least $1 - \exp\left(-\Omega(\dout)\right) - \exp\left(-\Omega(s^2)\right) - \exp\left(-\Omega\left((\din - s)^2\right)\right) - \alpha$.
\end{corollary}
The benefit of weight decay can be deduced from the qualitative behavior of the two bounds: on one hand,
the error in~\eqref{eq:robustnesslambda} can be driven
arbitrarily close to $\sigma \sqrt{s}$ -- which is unimprovable in general -- by choosing $\gamma$ sufficiently small and $\dhid$ sufficiently large.
On the other hand, when $m \gtrsim \kappa^4(X) s$ (which is order optimal for bounded $\KX$),
training without weight decay incurs a
test error of at least $\sigma \sqrt{\dout} - \beta \kappa(X) \sqrt{\sr(X)} \norm{y}$ with high probability. In high dimensions,
this lower bound can be significantly larger than~\eqref{eq:robustnesslambda} unless $\beta$ is large, which means that
$W_{L:1}^{\lambda = 0}(t)$ is a poor fit
to the training data. 

\begin{remark}
 \cref{theorem:main-informal} reveals that when the weight decay parameter $\lambda$ is very small, gradient descent requires many iterations to drive the off-subspace error close to zero. For that reason, it is advantageous to choose $\gamma$ (and hence $\lambda$) as large as possible without compromising robustness to noise at test time.
    From the above corollary, a reasonable choice of $\lambda$ to balance the bias-variance tradeoff and maximize test-time robustness can be determined. 
    Choosing $\gamma$ so that $\gamma \kappa(X) \sqrt{\sr(X)} \propto \nicefrac{1}{\dhid^{C_3}}$ (and therefore $\lambda \propto \frac{\sigma_{\min}^2(X)\sqrt{m}}{\dhid^{C_3}\kappa(X)\sqrt{d\sr(X)}}$) balances the reconstruction error with the off-subspace error in the test-time performance bound. 
    \label{remark:choice of lambda}
\end{remark}
\begin{proof}[Proof of \cref{cor:robustness}]
At a high level, we will  compare the errors of mappings learned with or without weight decay to the error of the oracle mapping,  $\norm{\Worc y - x}$.
    Because the oracle mapping is robust to noise, it remains to bound the difference between the predictions made by the learned and oracle mappings. 
    For the upper bound in \eqref{eq:robustnesslambda}, we control this difference by leveraging \cref{lem:dist from oracle,theorem:main-informal}.
    For the lower bound in \eqref{eq:robustnessnolambda}, we argue that when training without weight decay, driving the training error to zero results in a learned mapping that approximates the oracle on the subspace spanned by the training measurements, but the off-subspace component of the learned mapping is unable to move away from its poorly behaved initialization, resulting in a lack of robustness at test time. We now give the detailed proof.
    
    \paragraph{Upper bound with weight decay} We first prove \eqref{eq:robustnesslambda}. We have
    \begin{equation}
    \label{eq:robust-proof-decomposition}
        \norm{W_{L:1}(T)y - x}
        \leq
        \opnorm{W_{L:1}(T) - \Worc} \norm{y} +
        \norm{\Worc y - x}.
    \end{equation}
    
    In \cref{lem:oracle robustness},
    we show that $ \norm{\Worc y - x} \lesssim \sigma \sqrt{s}$
    with probability at least $1 - \exp\left(-\Omega(s^2)\right)$ over the randomness in the noise.
    By plugging the bounds on the reconstruction and off-subspace errors in~\cref{eq:thm-regression-error,eq:thm-generalization-error} from \cref{theorem:main-informal} into the bound on the distance to the oracle mapping in \cref{lem:dist from oracle}, we have
    \begin{align*}
        \opnorm{ W_{L:1}(T)-\Worc}
        &\lesssim \frac{\frobnorm{ W_{L:1}(T)Y-X}}{\sigma_{\min}(X)}+ \opnorm{ W_{L:1}(T)P_{\range(Y)}^{\perp}} \\
        & \lesssim \frac{\gamma \frobnorm{X}}{\sigma_{\min}(X)}
        + \dhid^{-C_3}\\
        &= \gamma \KX \sqrt{\sr(X)} + \dhid^{-C_3},
    \end{align*}
    with probability at least $1 - \exp\left(-\Omega(\dout)\right)$ over the initialization.
    Returning to~\eqref{eq:robust-proof-decomposition}, we conclude that
    \[
        \norm{W_{L:1}(T)y - x} \lesssim 
        \left( \gamma \KX \sqrt{\sr(X)} + \dhid^{-C_3} \right) \norm{y}
        + \sigma \sqrt{s},
    \]
    with probability at least $1 - \exp\left(-\Omega(\dout)\right) - \exp\left(-\Omega(s^2)\right)$. 
    This yields the upper bound~\eqref{eq:robustnesslambda}.
    
    \paragraph{Lower bound without weight decay}
    We turn to the lower bound in \cref{eq:robustnessnolambda}. 
    By the reverse triangle inequality, we obtain
    \begin{equation}
    \label{eq:robust-proof-decomposition-nolambda}
         \norm{W_{L:1}^{\lambda = 0}(t)y-x}
         \geq \norm{W_{L:1}^{\lambda = 0}(t)P_{\range(Y)}^{\perp}y}
         - \frobnorm{W_{L:1}^{\lambda = 0}(t)P_{\range(Y)}-\Worc} \norm{y}
         - \norm{\Worc y-x}.
    \end{equation}
    Once again, we know from \cref{lem:oracle robustness} that $\norm{\Worc y - x} \lesssim \sigma \sqrt{s}$
    with probability at least $1 - \exp\left(-\Omega(s^2)\right)$.
    \subparagraph{Learned mapping approximates oracle on $\range(Y)$}
    Given $\lambda = 0$, our convergence analysis 
    in \cref{prop:step1-induction} suggests that the reconstruction error decreases monotonically at a
    rate of $1 - \nicefrac{1}{32 \kappa^2(X)}$ with probability $1 - \exp\left(-\Omega(\dout)\right)$ over the initialization (see Remark~\ref{rem:phase 1 lambda 0}). Consequently, for any
    $\beta > 0$ there exists a $T'$ such that $\frobnorm{W_{L:1}^{\lambda = 0}(t) Y - X} \leq \beta \frobnorm{X}$ for $t \geq T'$. 
    Applying $P_{\range(Y)} = YY^\dagger$ and the characterization $\Worc = XY^\dagger$ in \cref{lem:orcal-pseudo-inverse}, we see that
    \begin{equation}
    \label{eq:regression error}
        \frobnorm{W_{L:1}^{\lambda = 0}(t)P_{\range(Y)}-\Worc}
        = \frobnorm{(W_{L:1}^{\lambda = 0}(t)Y-X)Y^\dagger}
        \lesssim \frac{\beta \frobnorm{X}}{\sigma_{\min}(X)}
        = \beta \kappa(X) \sqrt{\sr(X)} 
    \end{equation}
    where 
    the inequality uses the facts that $W_{L:1}^{\lambda = 0}(t)$ attains small reconstruction error and Assumption \ref{assumption:rip} implies that $\sigma_{\min}(Y) \gtrsim \sigma_{\min}(X)$ (see \cref{lem:RIP bounds on sv}).

    \subparagraph{Initialization behaves poorly on $\range(Y)^\perp$}
    We now consider the term $\norm{W_{L:1}^{\lambda = 0}(t) P_{\range(Y)}^{\perp}y}$. Because $Ax \in \range(Y)$, we see that
    \begin{align*}
        \norm{W_{L:1}^{\lambda = 0}(t) P_{\range(Y)}^{\perp}y}
        &= \norm{W_{L:1}^{\lambda = 0}(t) P_{\range(Y)}^{\perp}\epsilon} \\
        &\ge 
         \norm{W_{L:1}^{\lambda = 0}(0) P_{\range(Y)}^{\perp}\epsilon}
         - 
          \norm{(W_{L:1}^{\lambda = 0}(0) - W_{L:1}^{\lambda = 0}(t)) P_{\range(Y)}^{\perp}\epsilon}
        .
    \end{align*}
    We show in \cref{lem:init off subspace error lambda 0} that 
    \begin{equation}
    \label{eq:init off subspace error}
        \norm{W_{L:1}^{\lambda = 0}(0) P_{\range(Y)}^{\perp}\epsilon}
        \gtrsim \sigma \sqrt{\frac{\dout(\din - s)}{\din}}
    \end{equation}
    with probability $1-\exp{(-\Omega((m-s)^2))} - \exp{(-\Omega(d))}$. (Notice that $W_{L:1}^{\lambda = 0}(0) = W_{L:1}(0)$, because the initalization is independent of $\lambda$.)
    
    \subparagraph{Learned mapping stays near initialization}
    To complete the proof, we leverage bounds derived in the proof of our main result (\cref{prop:step1-induction}). Specifically,
    our analysis shows that when training without weight decay, the learned mapping $W_{L:1}^{\lambda = 0}$ cannot move very far from its initialization; we have the bound
    \begin{equation}\label{eq:dont leave init no wd}
        \frobnorm{W_{L:1}^{\lambda = 0}(0) - W_{L:1}^{\lambda = 0}(t)} 
        \lesssim \sqrt{\frac{\dout\kappa^4(X)\sr(X)}{\din}}
    \end{equation}
    with probability $1-\exp(-\Omega(d))$ over the initialization. (We note that for $\lambda=0$, \eqref{eq:dont leave init no wd} can be obtained by simplifying the bound on $\frobnorm{F(t)-\call^t  F(0)}$ in \cref{eq:event-B-prestop} from \cref{prop:step1-induction}, with $F$ and $\call$ defined in \cref{eq:F def} and \cref{eq:c-prod-i}). 
    In \cref{lem:gaussian concentration}, we use standard concentration arguments to show that, given a matrix $M$, $\norm{M\epsilon} \le \sigma(1+\sqrt{\log(\sfrac{1}{\alpha})})\frobnorm{M}$ with probability at least $1-\alpha$ with respect to the randomness in the noise $\epsilon$. 
    Combined with \cref{eq:dont leave init no wd}, this implies that 
    \begin{equation}
    \label{eq:change in off subspace error}
        \norm{(W_{L:1}^{\lambda = 0}(0) - W_{L:1}^{\lambda = 0}(t)) P_{\range(Y)}^{\perp}\epsilon}
        \lesssim \sigma(1+\sqrt{\log(\sfrac{1}{\alpha})}) \sqrt{\frac{\dout\kappa^4(X)\sr(X)}{\din}}
    \end{equation}
    with probability $1-\alpha - \exp(-\Omega(d))$.
    Finally, we obtain the lower bound in \eqref{eq:robustnessnolambda} by combining the bounds on each term in \cref{eq:robust-proof-decomposition-nolambda} from \cref{eq:init off subspace error,eq:change in off subspace error,eq:regression error}.
\end{proof}

\section{Generalization of main result and proof}
\label{sec:mainresultandproof-appendix}
\label{sec:subsec:main result formal}
We state the generalized version of our main result (with relaxed conditions on the step size $\eta$), which includes \cref{theorem:main-informal} and \cref{cor:ill-conditioned} as special cases. 
\begin{theorem}[Generalized main theorem]
  \label{thm:mainresult-formal}  Let Assumptions \ref{assumption:rip} and \ref{assumption:initialization} hold and fix $r \le s$ to decompose $X$ as in~\eqref{eq:lowrank-plus-small-main}. Assume
  \begin{equation}
      \eta \leq \frac{m}{L \sigma^2_{\max}(X)},  \;\; 
      \dhid = \widetilde \Omega\left( \dout \cdot \sr(X) \cdot \mathrm{poly}(L, \kappa(X)) \right), \;\;  
      \lambda =\gamma \sigma_{\min}^2(\Xlr) \sqrt{\frac{m}{d}}
      \label{eq:main-thm-assumptions}
  \end{equation}
   where 
    $0 < \gamma \leq     
    \min\left\{1,10^{-7}\cdot \frac{\sqrt{d}}{L\sqrt{m}} \right\}$.
   Moreover, define the times
		\begin{align}
    		\tau & = 
            \inf\set*{
				t \in \mathbb{N} \mid
				\frobnorm{W_{L:1}(t) AX_{\tlr} - X_{\tlr}} \leq
				\frac{C_1 \gamma\frobnorm{X_{\tlr}}}{L}
            },
            \label{eq:tau-def-phi}
            \\
			T    & = \frac{2 \din \log(\dhid)}{\eta\lambda}.
            \label{eq:T def appendix}
		\end{align}
        As long as $\Xsmall$ fulfills
    \begin{equation}\label{eq:Xsmall sufficiently small-appendix}
        \opnorm{\Xsmall}
        \lesssim \gamma \cdot \frac{\sigma_{\min}(\Xlr)}{L^2 \kappa^3(\Xlr) \log(\dhid)} \sqrt{\frac{m}{d(s-r)}},
    \end{equation}  with probability at least $1-\exp\left(-\Omega(\dout)\right)$ over the random initialization, the following holds:
\begin{align}
	\frobnorm{W_{L:1}(t+1) A X_{\tlr} - X_{\tlr}} & \leq
		\begin{cases}
			\left(1 - \frac{\eta L \sigma_{\min}^2(X_{\tlr})}{64 \din}\right) \frobnorm{W_{L:1}(t) A X_{\tlr} - X_{\tlr}},     & t < \tau;        \\
			C_2 \gamma \frobnorm{X_{\tlr}}, & \tau \leq t \leq T.
		\end{cases}
        \label{eq:thm-regression-error-formal} \\
     \opnorm{W_{L:1}(T) P_{\range(Y)}^{\perp}}  
     &\leq \left( \frac{1}{\dhid} \right)^{C_3}.
		\label{eq:thm-generalization-error-formal}
\end{align}
The values $C_1$, $C_2$, $C_3 > 0$ are universal constants.
\end{theorem}
The remainder of \cref{sec:mainresultandproof-appendix} presents the proof of \cref{thm:mainresult-formal}.
\begin{remark}
    Note that \cref{eq:Xsmall sufficiently small-appendix} is trivially satisfied even for arbitrarily small $\gamma$ by choosing $r = s$ and hence $\Xsmall = 0_{\dout \times n}$; this gives the result presented in \cref{theorem:main-informal}. However, \cref{thm:mainresult-formal} is most informative when $r$ can be chosen so that $\Xlr$ is well conditioned.
    We note that when $X = X_{\tlr}$ and $X_{\tsm} = 0$, the guarantee for the reconstruction error can be improved to
    \[
        \frobnorm{W_{L:1}(t+1) Y - X} \leq
        \left(1 - \frac{\eta L \sigma_{\min}^2(X)}{32 \din}\right) \frobnorm{W_{L:1}(t)Y - X}
        \quad \text{for $t < \tau$}.
    \]
    This is the convergence rate reported in~\cref{theorem:main-informal}.
\end{remark}
\begin{remark}
    Throughout the remainder of the proof, Assumptions \ref{assumption:initialization}, \ref{assumption:rip}, \eqref{eq:Xsmall sufficiently small-appendix}, and \eqref{eq:main-thm-assumptions}  are in force.
\end{remark}

\subsection{Simplifying the number of samples}
\label{sec:subsec:simplifying num samples}
We may assume that we have exactly $s$ input samples for the purpose of analysis.
As shown in~\cref{lemma:exact-rank-X}, the gradient descent trajectories remain
unchanged when the number of samples $n$ is larger than $s$. In particular, for any dataset with $n$ samples, we can find a dataset with $s$ samples that induces the exact
same gradient descent trajectory and sequence of loss
function values. Hence, without loss of generality, we may assume that $X = RZ$ for some $Z \in \Rbb^{s \times s}$ with $\rank(Z) = s$.
Throughout the remainder of this section, we will assume that $n = s$.

\subsection{Preliminary derivations}

\label{sec:subsec:Lemmas used for the proof} 
It is convenient to be explicit in our notation about the normalization factors at initialization. 
For that reason, throughout the remainder of the proof, we will consider the equivalent loss function 
\begin{equation}
	\cL(\set{W_{\ell}}_{\ell = 1, \dots, L}; (X, Y)) := \frac{1}{2} \left\lVert{
		d_w^{-\frac{L-1}{2}}m^{-\frac{1}{2}} W_{L:1} Y - X}\right\rVert_{\mathsf{F}}^2
	+ \frac{\lambda}{2} \sum_{\ell=1}^L
	\frac{\frobnorm{W_\ell}^2}{d_\ell}
	\label{eq:reformulated-loss-appendix}
\end{equation}
under the assumption that $(W_\ell(0))_{ij} \iid \cN(0, 1)$ 
for all $\ell = 1, \ldots, L$.
We use the shorthand 
$F$
for the learned mapping from measurements to signals:
\begin{equation}
\label{eq:F def}
    F := d_w^{-\frac{L-1}{2}}m^{-\frac{1}{2}} W_{L:1}.
\end{equation}
We use
$U$ 
to refer to the network predictions on the training samples $X$, and 
$\Phi$
to denote the training residuals; 
we will further break $U$ an $\Phi$ into components corresponding to the 
well-conditioned and ``small" parts of $X = \Xlr + \Xsmall$:
\begin{align}
\label{eq:U and Phi def}
\begin{aligned}
    U      &:= FY, \\
    \Phi    &:= U-X,
\end{aligned}
\qquad
\begin{aligned}
    \Ulr      &:= FA\Xlr, \\
    \Philr    &:= \Ulr-\Xlr,
\end{aligned}
\qquad
\begin{aligned}
    \Usmall   &:= FA\Xsmall, \\
    \Phismall &:=  \Usmall-\Xsmall.
\end{aligned}
\end{align}
We shall also write $\cprod{i}$ 
and $\call$ for the following products appearing in our proofs:
\begin{equation}
	\cprod{i} := \prod_{\substack{j = 1\\j \neq i}}^{L} \left(1 - \frac{\eta \lambda}{d_j} \right) 
    \quad \text{and} \quad
	\call := \prod_{i = 1}^{L} \left(1 - \frac{\eta \lambda}{d_i}\right).
	\label{eq:c-prod-i}
\end{equation}
Given the gradient with respect to $W_i$ of the loss in \cref{eq:reformulated-loss-appendix}, the gradient descent updates are
\begin{equation}
	W_i(t+1)
	= \left(1-\frac{\eta\lambda}{d_i}\right) W_i(t) - \eta
	d_w^{-\frac{L-1}{2}}m^{-\frac{1}{2}} W_{L:i+1}(t)^{\T} \Phi(t) Y^{\T} W_{i-1:1}(t)^{\T}
  \quad \text{for $1 \leq i \leq L$.}
	\label{eq:Wi-update}
\end{equation}
The above leads to the following decomposition of the product weight matrix at step $t$:
	\begin{align}
		W_{L:1}(t+1) &=
    \begin{aligned}[t]
         & \call W_{L:1}(t) + E_0(t) \\
         & - \eta d_w^{-\frac{L-1}{2}}m^{-\frac{1}{2}}
			     \sum_{i=1}^L
			     \cprod{i}
			     W_{L:i+1}(t) W_{L:i+1}^{\T}(t) \Phi(t) Y^{\T} W_{i-1:1}^{\T}(t) W_{i-1:1}(t),
    \end{aligned}
    \label{eq:prod-evolution-I}
	\end{align}
	with $E_0(t)$ containing all $O(\eta^2)$ terms. This is essentially the decomposition in~\cite[Section 5]{du2019width}, modified because of the effect of weight decay in \cref{eq:Wi-update}. For the sake of brevity, we do not repeat the argument here.
Multiplying both sides of \eqref{eq:prod-evolution-I} from the right by $d_w^{-\frac{L-1}{2}}m^{-\frac{1}{2}} A \Xlr$, subtracting $\Xlr$, taking norms, and applying the triangle inequality and the bound $|1 - \call| \le \sfrac{2\eta \lambda}{m}$ from \cref{lemma:one-minus-folded-product}, we obtain
    \begin{align} \label{eq:error-evolution-II-AX-main}
	\frobnorm{\Philr(t+1)} 
  & \leq
	\begin{aligned}[t]
    & \left(1 - \eta \left(\lambda_{\min}(\Plr(t)) - \opnorm{\Psmall(t)} \right)\right)
	\frobnorm{\Philr(t)} \\
	&+ \callbound \frobnorm{\Ulr(t)} 
	+ \eta \opnorm{P(t)} \frobnorm{\Phismall(t)}+ \frobnorm{E(t)},
    \end{aligned}
\end{align}
as long as $\eta \le \frac{1}{\lambda_{\max}(P(t))}$.
In the previous equation, we have used the following notation
(dropping the
time index $t$ for simplicity):
\begin{align}
	\Plr
	&:= d_w^{-(L-1)}m^{-1}
	\sum_{i=1}^L
	\cprod{i}
	\left((A\Xlr)^{\T} W_{i-1:1}^{\T} W_{i-1:1}A\Xlr \right)
	\otimes
	\left(W_{L:i+1} W_{L:i+1}^{\T}\right) \label{eq:def-pr}\\
    \Psmall &:= d_w^{-(L-1)}m^{-1}
	\sum_{i=1}^L
	\cprod{i}
	\left((A\Xsmall)^{\T} W_{i-1:1}^{\T} W_{i-1:1}A\Xlr \right)
	\otimes
	\left(W_{L:i+1} W_{L:i+1}^{\T}\right) \label{eq:def-psmall}\\
    P &:= \Plr + \Psmall.
    \label{eq:def-p}\\
    E &:= d_w^{-\frac{L-1}{2}}m^{-\frac{1}{2}} E_0 A \Xlr \label{eq:def-E}
\end{align}
A full proof of \cref{eq:error-evolution-II-AX-main} can be found in \cref{lem:error-evolution-II-AX}.

Intuitively, \cref{eq:error-evolution-II-AX-main} suggests that
bounding the spectrum of $\Plr$ and $\Psmall$ will allow us to get a recursive bound on the norm of the residual. 
Therefore, we furnish bounds on the spectrum of $\Plr$ and $\Psmall$ in terms of the spectrum of $W_{L:i+1}$ and $W_{i-1:1} Y$, for $i = 1 \ldots L$ (see \cref{lemma:P-k-spectrum} for details).

\subsection{Properties at initialization}
\label{sec:subsec:Properties at initialization}
Let us bound $\frobnorm{\Phi}, \frobnorm{U}$, and $\opnorm{W_{1} P_{\range(Y)}^{\perp}}$, and  as well as the extremal singular values of $W_{j:1}Y$, $W_{L:i}$, and $W_{i:j}$ at initialization.

\begin{lemma}
	\label{lemma:restricted-WL_singular_values}
    We have
	\begin{subequations}
		\begin{align}
			\prob{
			\max_{1 < i \leq L}
			d_{w}^{-\frac{L - i + 1}{2}}
			\sigma_{\max}(W_{L:i}(0))
			\leq \frac{6}{5}
			} & \geq 1 - \exp\left(-\Omega\left(\frac{\dhid}{L}\right) \right),
			\label{eq:restricted-WL-sval-ub-unif}
			\\
			\prob{
			\min_{1 < i \leq L}
			d_{w}^{-\frac{L - i + 1}{2}} \sigma_{\min}(W_{L:i}(0))
			\geq \frac{4}{5}
			} & \geq 1 - \exp\left(-\Omega\left(\frac{\dhid}{L}\right) \right).
			\label{eq:restricted-WL-sval-lb-unif}
		\end{align}
	\end{subequations}
\end{lemma}

\begin{lemma}
	\label{lemma:norm-product-bounded}
	We have
	\begin{equation}
		\prob{
			\max_{1 < k \leq j < L}
			\dhid^{-\frac{j - k + 1}{2}} \opnorm{W_{j:k}(0)} \lesssim \sqrt{L}
		} \geq
		1 - \exp\left(-\Omega\left(\frac{\dhid}{L}\right) \right).
		\label{eq:norm-product-bounded-uniform}
	\end{equation}
\end{lemma}

\begin{lemma}	\label{lemma:restricted-singular-values}
    We have
	\begin{subequations}
		\begin{align*}
             \prob{\max_{1 \leq i < L} d_{w}^{-\frac{i}{2}} \sigma_{\max}(W_{i:1}(0) A\Xlr) \leq \frac{6}{5} \sigma_{\max}(\Xlr)}
			 & \geq 1 - \exp\left(-\Omega\left(\frac{\dhid}{L}\right) \right), \\
			\prob{\min_{1 \leq i < L} d_{w}^{-\frac{i}{2}} \sigma_{\min}(W_{i:1}(0) A\Xlr) \geq \frac{4}{5} \sigma_{\min}(\Xlr)}
			 & \geq 1 - \exp\left(-\Omega\left(\frac{\dhid}{L}\right) \right),\\
             \prob{\max_{1 \leq i < L} d_{w}^{-\frac{i}{2}} \sigma_{\max}(W_{i:1}(0) A\Xsmall) \leq \frac{6}{5} \sigma_{\max}(\Xsmall)}
			 & \geq 1 - \exp\left(-\Omega\left(\frac{\dhid}{L}\right) \right),\\
             \prob{\max_{1 \leq i < L} d_{w}^{-\frac{i}{2}} \sigma_{\max}(W_{i:1}(0)) \leq \frac{6}{5}}
			 & \geq 1 - \exp\left(-\Omega\left(\frac{\dhid}{L}\right) \right).\\
		\end{align*}
	\end{subequations}
\end{lemma}

\begin{lemma}
	\label{lemma:initial-regression-error}
	At initialization, it holds that
	\begin{align*}
         \frobnorm{\Philr(0)} \lesssim 
		\sqrt{\frac{d}{m}} \frobnorm{\Xlr} 
        ~~~\text{and}~~~
        \frobnorm{\Usmall(0)} \lesssim 
		\sqrt{\frac{d}{m}} \frobnorm{\Xsmall}
	\end{align*}
	with probability at least $1 - \exp\left(-\Omega(\dout)\right)$ as long
	as $\din \gtrsim s$ and $\dhid \gtrsim L \din$. 
\end{lemma}

\begin{lemma}
\label{lem:off subspace at init}
    With probability at least $1-\exp(-\Omega(\dhid^2))$, $\opnorm{W_{1}(0) P_{\range(Y)}^{\perp}} \lesssim \sqrt{\dhid}.$
\end{lemma}

We note that a simple union bound shows that
all the bounds in~\cref{lemma:initial-regression-error,lemma:restricted-WL_singular_values,lemma:restricted-singular-values,lemma:norm-product-bounded} are fulfilled simultaneously with probability at least
$1 - \exp\left(-\Omega(\dout)\right)$. 
\cref{lemma:restricted-WL_singular_values,lemma:norm-product-bounded} are proven in Propositions 6.2 and 6.4 of \cite{du2019width}, while the proof of \cref{lemma:restricted-singular-values} is a slight modification of that of Proposition 6.3 in \cite{du2019width} to account for the in Restricted Isometry Property (Assumption~\ref{assumption:rip}).
The proofs of 
\cref{lemma:restricted-singular-values,lemma:initial-regression-error,lem:off subspace at init}
can be found in \cref{sec:prop at init proofs}.

\subsection{Step 1: Rapid early convergence}
\label{sec:subsec:Step 1: Rapid early convergence}
The first step of our convergence analysis is 
showing a sufficient decrease in the reconstruction error until time $\tau$.
We will prove the following proposition in this section. 
\begin{proposition}
\label{prop:step1-induction}
    For all $0 \leq t \leq \tau$, the following events hold with probability of at least $1-\exp\left(-\Omega(\dout)\right)$ over the random initialization:
    \begin{subequations}
      \begin{align}
          \cA(t) &:=
          \set[\Big]{
              \frobnorm{\Phi_{\tlr}(t+1)} \leq \left(1- \frac{\eta L\sigma_{\min}^2(\Xlr)}{64\din} \right) \frobnorm{\Phi_{\tlr}(t)}
          },
        \label{eq:event-A-prestop} \\
          \cB(t) &:= \left\{
          \begin{array}{lcll}
              \sigma_{\max}(W_{j:i}(t)) &\lesssim&
              \sqrt{L} \dhid^{\frac{j - i + 1}{2}}, &
              1 < i \leq j < L \\
              \sigma_{\max}(W_{i:1}(t) A \Xlr) &\leq&
              \frac{5}{4} \dhid^{\frac{i}{2}} \cdot \sigma_{\max}(\Xlr), &
              1 \leq i < L, \\
              \sigma_{\max}(W_{L:i}(t)) &\leq&
              \frac{5}{4} \dhid^{\frac{L - i + 1}{2}}, & 1 < i \leq L, \\
               \sigma_{\min}(W_{i:1}(t) A\Xlr) &\geq&  \frac{3}{4}\dhid^{\frac{i}{2}} \cdot
            \sigma_{\min}(\Xlr), & 1 \leq i < L, \\
              \sigma_{\min}(W_{L:i}(t)) &\geq&
              \frac{3}{4} \dhid^{\frac{L - i + 1}{2}}, & 1 < i \le L, \\ 
            \sigma_{\max}(W_{i:1}(t)A \Xsmall) &\leq& \frac{5}{4} \dhid^{\frac{i}{2}} \cdot
            \sigma_{\max}(\Xsmall), & 1 \leq i < L \\
            \frobnorm{F(t) - \call^t F(0)}  &\lesssim& \din^{-\frac{1}{2}} RL
          \end{array} \right\},
          \label{eq:event-B-prestop}
          \\
          \cC(t) &:= \begin{aligned}[t]
          & \set[\Big]{
              \frobnorm{W_{i}(t) - \left(
              1 - \frac{\eta \lambda}{d_i}
              \right)^t W_{i}(0)}
              \lesssim
              R
              \mid
              1 \leq i \leq L
          }
          \end{aligned}
          \label{eq:event-C-prestop}
      \end{align}
      where
      $R := \frac{\Klr^2 \sqrt{\dout \sr(X_{\tlr})}}{L}$.
    \end{subequations}
\end{proposition}
We will prove \cref{prop:step1-induction} by induction, starting with $t=0$ (\cref{lemma:step-i-properties-at-initialization}).
We then proceed by showing that:
\begin{itemize}
  \item $\set{\cA(j)}_{j < t}$ and $\cB(t)$ imply $\cA(t)$ (\cref{lemma:step-i-all-else-implies-A});
  \item $\set{\cA(j), \cB(j)}_{j < t}$ imply $\cC(t)$ (\cref{lemma:Bj-implies-Ct-phase-1});
  \item $\cC(t)$ implies $\cB(t)$ (\cref{lemma:Ct-implies-Bt-phase-1-main}).
\end{itemize}
The proof of~\cref{prop:step1-induction} follows by iterating the above implications
until the stopping time $\tau$ is reached.
\begin{remark}\label{rem:phase 1 lambda 0}
    We note that if $\lambda = 0$ (and hence $\gamma = 0)$, then \cref{prop:step1-induction} holds for all $0 \le t \le \infty$ with an improved convergence rate in \cref{eq:event-A-prestop} of 
    \[
    \frobnorm{\Phi_{\tlr}(t+1)} \leq \left(1- \frac{\eta L\sigma_{\min}^2(\Xlr)}{32\din} \right) \frobnorm{\Phi_{\tlr}(t)};
    \]
    compare to \eqref{eq:loss-evolution-decomposition}.
\end{remark}

\begin{lemma}[Initialization]
    \label{lemma:step-i-properties-at-initialization}
    The events $\cA(0)$, $\cB(0)$ and $\cC(0)$ hold with probability at least $1-\exp\left(-\Omega(\dout)\right)$.
\end{lemma}
\begin{proof}
    The base case $\cC(0)$ is trivial. On the other hand, $\cB(0)$ follows from
    \cref{lemma:restricted-singular-values,lemma:norm-product-bounded,,lemma:restricted-WL_singular_values}. 
    Finally, we show in \cref{lemma:step-i-all-else-implies-A}
    that $\cB(0)$ implies $\cA(0)$.
\end{proof}

\begin{lemma}
    \label{lemma:step-i-all-else-implies-A}
    For any $0 \leq t < \tau$, $\set{\set{\cA(j)}_{j < t}, \cB(t)} \implies \cA(t)$. 
\end{lemma}
\begin{proof}
Guided by our intuition from \eqref{eq:error-evolution-II-AX-main} in \cref{sec:subsec:Lemmas used for the proof} that bounding the spectrum of $\Plr$ and $\Psmall$ should allow us to get a recursive bound on the reconstruction error $\frobnorm{\Philr}$, we leverage $\cB(t)$ to bound the spectrum of $\Plr$ and $\Psmall$.
Conceptually, we expect $E$ to be negligible because $E$ contains higher-order terms, 
and we expect terms involving $\Phismall$ to be small relative to $\lambda$ because \eqref{eq:Xsmall sufficiently small-appendix} requires $\Xsmall$ to be small relative to $\lambda$.
The definition of $\tau$ implies that while $t \le \tau$, $\frobnorm{\Philr}$ is large relative to terms involving $\lambda$. 
Thus, by bounding $E, \Phismall$, and $\Ulr$ appropriately, we may fold the remaining terms in \eqref{eq:error-evolution-II-AX-main} into our recursive bound on $\frobnorm{\Philr}$.
\paragraph{Bounds on the spectrum of $P$ via $\cB(t)$}
By applying bounds on the spectrum of $P$ proven in \cref{lemma:P-k-spectrum} and the event $\cB(t)$, one can easily obtain   
\begin{equation}
\label{eq:P(t)-bounds-1}
  \lambda_{\min}(\Plr(t)) \geq \frac{L\sigma_{\min}^2(\Xlr) }{16m}, \ \
   \opnorm{\Psmall(t)} \lesssim \frac{L\opnorm{\Xsmall} \opnorm{X} }{m}, \ \
    \opnorm{P(t)} \lesssim \frac{L \opnorm{X}^2 }{m}. 
\end{equation} 
  A detailed proof of these inequalities can be found in \cref{lemma:supplement-Pt-extreme-eigenvalues}.

  \paragraph{Bounding the higher-order terms $E$}
  One can bound the Frobenius norm of $E(t)$ by 
    \begin{align}
    \label{eq:E(t):bound-phase-1}
        \frobnorm{\Elr(t)}
        \leq  \frac{\eta L \sigma_{\min}^2(\Xlr)}{64 m}
        \frobnorm{\Phi(t)}
        \leq  
        \frac{\eta L \sigma_{\min}^2(\Xlr)}{64 m}
        \frobnorm{\Philr(t)} 
        + O\left( \frac{\eta L }{m}\right)
        \opnorm{X}^2 \frobnorm{\Phismall(t)}.
    \end{align}
 The proof of this bound is deferred to the supplement (see \cref{lemma:Bt-implies-bound-on-E}).
 
 \paragraph{Bounding $\Ulr$} Since $\Philr = \Ulr - \Xlr$, we may use $\cA(0), \ldots, \cA(t-1)$ and \cref{lemma:initial-regression-error} to see that 
 \begin{equation}
     \frobnorm{\Ulr(t)} 
     \le \frobnorm{\Philr(t)} + \frobnorm{\Xlr} 
     \le \frobnorm{\Philr(0)} + \frobnorm{\Xlr} 
     \lesssim \sqrt{\frac{d}{m}} \frobnorm{\Xlr}.
     \label{eq:Ulr bound}
 \end{equation}

 \paragraph{Bounding a quantity involving $\Phismall$}
We show in \cref{lemma:Phismall bound phase 1} that combining the bound on $\frobnorm{F(t) - \call^t F(0)}$ from $\cB(t)$ and the assumption that $\Xsmall$ is sufficiently small in \cref{eq:Xsmall sufficiently small-appendix} yields
    \begin{equation}\label{eq:Phismall bound phase 1}
        L \opnorm{X}^2\frobnorm{\Phismall(t)} 
        \lesssim \lambda \sqrt{\frac{\dout}{\din}} \frobnorm{\Xlr}.
    \end{equation}

 \paragraph{Plugging $E, \Ulr$ and $\Phismall$ bounds into bound on $\frobnorm{\Philr}$}
    Starting from the error decomposition in~\cref{eq:error-evolution-II-AX-main} and plugging in the bounds from \cref{eq:P(t)-bounds-1,eq:E(t):bound-phase-1,eq:Ulr bound},
    we obtain
 \begin{align}
      \frobnorm{\Philr(t+1)}
     & \leq 
     \begin{aligned}[t] &\left(1- \frac{\eta L}{\din} \left( \frac{\sigma_{\min}^2(\Xlr)}{16} - O\left(\opnorm{\Xsmall} \opnorm{X}\right) \right) \right) \frobnorm{\Philr(t)} \\
    	&+ O\left(\frac{\eta \lambda}{\din}\right)  \sqrt{\frac{d}{m}} \frobnorm{\Xlr}
    	+ O\left(\frac{\eta L}{\din} \right)\opnorm{X}^2\frobnorm{\Phismall(t)}\\
     &+ \frac{\eta L}{\din} \cdot \frac{\sigma_{\min}^2 (\Xlr)}{64} \frobnorm{\Philr(t)}
     \end{aligned}
 \end{align}
 By the following observations, we can further bound the above. 
 First, we combine terms involving $\Philr(t)$.
 The bound on $\Xsmall$ in \cref{eq:Xsmall sufficiently small-appendix} implies that $\opnorm{\Xsmall}\opnorm{X} \lesssim \sigma_{\min}^2(\Xlr)$, and so by choosing the constant hidden in \cref{eq:Xsmall sufficiently small-appendix} sufficiently small, we get
 \begin{equation*}
    -\frac{\sigma_{\min}^2(\Xlr)}{16} 
    + \frac{\sigma_{\min}^2 (\Xlr)}{64} 
    + O(\opnorm{\Xsmall}\opnorm{X})
    \le -\frac{ \sigma_{\min}^2 (\Xlr)}{32}.
 \end{equation*}
 Next, we combine terms involving $\Phismall(t)$ with those involving $\lambda$ via \eqref{eq:Phismall bound phase 1}
 Altogether, this leads us to the bound
 \begin{align}
  \frobnorm{\Philr(t+1)}
&\leq \begin{aligned}[t]
& \left(1- \frac{\eta  L\sigma_{\min}^2(\Xlr)}{32\din} \right) \frobnorm{\Philr(t)} + 
\frac{\eta \lambda}{\din} O\left(\sqrt{\frac{\dout}{\din}}
\frobnorm{\Xlr}\right).
\end{aligned}
\label{eq:loss-evolution-decomposition}
    \end{align}
    \paragraph{Folding $O(\lambda)$ terms into recursive bound on $\frobnorm{\Philr}$}
     If $t < \tau$, then by the definition of $\tau$ in~\eqref{eq:tau-def-phi} we must have that
     \begin{equation*}
         \frobnorm{\Philr(t)} 
         > \frac{C_1 \gamma \frobnorm{X_{\tlr}}}{L}
         = \frac{C_1 \lambda}{L\sigma_{\min}^2(\Xlr)}\sqrt{\frac{\dout}{\din}} \frobnorm{X_{\tlr}},
     \end{equation*}
     and so
    \begin{align*}
      \frobnorm{\Philr(t+1)} 
      \leq&
     \left(1- \frac{\eta L\sigma_{\min}^2(\Xlr)}{\din} \left[\frac{1}{32} - O\left(\frac{1}{C_1}\right)\right]\right) \frobnorm{\Philr(t)}\\
      \leq&
       \left(1- \frac{\eta L\sigma_{\min}^2(\Xlr)}{64\din} \right) \frobnorm{\Philr(t)}.  
     \end{align*}
    by choosing $C_1$ sufficiently large. Therefore, $\cA(t)$ holds.
\end{proof}
\begin{corollary}
    \label{corollary:length-of-step-1}
    The stopping time $\tau$ satisfies
    \[
        \tau 
        \leq
        \frac{64 \din}{\eta L \sigma_{\min}^2(\Xlr)}
        \log\left(
          \frac{
          L \sigma_{\min}^2(\Xlr)}{\lambda}  
        \right)
        \leq
        \frac{64 \din}{\eta \lambda}
        .
    \]
\end{corollary}
\begin{proof}
    For any $t < \tau$, \cref{lemma:step-i-all-else-implies-A} implies
    \begin{align*}
        \frobnorm{\Philr(t+1)} &\leq
       \left(1- \frac{\eta L\sigma_{\min}^2(\Xlr)}{64m} \right)  \frobnorm{\Philr(t)} \leq
        \left(1- \frac{\eta L\sigma_{\min}^2(\Xlr)}{64m}  \right)^{t} \frobnorm{\Philr(0)} \\
        &\leq
        \exp\left(
        -t \cdot \frac{\eta L\sigma_{\min}^2(\Xlr)}{64m} 
        \right) \frobnorm{\Philr(0)} \lesssim
         \exp\left(
        -t\frac{\eta L\sigma_{\min}^2(\Xlr)}{64m} 
        \right) \sqrt{\frac{d}{m}}  \frobnorm{\Xlr} ,
    \end{align*}
    where the penultimate inequality follows from the identity $1 - x \leq \exp(-x)$ and
    the last inequality follows from~\cref{lemma:initial-regression-error}. Finally, we obtain
    \[
    t \geq \frac{64 \din}{\eta L \sigma_{\min}^2(\Xlr)}
        \log\left(
          \frac{
          L \sigma_{\min}^2(\Xlr)}{\lambda}  
        \right) \implies
        \frobnorm{\Phi_{\tlr}(t+1)}
        \leq  \frac{C_1 \gamma \frobnorm{\Xlr}}{L}.
    \]
    This yields the first inequality in the statement of the result.
    Finally, since $\log x \le x$, we see that
    \begin{equation}
        \frac{64 \din}{\eta L \sigma_{\min}^2(\Xlr)}
        \log\left(
          \frac{
          L \sigma_{\min}^2(\Xlr)}{\lambda}  
        \right)
        \leq 
        \frac{64 \din}{\eta \lambda}.
    \end{equation}
\end{proof}
\begin{lemma}
  \label{lemma:Bj-implies-Ct-phase-1}
  For any $t \leq \tau$, we have that
  $\set{\cA(j), \cB(j)}_{j < t} \implies \cC(t)$: 
  \begin{equation}
    \frobnorm{W_{i}(t) - \left(1 - \frac{\eta \lambda}{d_i}\right)^{t} W_{i}(0)}
    \lesssim
    \frac{\kappa^2(X_{\tlr}) \sqrt{d \sr(\Xlr)}}{L} := R
    , \quad
    \text{for all $i = 1, \dots, L$.}
    \label{eq:Wi-travel-distance-improved}
  \end{equation}
\end{lemma}
\begin{proof}
    In outline, our approach is to bound the left-hand side of \eqref{eq:Wi-travel-distance-improved} via the triangle inequality in terms of a sum of individual gradient descent updates. After applying the bounds from $\cB(t)$, this yields a bound involving a sum of the errors $\frobnorm{\Phi(j)}$ for $j<t$, which we split into a sum of low-rank parts $\frobnorm{\Philr(j)}$ and small parts $\frobnorm{\Phismall(j)}$. Since $\frobnorm{\Philr(j)}$ is shrinking for $j<t$, we may control the sum of low-rank parts $\frobnorm{\Philr(j)}$ via $\frobnorm{\Philr(0)}$, and we control the small parts $\frobnorm{\Phismall(j)}$ via the norm of $\Xsmall$, as detailed below.

    \paragraph{Sum of individual gradient descent updates}
    In \cref{lemma:difference-norm}, we bound the quantity of interest in terms of a sum of individual gradient descent updates. This gives
    \begin{align}
     &\frobnorm{W_{i}(t) - \left(1 - \frac{\eta \lambda}{d_i}\right)^t W_{i}(0)}     \label{eq:bound on deviation}\\
    &\leq \eta \sum_{j = 0}^{t-1}
    d_{w}^{-\frac{L - 1}{2}} \din^{-\frac{1}{2}}
    \opnorm{W_{L:(i+1)}(j)} \frobnorm{\Phi(j)} \opnorm{W_{(i-1):1}(j) Y} \notag\\
    &\lesssim \eta  \din^{-\frac{1}{2}} \opnorm{X} \sum_{j = 0}^{t-1} \frobnorm{\Phi(j)} \notag
    \end{align}
    where in the final line we used the upper bounds on $\opnorm{W_{L:(i+1)}(j)}$, $ \opnorm{W_{(i-1):1}(j) A\Xlr}$, and $\opnorm{W_{(i-1):1}(j) A\Xsmall}$ from $\cB(j)$ and resulting cancellation of the $\dhid$ terms.
    We now split the Frobenius norm of $\Phi(j) = \Philr(j) + \Phismall(j)$ into its ``well-conditioned" and ``small" parts.

    \paragraph{Bounding sum of low-rank parts}
    We can bound $\frobnorm{\Philr(j)}$ using $\cA(t)$:
    \begin{align*}
        \frobnorm{\Philr(j)} 
        \lesssim \left(1 - \frac{\eta L \sigma_{\min}^2(\Xlr)}{64 \din}\right)^{j} \frobnorm{\Philr(0)},
    \end{align*}
    and so by the geometric sum formula, we see that
    \begin{equation*}
        \sum_{j = 0}^{t-1} \frobnorm{\Philr(j)}
        \lesssim \frac{\din}{\eta L \sigma_{\min}^2(\Xlr)}\frobnorm{\Philr(0)}.
    \end{equation*}
    Using the bound $\frobnorm{\Philr(0)} \lesssim \sqrt{\frac{\dout}{\din}} \frobnorm{\Xlr}$ from \cref{lemma:initial-regression-error}, we obtain
    \begin{equation*}
        \eta  \din^{-\frac{1}{2}} \opnorm{X} \sum_{j = 0}^{t-1} \frobnorm{\Philr(j)}
        \lesssim \frac{\sqrt{\dout} \frobnorm{\Xlr} \opnorm{X}}{L \sigma^2_{\min}(\Xlr)}  
        = \frac{\Klr^2 \sqrt{d \sr(\Xlr)}}{L}.
    \end{equation*}

    \paragraph{Bounding small parts}
    We now bound $\Phismall(j)$. In \cref{lemma:Phismall bound phase 1} we show that combining the bound on $\frobnorm{F(j) - \call^j F(0)}$ from $\cB(j)$ and the assumption that $\Xsmall$ is sufficiently small in \cref{eq:Xsmall sufficiently small-appendix} gives
    \begin{equation*}
        \frobnorm{\Phismall(j)} 
        \lesssim \frac{\lambda}{L \opnorm{X}} \sqrt{\frac{\dout \sr(\Xlr)}{\din}}.
    \end{equation*}
    By summing over $j$ and bounding $t \lesssim \frac{m}{\eta \lambda}$ via \cref{corollary:length-of-step-1}, we obtain 
    \begin{equation*}
        \eta  \din^{-\frac{1}{2}} \opnorm{X} \sum_{j=0}^{t-1}\frobnorm{\Phismall(j)} 
        \lesssim \frac{\sqrt{\dout \sr(\Xlr)}}{L}.
    \end{equation*}
    \paragraph{Combining bounds}
    Combining the bounds on the low-rank and small parts of the sum with \eqref{eq:bound on deviation}, we obtain
    \begin{align*}
    \frobnorm{W_{i}(t) - \left(1 - \frac{\eta \lambda}{d_i}\right)^t W_{i}(0)}
    \lesssim
    \frac{\Klr^2 \sqrt{d \sr(\Xlr)}}{L}
    \end{align*}
    as desired.
\end{proof}
\begin{lemma}\label{lemma:Ct-implies-Bt-phase-1-main}
We have that $\cC(t) \implies
    \cB(t)$ for any $t \leq \tau$.
\end{lemma}
\begin{proof}\label{proof:Ct-implies-Bt-main-part}
In what follows, we consider a damped version $\widetilde W_k$ of the initialized weight matrix at each layer $k$.
    The properties at initialization established in \cref{sec:subsec:Properties at initialization} provide control over the norms and singular values of products involving $\widetilde W_{j:i} := \prod_{k=i}^j \widetilde W_k$, while the event $\cC(t)$ allows us to control the norm of the difference $\Delta_k$ between $W_k(t)$ and $\widetilde W_k$.
    Since each layer satisfies $W_k(t) = \widetilde W_k + \Delta_k$, we may expand the product $W_{j:i}(t)$ in order to bound the distance between 
    $W_{j:i}(t)$ and $\widetilde W_{j:i}$ in terms of a sum of quantities we can control, and consequently control the singular values of products involving $W_{j:i}(t)$ via Weyl's inequality. We now make the preceding argument precise.

    \paragraph{Expanding $W_{j:i}(t)$ in terms of matrices whose norms and spectrums we can control}
    Fix $t \le \tau$.
    Define 
    \begin{equation*}
        \widetilde{W}_k := \left(1-\frac{\eta \lambda}{d_{k}}\right)^{t} W_k(0)
        ~~\text{and}~~
        \Delta_k := W_k(t) - \widetilde{W}_k.
    \end{equation*}
    Given $1 \le i \le j \le L$, we may decompose $W_{j:i}(t)$ as
    \begin{align}\label{eq:Wi1-phase-1-derivation-B(t)}
      W_{j:i}(t) 
      = \prod_{k = i}^{j} 
      \left(\widetilde{W}_k + \Delta_{k} \right)
      = \widetilde{W}_{j:i}
        +
        \sum_{\myidx = 1}^{j-i+1}
        \sum_{\{k_{1}, \dots,  k_{\myidx}\} \subseteq \{i,\ldots,j\}}
        \widetilde{W}_{j:(k_{\myidx}+1)}
        \Delta_{k_{\myidx}} \dots \Delta_{k_1}
        \widetilde{W}_{(k_1 - 1):i}.
    \end{align}
    The double sum above is over all possible choices of $\myidx$ indices $\{k_{1}, \dots,  k_{\myidx}\}$ between $i$ and $j$ at which at a perturbation term $\Delta_{k_l}$ could appear when expanding the product.

    \paragraph{Bounding the spectrum of $W_{i:1}A\Xlr$ and $W_{i:1}A\Xsmall$}
    Let us now bound $\sigma_{\min}(W_{i:1}A\Xlr)$ and $\opnorm{W_{i:1}A\Xlr}$ for $1 \le i < L$.
    From \eqref{eq:Wi1-phase-1-derivation-B(t)}, it follows that 
    \begin{equation}
      \opnorm{W_{i:1}(t) A\Xlr) - \widetilde{W}_{i:1} A\Xlr}
      \leq \sum_{\myidx = 1}^{i} \sum_{\{k_{1}, \dots,  k_{\myidx}\} \subseteq \{1,\ldots,i\}} \opnorm{
        \widetilde{W}_{i:(k_{\myidx}+1)}
        \Delta_{k_{\myidx}} \dots \Delta_{k_1}
        \widetilde{W}_{(k_{1} - 1):1} A\Xlr
      }.
      \label{eq:singular-value-bound-phase-I}
    \end{equation}
    \subparagraph{Bounding $\opnorm{
        \widetilde{W}_{i:(k_{\myidx}+1)}
        \Delta_{k_{\myidx}} \dots \Delta_{k_1}
        \widetilde{W}_{(k_{1} - 1):1} A\Xlr
      }$}
 In each term of the right-hand side of \eqref{eq:singular-value-bound-phase-I}, there are $\myidx$ factors of the form $\Delta_{k_l}$, each of which satisfies 
    $\frobnorm{\Delta_{k_l}} \lesssim R$ by $\cC(t)$.
    There are also at most $\myidx$ factors of the form $\widetilde{W}_{b:a}$ with $1 < a \le b < L$, each of which satisfies $\opnorm{\widetilde{W}_{b:a}} \le \opnorm{W_{b:a}(0)} \lesssim \sqrt{L} \dhid^{\frac{b-a+1}{2}}$ by~\cref{lemma:norm-product-bounded}. 
    Finally, we have $\opnorm{\widetilde{W}_{(k_{1}-1):1} A\Xlr} \lesssim \dhid^{\frac{k_{1}-1}{2}} \cdot \sigma_{\max}(\Xlr)$
    by \cref{lemma:restricted-singular-values}.
    Combining these bounds and canceling out factors of $\dhid$ appropriately, we get
    \begin{align*}
    \opnorm{
      \widetilde{W}_{i:(k_{\myidx}+1)}
      \Delta_{k_{\myidx}} \dots \Delta_{k_1}
      \widetilde{W}_{(k_{1} - 1):1} A\Xlr
    }
    &\lesssim
    \left(O\left(R  \sqrt{\frac{L}{\dhid}}\right)\right)^{\myidx} \dhid^{\frac{i}{2}} \cdot \sigma_{\max}(\Xlr).
    \end{align*}
    \subparagraph{Summing over choices of indices $\{k_{1}, \dots,  k_{\myidx}\}$}
    Summing over all choices of indices $\{k_{1}, \dots, k_{\myidx}\}$ and using the bound $\binom{i}{\myidx} \leq i^\myidx \leq L^\myidx$ and a bound on geometric sums proved in  \cref{lemma:truncated-geometric-series}, we obtain 
    \begin{equation}  
        \sum_{\myidx = 1}^{i} \binom{i}{\myidx}
        \left(O\left(R \sqrt{\frac{L}{\dhid}}\right)\right)^{\myidx} 
        \leq \sum_{\myidx = 1}^{i} \left(O\left( R  L^{3/2} \sqrt{\frac{1}{\dhid}}\right)\right)^{\myidx} 
        \lesssim
        RL^{3/2} \sqrt{\frac{1}{\dhid}}.
    \label{eq:sum-bound-Bt}
    \end{equation}
    \subparagraph{Applying Weyl's inequality}
    Given the above, we obtain by Weyl's inequality that
    \begin{equation}
    \label{eq:bound on deviation from dampened}
        \abs{\sigma_{j}(W_{i:1}(t) A\Xlr) - \sigma_{j}(\widetilde{W}_{i:1} A\Xlr)}
        \le \opnorm{W_{i:1}(t) A\Xlr) - \widetilde{W}_{i:1} A\Xlr}
        \lesssim
        \sigma_{\max}(\Xlr) \dhid^{\frac{i-1}{2}}
        RL^{3/2}.
    \end{equation}
    In particular, by choosing $\dhid$ as
      \[
        \dhid \gtrsim  \frac{\Klr^2 R^2 L^3} {c_{\mathsf{b}}^2}
        \asymp
        L \dout \cdot \frac{\Klr^6 \sr(\Xlr)}{c_{\mathsf{b}}^2}
      \]
      where $c_{\mathsf{b}} > 0$ is a small constant to be chosen later, we may obtain
    \begin{equation}
    \label{eq:weyl bound}
        \abs{\sigma_{j}(W_{i:1}(t) A\Xlr) - \sigma_{j}(\widetilde{W}_{i:1} A\Xlr)} 
        \leq
        c_{\mathsf{b}} \cdot \sigma_{\min}(\Xlr) \cdot \dhid^{\frac{i}{2}}.
    \end{equation}
   Hence using the properties at initialization from \cref{lemma:restricted-singular-values} and \cref{eq:weyl bound}, we see that
      \begin{align*}
    \sigma_{\max}(W_{i:1}(t)A\Xlr) &\leq
    \sigma_{\max}(\widetilde{W}_{i:1} A\Xlr) + c_{\mathsf{b}} \cdot \sigma_{\min}(\Xlr) \cdot \dhid^{\frac{i}{2}} \\
                               &\leq
                               \left( \frac{6}{5} + c_{\mathsf{b}} \right) \sigma_{\max}(\Xlr) \cdot \dhid^{\frac{i}{2}} \\
                               &\leq
                               \frac{5}{4} \sigma_{\max}(\Xlr) \cdot \dhid^{\frac{i}{2}}
    \end{align*}
    by choosing $c_{\mathsf{b}}$ sufficiently small. This proves the second bound
    in the definition of $\cB(t)$.
    A nearly identical proof gives the bound 
    $\sigma_{\max}(W_{i:1}(t)A\Xsmall) \le \frac{5}{4} \sigma_{\max}(\Xsmall) \cdot \dhid^{\frac{i}{2}}$.
    In addition, \cref{lemma:restricted-singular-values} and \cref{eq:weyl bound} yield
    \begin{align*}
    \sigma_{\min}(W_{i:1}(t) A\Xlr) 
                                &\geq
                                \sigma_{\min}(W_{i:1}(0) A\Xlr) \cdot \prod_{k = 1}^{i} \left(1 - \frac{\eta \lambda}{d_{k}}\right)^{t}
    - c_{\mathsf{b}} \cdot \sigma_{\min}(\Xlr) \cdot \dhid^{\frac{i}{2}} \\
                                &\geq 
                                \left[ \frac{4}{5}\prod_{k = 1}^{i} \left(1 - \frac{\eta \lambda}{d_{k}}\right)^{t} - c_{\mathsf{b}} \right] \cdot \sigma_{\min}(\Xlr) \cdot \dhid^{\frac{i}{2}} \\
                                &\geq
                                \frac{3}{4} \dhid^{\frac{i}{2}}\sigma_{\min}(\Xlr),
    \end{align*}
    where the lower bound on the bracketed terms by $\nicefrac{3}{4}$ while $t \le \tau$ is proven in \cref{lem:bracketed terms}. This proves the fourth bound
    in the definition of $\cB(t)$.
    The remaining bounds in $\cB(t)$ follow from similar proofs and can be found in~\cref{proof:Ct-implies-Bt-supplement}.
\end{proof}
\subsection{Step 2: Reconstruction error stays small}
\label{sec:subsec:Step 2: he error stays small}
In \cref{sec:subsec:Step 1: Rapid early convergence} we have shown that after $\tau$ iterations our reconstruction error is small; namely,
$
\frobnorm{\Philr(\tau)} \leq \frac{C_1 \gamma}{L}  \frobnorm{\Xlr}.
$
We now want to show that the reconstruction error remains small until at least iteration $T$.
In particular, we will show that
$
\frobnorm{\Philr(\tau)} \leq   C_2 \gamma \frobnorm{\Xlr} 
$
for all $\tau \leq t \leq T$.
This we will show again by induction over the events stated in the following proposition.
\begin{proposition}\label{prop:step2-induction}
 Given $\tau$ defined in \eqref{eq:tau-def-phi} and $T$ defined in \eqref{eq:T def appendix}, then for all $\tau \leq t \leq T$ the following events hold with probability of at least $1-e^{-\Omega(d)}$ over the random initialization,
    \begin{subequations}
	\begin{align}
         \cA(t) & :=
         \set*{\frobnorm{\Philr(t)} \leq  C_2 \gamma \frobnorm{\Xlr} }
         \label{eq:induction-a}\\
		\cB(t) & := \left\{ \begin{array}{rclrr}
			            \sigma_{\max}(W_{j:i}(t))   & \lesssim &  \sqrt{L}d_w^{\frac{j-i+1}{2}}, \;\; \forall 1 < i \leq j < L                 \\
			            \sigma_{\max}(W_{i:1}(t) A\Xlr)  & \leq & \frac{9}{7} d_w^{\frac{i}{2}} \sigma_{\max}(\Xlr),  \;\; \forall 1 \leq i < L\\
                        \sigma_{\max}(W_{L:i}(t) ) &\leq& \frac{9}{7} d_w^{\frac{L-i+1}{2}} , \;\; \forall 1 < i \leq  L\\
                        \sigma_{\min}(W_{L:i}(t) ) & \geq& \frac{5}{7} d_w^{\frac{L-i+1}{2}}  , \;\; \forall 1 < i \leq  L\\
                        \sigma_{\max}(W_{i:1}(t) A \Xsmall)  & \leq & \frac{9}{7} d_w^{\frac{i}{2}} \sigma_{\max}(\Xsmall),  \;\; \forall 1 \leq i < L \\
                        \frobnorm{F(t) - \call^t F(0)}  &\lesssim& \din^{-\frac{1}{2}} \Delta_\infty L
                        \end{array} \right\} \label{eq:induction-b} \\
      \cC(t) & :=
      \set*{
        \frobnorm{W_{i}(t) - \left(1 - \frac{\eta \lambda}{d_i}\right)^{t - \tau} W_{i}(\tau)}
        \lesssim \Delta_{\infty}
      }\label{eq:induction-d-after-tao}\\
    \text{where} \;\; \Delta_{\infty} &:=
      \kappa^2(X_{\tlr}) \sqrt{\dout \sr(X_{\tlr})} \log(\dhid),
\end{align}
\end{subequations}
where $C_2 > 0$ is a universal constant.
\end{proposition}
The events in \cref{prop:step2-induction} are similar to those in the first phase. 
Notice that in \cref{eq:induction-b} we can no longer guarantee that the smallest singular value of $\sigma_{\min}(W_{i:1}A\Xlr)$ remains bounded away from zero.
We have also replaced the shrinkage condition in $\cA(t)$ with a condition that $\frobnorm{\Philr(t)}$ can only increase by roughly a factor of $O(L)$ from its value at time $\tau$.
In $\cC(t)$, instead of bounding the distance from $W_i(t)$ to a dampened version of $W_i(0)$, we now bound the distance to a dampened version of $W_i(\tau)$.
Finally, some of the constants in $\cB(t)$ are slightly worse than before. 
Since the proof of \cref{prop:step2-induction} is similar to that of \cref{prop:step1-induction}, it is presented in \cref{sec:proof of induction step 2}. 

\subsection{Step 3: Convergence off the subspace}
\label{sec:subsec:Step 3: Convergence off the subspace}
In this section, we show that the off-subspace error depends on the hidden width. 
To determine the behavior off the subspace, we must consider the projection onto ${\range(Y)}^{\perp}$, which we denote $P_{\range(Y)}^{\perp}.$
Note that $P_{\range(Y)}^{\perp} = P_{\ker(Y^{\T})}$ and so
\begin{align*}
	W_{1}(t+1) P_{\range(Y)}^{\perp} & =
	W_{1}(t)\left(1 - \frac{\eta \lambda}{\din}\right) P_{\range(Y)}^{\perp} -
	\eta \cdot \frac{1}{\sqrt{d_w^{L-1}m}} W_{L:2}^{\T} \Phi(t) Y^{\T} P_{\range(Y)}^{\perp}                    \\
	                        & =
	\left(1 - \frac{\eta \lambda}{\din}\right) W_{1}(t)P_{\range(Y)}^{\perp}    \\                                 
    &= \left(1 - \frac{\eta \lambda}{\din}\right)^{t+1} W_{1}(0) P_{\range(Y)}^{\perp}.
\end{align*}
By event $\cB(t)$ from~\cref{eq:induction-b}, we have
\begin{align*}
	\opnorm{W_{L:1}(t) P_{\range(Y)}^{\perp}}
    \leq
    \opnorm{W_{L:2}(t)} \opnorm{W_{1}(t) P_{\range(Y)}^{\perp}}
    \lesssim
    \dhid^{\frac{L - 1}{2}}
	\left(1 - \frac{\eta \lambda}{\din}\right)^t \opnorm{W_{1}{(0)} P_{\range(Y)}^{\perp}}.
	\label{eq:W-perp-bound-refine-i}
\end{align*}
Normalizing both sides, we obtain that the learned mapping $F(t) = \dhid^{-\frac{L-1}{2}} \din^{-\frac{1}{2}} W_{L:1}(t)$ satisfies
\begin{align}
	\opnorm{F(t) P_{\range(Y)}^{\perp}} 
    &\lesssim   \left(1 - \frac{\eta \lambda}{\din}\right)^t \frac{1}{\sqrt{\din}} \opnorm{W_{1}(0) P_{\range(Y)}^{\perp}}.
\end{align}
Using the bound $\opnorm{W_{1}(0) P_{\range(Y)}^{\perp}} \lesssim \sqrt{\dhid}$ from \cref{lem:off subspace at init}, we see that
\begin{equation*}
	\opnorm{F(T) P_{\range(Y)}^{\perp}}
	 \lesssim \left(1 - \frac{\eta \lambda}{\din}\right)^{T} \sqrt{\frac{\dhid}{\din}}         
	 \leq
	 \exp\left(-\frac{T \eta \lambda}{\din}
    + \frac{1}{2} \log(\dhid)
	\right)                                                       
    = \dhid^{-\frac{3}{2}},
	\label{eq:W-perp-bound-refined}
\end{equation*} 
where the second inequality follows from $1 - x \leq \exp(-x)$
and the last inequality follows from the choice of 
$T = \frac{2m\log(\dhid)}{\eta\lambda}$.

\section{Conclusions and future directions}
\label{sec:discussion}
\paragraph{Adaptivity of nonlinear networks to nonlinear structure}
This work proves that neural networks can automatically adapt to subspace structure in training data when trained with weight decay using standard initialization and hyperparameter schemes; future work includes extending this framework beyond the simplified setting we consider here to explore nonlinear networks and more complex models of training data structure. 
For example, a natural next step would be to study what kind of low-dimensional structure neural networks with one ReLU layer can automatically adapt to when trained with $\ell_2$-regularized gradient descent.
As shown in~\cref{fig:wd-robustness-nonlinear}, preliminary experiments suggest that this procedure can lead to robustness in the setting of data from a union of subspaces model, which is known to closely approximate patches from real-world images \cite{mobahi2009data}.
\paragraph{Benefit of depth}
Previous work has suggested a significant advantage of deeper networks over shallower ones in adapting to low-rank structure \cite{dai2021,parkinson2023linear,Jacot23}. 
Consistent with these findings, preliminary experiments in \cref{sec:subsec:depth} suggest that depth is beneficial for both the regression and the ``off-subspace'' errors: in~\cref{fig:errors-by-depth} we see that larger depth, at least up to a certain point, leads to faster convergence.  This phenomenon is not covered by our main theoretical result, but constitutes an interesting direction for future work.

\section*{Acknowledgments}
SP gratefully acknowledges the support of the NSF Graduate Research Fellowship Program NSF DGE-2140001 and the American Association of University Women. VC and RW gratefully acknowledge the support of NSF DMS-2023109, the NSF-Simons National Institute for Theory and Mathematics in Biology (NITMB) through NSF (DMS-2235451) and Simons Foundation (MP-TMPS-00005320), and the Margot and Tom Pritzker Foundation. FK gratefully acknowledges the support of the German Science Foundation (DFG) in the context of the priority program Theoretical Foundations of Deep Learning
(project KR 4512/6-1). HL and FK gratefully acknowledge the support of the Munich Center for Machine Learning (MCML).

\appendix
\section{Supplementary lemmas for the proof of \texorpdfstring{\cref{thm:mainresult-formal}}{the main result}}
\subsection{Proofs for Preliminary Derivations}
\begin{lemma}
    \label{lemma:exact-rank-X}
    Without loss of generality, we may assume that $X = RZ$ for some $Z \in \Rbb^{s \times s}$ with $\rank(Z) = s$.
\end{lemma}
\begin{proof}
	Since $X = RZ$ where $Z \in \Rbb^{s \times n}$ and $n \geq s$, the economic SVD of $Z$ yields
	\[
		X = R U_{Z} \Sigma_{Z} V_{Z}^{\T}, \quad
		U_{Z} \in O(s), \; \Sigma_{Z} = \diag(\sigma_1, \dots, \sigma_{s}),
		\; V_{Z} \in O(n, s).
	\]
	Since the Frobenius norm is unitarily invariant,
	\begin{align*}
		\frobnorm{\Phi}
		 & = \frobnorm{d_w^{-\frac{L-1}{2}}m^{-\frac{1}{2}} W_{L:1}Y - X}
		\\&= \frobnorm{(d_w^{-\frac{L-1}{2}}m^{-\frac{1}{2}}W_{L:1}AR - R) U_{Z} \Sigma_{Z} V_{Z}^\T}
		\\&= \frobnorm{d_w^{-\frac{L-1}{2}}m^{-\frac{1}{2}}W_{L:1}ARU_{Z}\Sigma_{Z} - RU_{Z} \Sigma_{Z}}.
	\end{align*}
	Moreover, the gradient of the reconstruction error satisfies
	\begin{align*}
		& \grad_{W_i} \left[\frac{1}{2}\frobnorm{\Phi}^2 \right]
		 \\
         & = d_w^{-\frac{L-1}{2}}m^{-\frac{1}{2}} W_{L:i+1}^{\T} \Phi Y^{\T} W_{i-1:1}^{\T}                                                                                                                              \\
		 & = d_w^{-\frac{L-1}{2}}m^{-\frac{1}{2}} W_{L:i+1}^{\T} (d_w^{-\frac{L-1}{2}}m^{-\frac{1}{2}} W_{L:1}A - I)R U_{Z} \Sigma_{Z} \underbrace{V_{Z}^{\T} V_{Z}}_{I_{s}} \Sigma_{Z} U_{Z}^{\T} R^{\T} A^{\T} W_{i-1:1}^{\T} \\
		 & = d_w^{-\frac{L-1}{2}}m^{-\frac{1}{2}} W_{L:i+1}^{\T} (d_w^{-\frac{L-1}{2}}m^{-\frac{1}{2}}  W_{L:1}A - I)R U_{Z} \Sigma_{Z}^2 U_{Z}^{\T} R^{\T} A W_{i-1:1}^{\T}.
	\end{align*}
	Thus, without loss of generality, we can assume that $X = R U_{Z} \Sigma_{Z} \in \Rbb^{\dout \times s}$ since this assumption does not change the gradient descent trajectory or value of the loss function.
\end{proof}
\begin{lemma}
	\label{lemma:difference-norm}
  For any $i \in [L]$, any $t \in \mathbb{N}$, and any matrix norm $\norm{\cdot}$, we have
	\begin{align*}
     & \norm{W_{i}(t)- \left(1-\frac{\eta \lambda}{d_i} \right)^t W_{i}(0)} \\
		 & \leq
		\eta d_w^{-\frac{L-1}{2}}m^{-\frac{1}{2}} \sum_{j=0}^{t-1}
		\left(1 - \frac{\eta \lambda}{d_i}\right)^{t-1-j}
		\norm{W_{L:i+1}(j)^{\T} \Phi(j) (W_{i-1:1}(j) Y)^{\T}}.
		\label{eq:difference-norm-scaled}
	\end{align*}
\end{lemma}
\begin{proof}
	The proof follows from the update formula for $W_{i}$ in \cref{eq:Wi-update}.
	Writing
	\begin{equation}
		B_{t} := d_w^{-\frac{L-1}{2}}m^{-\frac{1}{2}} W_{L:i+1}(t)^{\T} \Phi(t) Y^{\T} W_{i-1:1}(t)^{\T},
	\end{equation}
	we rewrite~\cref{eq:Wi-update} as the recursion
	\begin{align*}
		W_i(t) & = \left(1 - \frac{\eta \lambda}{d_i}\right) W_i(t-1) - \eta B_{t-1} \notag \\
		       & = \left(1 - \frac{\eta \lambda}{d_i}\right)^2 W_i(t-2)
		- \eta \left(1 - \frac{\eta \lambda}{d_i}\right) B_{t-2} - \eta B_{t-1} \notag      \\
		       & \qquad \qquad \vdots \notag                                                \\
		       & = \left(1 - \frac{\eta \lambda}{d_i}\right)^t W_i(0)
		- \eta \sum_{j=0}^{t-1} \left(1 - \frac{\eta \lambda}{d_i}\right)^{t-1-j} B_{j}.
	\end{align*}
	Rearranging, taking norms and applying the triangle inequality yields the result.
\end{proof}
\begin{lemma}
\label{lem:error-evolution-II-AX}
    We can bound the reconstruction error $\Philr$ at iteration $t+1$ as follows:
    \begin{align}
	\frobnorm{\Philr(t+1)} 
  & \leq
	\begin{aligned}[t]
    & \left(1 - \eta \left(\lambda_{\min}(\Plr(t)) - \opnorm{\Psmall(t)} \right)\right)
	\frobnorm{\Philr(t)} \\
	&+ \callbound \frobnorm{\Ulr(t)} 
	+ \eta \opnorm{P(t)} \frobnorm{\Phismall(t)}+ \frobnorm{E(t)},
    \end{aligned}
	\label{eq:error-evolution-II-AX}
\end{align}
as long as $\eta \le \frac{1}{\lambda_{\max}(P(t))}$, where $P,\Plr,\Psmall$ and $E$ are defined in \cref{eq:def-E,eq:def-p,eq:def-pr,eq:def-psmall}.
\end{lemma}

\begin{proof}
\label{sec:evolution of res}
We right-multiply
both sides of~\eqref{eq:prod-evolution-I} by $d_w^{-\frac{L-1}{2}}m^{-\frac{1}{2}} A\Xlr$ to obtain
\begin{align*}
	\Ulr(t+1)
	  = &\call \Ulr(t) + \Elr(t)
	\\&- \eta d_w^{-(L-1)}m^{-1}
	  \sum_{i=1}^L
	\cprod{i}
	W_{L:i+1}(t) W_{L:i+1}(t)^{\T} \Phi(t) Y^{\T} W_{i-1:1}^{\T}(t) W_{i-1:1}(t)(A\Xlr).
\end{align*}
Vectorizing and using the property $\vect(AXB) = (B^{\T} \otimes A) \vect(X)$, we get
\begin{align}
	\vect (\Ulr(t+1))
	= \call \vect (\Ulr(t))
	-\eta P(t) \vect(\Phi(t))
	+ \vect(E(t)).
	\label{eq:prod-evolution-I-vec-AX}
\end{align}
We subtract $\vect(\Xlr)$ from both sides of
\cref{eq:prod-evolution-I-vec-AX}; the result is equal to
\begin{align}
\notag
	\vect(\Philr(t+1))
	 & =
	\begin{aligned}[t]
		 & \call \vect (\Ulr(t))
		- \vect(\Xlr)
		-\eta (\Plr(t)+\Psmall(t)) \vect(\Philr(t))\\
        &-\eta P(t) \vect(\Phismall(t))
		+ \vect(E(t))
	\end{aligned} \\
	 & =
	\begin{aligned}[t]
		 & \left(\call - 1\right) \vect (\Ulr(t))
		+ (I-\eta (\Plr(t)+\Psmall(t))) \vect(\Philr(t))\\
    &-\eta  P(t) \vect(\Phismall(t))
		+ \vect(E(t))
	\end{aligned}
	\label{eq:error-evolution-I-AX}
\end{align}
Taking the norm on both sides of~\eqref{eq:error-evolution-I-AX}
and using the triangle inequality and the bound $\left|\call-1\right| \le \frac{2\eta\lambda}{m}$ from~\cref{lemma:one-minus-folded-product} yields
\begin{align*}
	\frobnorm{\Philr(t+1)}
	 & \leq
    \begin{aligned}[t]
	& \opnorm{I - \eta(\Plr(t)+\Psmall(t))} \frobnorm{\Philr(t)}
	+ \left|\call-1\right| \frobnorm{\Ulr(t)}\\
    &+  \eta \opnorm{P(t)} \frobnorm{\Phismall(t)}
	+ \frobnorm{E(t)}
    \end{aligned} \\
	 & \leq
	\begin{aligned}[t]
    & \left(1 - \eta \lambda_{\min}(\Plr(t))\right)
	\frobnorm{\Philr(t)} + \eta \opnorm{\Psmall(t)} \frobnorm{\Philr(t)}\\
	&+\callbound  \frobnorm{\Ulr(t)}
	+ \eta \opnorm{P(t)} \frobnorm{\Phismall(t)} + \frobnorm{E(t)}
    \end{aligned} \\
   &=
	\begin{aligned}[t]
    & \left(1 - \eta \left(\lambda_{\min}(\Plr(t)) - \opnorm{\Psmall(t)} \right)\right)
	\frobnorm{\Philr(t)} \\
	&+\callbound  \frobnorm{\Ulr(t)}
	+ \eta \opnorm{P(t)} \frobnorm{\Phismall(t)}+ \frobnorm{E(t)},
    \end{aligned}
\end{align*}
as long as $\eta \le \frac{1}{\lambda_{\max}(P(t))}$.
\end{proof}
We now furnish bounds on the spectrum of $P$ in terms of the spectrum of $W_{L:i+1}$ and $W_{i-1:1} Y$, for $i = 1 \ldots L$. In the following lemma, we drop the time index $t$ for
simplicity.
\label{sec:bounding spectrum of p}
\begin{lemma}
	\label{lemma:P-k-spectrum}
	We have the following inequalities:
	\begin{align}
		\lambda_{\max}(\Plr)
		 & \leq d_w^{-(L-1)}m^{-1}
		\sum_{i=1}^L
		\sigma_{\max}^2(W_{i-1:1} A\Xlr) \sigma_{\max}^2(W_{L:i+1}); \label{eq:lambda-max-pk} \\
		\lambda_{\min}(\Plr)
		 & \geq \frac{1}{4}d_w^{-(L-1)}m^{-1}
		\sum_{i=1}^L
		\sigma_{\min}^2(W_{i-1:1} A\Xlr) \sigma_{\min}^2(W_{L:i+1}). \label{eq:lambda-min-pk-AX}\\
         \opnorm{\Psmall} 
 &\leq d_w^{-(L-1)}m^{-1} \sum_{i=1}^L \opnorm{
	 W_{i-1:1} A\Xsmall} \opnorm{W_{i-1:1}A\Xlr}
	\opnorm{
	W_{L:i+1}}^2
    \end{align}
\end{lemma}
\begin{proof}
	The inequalities are straightforward to prove using the definition of $\Plr$ and $\Psmall$ in \cref{eq:def-pr,eq:def-psmall,eq:def-p} and the following facts:
	\begin{enumerate}
		\item The largest (or smallest) eigenvalue of a sum of matrices is bounded above (or below) by the sum of the largest (or smallest) eigenvalues.
		\item The eigenvalues of a Kronecker product are the products of the eigenvalues of the individual factors. The same holds for the operator norm of a Kronecker product.
		\item For any matrix $A$, $\lambda_{\max}(A^{\T} A) = \sigma_{\max}^2(A)$.
        \item The operator norm of a matrix is sub-multiplicative.
	\end{enumerate}
	Using these facts and the bound $\frac{1}{4} \le \cprod{i} \le 1$ for all $i = 1, \ldots, L$ proven in \cref{lemma:cprod-i-lb}, the result is immediate.
\end{proof}

\subsection{Proofs of Properties at Initialization}
\label{sec:prop at init proofs}
Here we prove the properties at initialization which are stated in \cref{sec:subsec:Properties at initialization}.
\begin{proof}[Proof of \cref{lemma:restricted-singular-values}]
\label{proof:lemma:restricted-singular-values}
Similar to Proposition 6.3 in \cite{du2019width}, we obtain with probability of at least $1-\exp\left(-\Omega(\dhid/L)\right)$ that 
\begin{align*}
    \sigma_{\max}(W_{i:1}(0)\cdot A \Xlr) &\leq 1.1 d_w^\frac{i}{2} \sigma_{\max}(A\Xlr) \\
    \sigma_{\min}(W_{i:1}(0)\cdot A \Xlr) &\geq 0.9 d_w^\frac{i}{2} \sigma_{\min}(A\Xlr).
\end{align*}
By~\cref{lem:RIP bounds on sv}, we may use Assumption~\ref{assumption:rip} to bound the spectrum of $A\Xlr$ as
\begin{align*}
    \sigma_{\max}(A\Xlr) &\le \sqrt{1+\delta}\sigma_{\max}(\Xlr) \\
    \sigma_{\min}(A\Xlr) &\ge \sqrt{1-\delta}\sigma_{\min}(\Xlr).
\end{align*}
Therefore, using $\delta = \frac{1}{10}$, we obtain
\begin{align*}
    \sigma_{\max}(W_{i:1}(0)\cdot A \Xlr) &\leq 1.2 d_w^\frac{i}{2} \sigma_{\max}(\Xlr) \\
    \sigma_{\min}(W_{i:1}(0)\cdot A \Xlr) &\geq 0.8 d_w^\frac{i}{2} \sigma_{\min}(\Xlr)
\end{align*}
as desired. The bound on $\sigma_{\max}(W_{i:1}(0)\cdot A \Xsmall)$ is similar.
Further, an argument similar to that of Proposition 6.2 in \cite{du2019width} provides the bound on $\sigma_{\max}(W_{i:1}(0))$.
\end{proof}

\begin{proof}[Proof of \cref{lemma:initial-regression-error}]
\label{proof:lemma:initial-regression-error}
    We first show that $\opnorm{W_{L:1}(0)\bar{U}} \lesssim \dhid^{\frac{L-1}{2}}\dout^{\frac{1}{2}}$ where $\bar{U} \in \R^{\din \times s}$ is a fixed matrix with orthonormal columns.
    To that end,
    we invoke~\cref{lemma:wide-gaussian-prod-tail} (which is a slight generalization of Lemma 6.1 from \cite{du2019width}) with
    \begin{align*}
        A_{1} = W_{1}(0) \bar{U} ~\text{and}~  A_{i} = W_i ~\text{for}~ 1<i\le L \\
        \text{with}~n_{0} = s, n_i =  \dhid ~\text{for}~ 1\le i< L, ~\text{and}~ n_{L} = \dout.
    \end{align*}
    For these choices, the failure probability in~\cref{lemma:wide-gaussian-prod-tail} will depend on the term
    \begin{equation}
        \label{eq:fail prob}
        \sum_{i = 1}^L \frac{1}{n_i}
        = \frac{L-1}{\dhid} + \frac{1}{\dout} 
        \lesssim \frac{1}{\dout},
    \end{equation}
    because $\dhid \gtrsim L \cdot \dout$. Indeed,
    \cref{lemma:wide-gaussian-prod-tail} yields (for any fixed $y \in \Rbb^s$):
    \begin{equation}
        \prob{\left|\norm{W_{L:1}(0)Uy}^2 - \dout \cdot \dhid^{L - 1} \norm{y}^2\right|
        \geq \frac{1}{10} \dout \cdot \dhid^{L-1} \norm{y}^2}
        \leq \exp(-\Omega(\dout)).
    \end{equation}
    Taking an $\varepsilon$-net of $\mathbb{S}^{s-1}$ and proceeding as in the proof of Proposition 6.2 in \cite{du2019width}, we obtain
    \[
        \opnorm{W_{L:1}(0)\bar{U}} \lesssim  \dhid^{\frac{L-1}{2}} \dout^{\frac{1}{2}}
    \]
    with probability at least $1-\exp(-\Omega(d))$.

	Now let $\bar{U} \bar{\Sigma} \bar{V}^{\T}$ denote the economic SVD of $A\Xlr$, with $\bar{U} \in \R^{m \times s}$. 
    By the unitary invariance of the Frobenius norm and a bound on the spectrum of $AR$ from Assumption~\ref{assumption:rip} (see~\cref{lem:RIP bounds on sv}), we get that
    \begin{equation*}
        \frobnorm{\bar{\Sigma} \bar{V}^{\T}}
        = \frobnorm{A\Xlr}
        \le \opnorm{AR} \frobnorm{\Zlr}
        \lesssim \frobnorm{\Zlr}
        = \frobnorm{\Xlr}.
    \end{equation*}
    Thus,
	\begin{align*}
		\frobnorm{\Philr(0)}
		&= \frobnorm{\dhid^{-\frac{L-1}{2}}m^{-\frac{1}{2}} W_{L:1}(0)A\Xlr - \Xlr} \\
        &\leq \dhid^{-\frac{L-1}{2}} \din^{-\frac{1}{2}}
        \opnorm{W_{L:1}(0)\bar{U}}
        \frobnorm{\bar{\Sigma} \bar{V}^{\T}} + \frobnorm{\Xlr} \\
        &\lesssim \sqrt{\frac{\dout}{\din}}
        \frobnorm{\Xlr}.
	\end{align*}
	By taking the economic SVD of $A\Xsmall$ instead, we similarly get $\frobnorm{\Usmall(0)} \lesssim \sqrt{\frac{\dout}{\din}}\frobnorm{\Xsmall}$.
\end{proof}
\begin{proof}[Proof of \cref{lem:off subspace at init}]
Let $V_{\perp} \in O(m, m-s)$ be a matrix whose columns span ${\range(Y)}^{\perp}$;
by orthogonal invariance of the operator norm and
the Gaussian distribution, we have
\begin{align*}
    \opnorm{W_{1}(0) P_{\range(Y)}^{\perp}} &=
    \opnorm{W_{1}(0) V_{\perp} V_{\perp}^{\T}} =
    \opnorm{W_{1}(0) V_{\perp}},
\end{align*}
where $W_{1}(0) V_{\perp} \in \Rbb^{\dhid \times (\din - s)}$ is a matrix with standard Gaussian elements.
Therefore by~\cite[Corollary 7.3.3]{Ver18}, the following holds
with probability $1-\exp(-\Omega(\dhid^2))$:
\begin{align*}
	\opnorm{W_{1}(0) P_{\range(Y)}^{\perp}} & \leq
    2\sqrt{\dhid} + \sqrt{\din - s}
    \lesssim \sqrt{\dhid}.
\end{align*}
\end{proof}

\subsection{Supplementary Lemmas for \texorpdfstring{\cref{prop:step1-induction}}{Step 1}}
\begin{lemma}
    \label{lemma:Phismall bound phase 1}
    If $t \le \tau$ and $\cB(t)$ holds, then
    \begin{equation}
        \frobnorm{\Phismall(t)} 
        \lesssim \frac{\gamma \frobnorm{\Xlr}}{L^2 \Klr^2} 
        \le \frac{\lambda \frobnorm{\Xlr}}{L \opnorm{X}^2} \sqrt{\frac{\dout}{\din}}
    \end{equation}
\end{lemma}
\begin{proof}
    By the definition of $\Usmall$ in \eqref{eq:U and Phi def}, we have that
    \begin{equation*}
        \frobnorm{\Usmall(t) - \call^t \Usmall(0)}
        \le \frobnorm{F(t) - \call^t F(0)} \opnorm{A\Xsmall}
        \lesssim \din^{-\frac{1}{2}}RL \opnorm{\Xsmall}
    \end{equation*}
    where the second inequality comes from $B(t)$ and \cref{lem:RIP bounds on sv}.
    In \cref{lemma:initial-regression-error}, we provide the bound $\frobnorm{\Usmall(0)} \lesssim \sqrt{\frac{\dout}{\din}} \frobnorm{\Xsmall}$.
    Hence
    \begin{equation*}
        \frobnorm{\Usmall(t)}
        \le \frobnorm{\Usmall(t) - \call^t \Usmall(0)} + \call^t \frobnorm{\Usmall(0)} \lesssim \din^{-\frac{1}{2}}RL \frobnorm{\Xsmall}.
    \end{equation*}
    It follows that
    \begin{equation*}
        \frobnorm{\Phismall(t)}
        = \frobnorm{\Usmall(t) - \Xsmall}
        \le \frobnorm{\Usmall(t)} + \frobnorm{\Xsmall}
        \lesssim \din^{-\frac{1}{2}}RL \frobnorm{\Xsmall}.
    \end{equation*}
    Because $\Xsmall$ has rank at most $s-r$, it follows that
    \begin{equation}
    \label{eq:Phismall bound}
        \frobnorm{\Phismall(t)}
        \lesssim \sqrt{\frac{s-r}{\din}} RL \opnorm{\Xsmall}.
    \end{equation}

    From the definition of $R$ in \cref{prop:step1-induction}, it is straightforward to verify that
    \begin{equation*}
        \frac{\frobnorm{\Xlr}}{RL} = \frac{\sigma_{\min}(\Xlr)}{\sqrt{d}\Klr}.
    \end{equation*}
    Using the bound on $\Xsmall$ from \cref{eq:Xsmall sufficiently small-appendix} (and dropping the $\log(\dhid)$ term from the denominator), we see that
    \begin{equation}
        \opnorm{\Xsmall} 
        \lesssim \frac{\gamma }{L^2 \Klr^2} \cdot \frac{\sigma_{\min}(\Xlr)}{\sqrt{d}\Klr} \cdot \sqrt{\frac{m}{s-r}}
        = \frac{\gamma \frobnorm{\Xlr}}{L^2 \Klr^2 RL} \sqrt{\frac{m}{s-r}}.
    \end{equation}
    Combining this with \cref{eq:Phismall bound} gives us
    \begin{equation*}
        \frobnorm{\Phismall(t)}
        \lesssim \frac{\gamma \frobnorm{\Xlr}}{L^2 \Klr^2}.
    \end{equation*}
    Using $\lambda = \gamma \sigma_{\min}^2(\Xlr) \sqrt{\frac{\din}{\dout}}$ from \cref{eq:main-thm-assumptions}, we see that
    \begin{equation*}
        \frac{\gamma \frobnorm{\Xlr}}{L^2 \Klr^2}
        = \frac{\lambda}{\sigma_{\min}^2(\Xlr)}\cdot \frac{\frobnorm{\Xlr}}{L^2 \Klr^2}  \sqrt{\frac{\dout}{\din}} 
        \le \frac{\lambda \frobnorm{\Xlr}}{L \opnorm{X}^2} \sqrt{\frac{\dout}{\din}}
    \end{equation*}
    as desired.
\end{proof}
\begin{lemma}
    \label{lemma:Bt-implies-bound-on-E}
    Fix $t \leq \tau$ and suppose that $\{\cA(j)\}_{j \leq t - 1}$ and $\cB(t)$ hold.
    Then
    \begin{equation}
        \frobnorm{\Elr(t)}:=\frobnorm{d_{w}^{-\frac{L-1}{2}} \din^{-\frac{1}{2}}  E_0(t) A\Xlr}
        \leq  \frac{\eta L \sigma_{\min}^2(\Xlr)}{64 m}
        \cdot \frobnorm{\Phi(t)}.
    \end{equation}
\end{lemma}
\begin{proof}
Define $\overline{W}_i := (1 - \nicefrac{\eta \lambda}{d_i}) W_i(t).$
Note that each term
in $E_0(t)$ is the product of $2$ or more factors of the form
$\grad_{W_i} \frac{1}{2} \frobnorm{\Phi}^2$
and $L-2$ or fewer factors of the form 
$\overline{W}_i$.
When $\ell$ of these factors are from the former category, there are $\binom{L}{\ell}$
ways to choose their indices $(s_1, \dots, s_{\ell})$. Each such choice induces
a term $C_{(s_1, \dots, s_{\ell})}$, defined by
\begin{align*}
  C_{(s_1, \dots, s_{\ell})}
  := \eta^{\ell} 
  \overline{W}_{L:(s_{\ell}+1)} \left( \grad_{W_{s_{\ell}}} \frac{1}{2} \frobnorm{\Phi}^2 \right)
  \overline{W}_{(s_{\ell}-1):(s_{\ell-1} + 1)} 
  \cdots
  \left( \grad_{W_{s_1}} \frac{1}{2} \frobnorm{\Phi}^2 \right)
  \overline{W}_{(s_1-1):1}.
\end{align*}
By $\cB(t)$, each factor of the form $\grad_{W_k} \frac{1}{2} \frobnorm{\Phi}^2$ satisfies
\begin{align*}
  \frobnorm{\grad_{W_k} \frac{1}{2} \frobnorm{\Phi}^2} 
  &\leq d_{w}^{-\frac{L-1}{2}} \din^{-\frac{1}{2}}
  \opnorm{W_{L:(k+1)}(t)} \frobnorm{\Phi(t)} \opnorm{W_{(k-1):1}(t) Y} \notag \\
  &\lesssim d_{w}^{-\frac{L-1}{2}} \din^{-\frac{1}{2}}
  \cdot d_{w}^{\frac{L - k}{2}} \frobnorm{\Phi(t)}
  d_{w}^{\frac{k-1}{2}} \opnorm{X} \notag \\
  &= \frac{1}{\sqrt{\din}} \frobnorm{\Phi(t)} \opnorm{X}.
\end{align*}
Also by $\cB(t)$, the factors $\overline{W}_{b:a}$ with $1 < a \le b < L$ satisfy
$
\opnorm{\overline{W}_{b:a}}
\lesssim \sqrt{L} \dhid^{\frac{b - a + 1}{2}}.
$
We also get 
$
\opnorm{\overline{W}_{(s_{1}-1):1} A\Xlr}
\lesssim \dhid^{\frac{s_1-1}{2}} \opnorm{X}
$
and
$
\opnorm{\overline{W}_{L:s_{\ell} + 1}}
\lesssim \dhid^{\frac{L - s_{\ell}}{2}}.
$
Consequently, $C_{(s_1, \dots, s_{\ell})}A\Xlr$
admits the following bound:
\begin{align*}
   \frobnorm{C_{(s_1, \dots, s_{\ell})}A\Xlr} 
  \lesssim
  \left( O\left(\frac{\eta}{\sqrt{\din}} \frobnorm{\Phi(t)} \opnorm{X}\right)\right)^\ell
  \left[
    d_{w}^{\frac{s_1 - 1}{2}} \opnorm{X}
    \cdot \dhid^{\frac{L - s_{\ell}}{2}}
    \prod_{k = 1}^{\ell-1}
    \sqrt{L} d_{w}^{\frac{s_{k+1} - s_{k} - 1}{2}} 
  \right].
\end{align*}
Note that the factor in brackets above equals
\begin{equation*}
  d_{w}^{\frac{s_1 - 1}{2}} \opnorm{X}
    \cdot \dhid^{\frac{L - s_{\ell}}{2}}
    \prod_{k = 1}^{\ell-1}
    \sqrt{L} d_{w}^{\frac{s_{k+1} - s_{k} - 1}{2}} 
    = \opnorm{X}
  \left(\sqrt{L}\right)^{\ell-1} \cdot
  d_{w}^{\frac{L - \ell}{2}},
  \label{eq:Et-decomp-last-term-bound}
\end{equation*}
and so
\begin{equation}
    \frobnorm{C_{(s_1, \dots, s_{\ell})}A\Xlr} 
    \lesssim
    \left( 
        O\left(\frac{\eta \sqrt{L} \frobnorm{\Phi(t)} \opnorm{X}}{\sqrt{\din \dhid}} \right)
    \right)^\ell
    \opnorm{X}
    L^{-\frac{1}{2}} \cdot
    d_{w}^{\frac{L}{2}}.
\end{equation}
Summing over all possible choices of $(s_1, \dots, s_{\ell})$ and using the fact that $\binom{L}{\ell} \leq L^\ell$, we obtain
\begin{align}
    \frobnorm{\Elr(t)} 
    &= \frobnorm{d_{w}^{-\frac{L-1}{2}} \din^{-\frac{1}{2}} E_{0}(t) A\Xlr} \notag\\
    &\lesssim \opnorm{X}\sqrt{\frac{\dhid}{L \din}} 
        \sum_{\ell = 2}^{L}
        \binom{L}{\ell} \left( 
            O\left(\frac{\eta \sqrt{L} \frobnorm{\Phi(t)} \opnorm{X}}{\sqrt{\din \dhid}} \right)
        \right)^\ell \notag\\
    &\le \opnorm{X}\sqrt{\frac{\dhid}{L \din}} 
        \sum_{\ell = 2}^{L}
        \left( 
            O\left(\frac{\eta L^{3/2} \frobnorm{\Phi(t)} \opnorm{X}}{\sqrt{\din \dhid} } \right)
        \right)^\ell \notag\\
    &\lesssim 
        \frac{\eta L \frobnorm{\Phi(t)} \opnorm{X}^2}{\din} 
        \sum_{\ell = 1}^{L-1}
        \left( 
            O\left(\frac{\eta L^{3/2} \frobnorm{\Phi(t)} \opnorm{X}}{\sqrt{\din \dhid}} \right)
        \right)^\ell. \label{eq:E sum bound}
\end{align}
where in the last line we used the fact that 
$\sum_{\ell=2}^L \alpha^\ell = \alpha \sum_{\ell=1}^{L-1} \alpha^\ell$.
We now bound the sum.
Using $\cA(0), \dots, \cA(t-1)$ and the bounds $\frobnorm{\Philr(0)} \lesssim \sqrt{\frac{d}{m}} \frobnorm{\Xlr}$ from~\cref{lemma:initial-regression-error} and 
\begin{equation*}
    \frobnorm{\Phismall(t)} 
    \lesssim \frac{\gamma \frobnorm{\Xlr}}{L^2 \Klr^2} \le \frobnorm{\Xlr}
\end{equation*} 
from \cref{lemma:Phismall bound phase 1}, we see that
\begin{equation}
\label{eq:Phi bound}
    \frobnorm{\Phi(t)} 
    \le \frobnorm{\Philr(t)} + \frobnorm{\Phismall(t)} 
    \lesssim \frobnorm{\Philr(0)} + \frobnorm{\Xlr} 
    \lesssim \sqrt{\frac{d}{m}} \frobnorm{\Xlr}.
\end{equation}
Using~\cref{eq:Phi bound} and the assumption in \cref{eq:main-thm-assumptions} that $\eta \le \frac{m}{L \sigma^2_{\max}(X)}$, we see that
\begin{equation}
    \frac{\eta L^{3/2} \frobnorm{\Phi(t)} \opnorm{X}}{\sqrt{\din \dhid}}
    \lesssim \frac{\sqrt{Ld}\frobnorm{\Xlr}}{\sqrt{\dhid}\opnorm{X}}
    = \sqrt{\frac{Ld \sr(\Xlr)}{\dhid}}.
\end{equation}
From the bound on geometric sums from \cref{lemma:truncated-geometric-series}, it follows that
\begin{equation*}
    \sum_{\ell = 1}^{L-1}
        \left( 
            O\left(\frac{\eta L^{3/2} \frobnorm{\Phi(t)} \opnorm{X}}{\sqrt{\din \dhid}} \right)
        \right)^\ell
    \lesssim \sqrt{\frac{Ld \sr(\Xlr)}{\dhid}}.
\end{equation*}
By choosing $\dhid \gtrsim Ld \sr(\Xlr) \Klr^4$, we therefore ensure that $\frobnorm{E(t)} \le \frac{\eta L \sigma_{\min}^2(\Xlr)}{64 m}
        \cdot \frobnorm{\Phi(t)}$ as desired.
\end{proof}
\begin{lemma}
  \label{lemma:supplement-Pt-extreme-eigenvalues}
  Under the event $\cB(t)$ in~\cref{eq:event-B-prestop}, the following inequalities hold:
  \begin{align*}
  \lambda_{\min}(\Plr(t))
		  \geq \frac{L\sigma_{\min}^2(\Xlr) }{16m}, \ \
    \opnorm{\Psmall(t)} \lesssim \frac{L\opnorm{\Xsmall} \opnorm{X} }{m}, \ \
     \opnorm{P(t)} \lesssim \frac{L \opnorm{X}^2 }{m}.
  \end{align*}
\end{lemma}
\begin{proof}
  \cref{lemma:P-k-spectrum} implies that
  \begin{align*}
    \lambda_{\min}(\Plr(t)) &\geq \frac{1}{4 d_{w}^{L-1} \din} \sum_{i = 1}^L \sigma_{\min}^2(W_{L:(i+1)}(t)) \sigma_{\min}^2(W_{(i-1):1}(t)A \Xlr) \\
                         &\geq
                         \frac{1}{4 d_w^{L-1} \din} \sum_{i = 1}^L \left(\frac{3}{4} d_{w}^{\frac{L - i}{2}}\right)^2 \left(\frac{3}{4} d_{w}^{\frac{i-1}{2}} \sigma_{\min}(\Xlr) \right)^2 \\ 
                        &\ge \frac{L \sigma_{\min}^2(\Xlr)}{16 \din},
  \end{align*}
  where the second inequality uses event $\cB(t)$ in~\cref{eq:event-B-prestop}.
 Similarly, we have
  \begin{align*}
    \lambda_{\max}(\Plr(t)) &\leq
    \frac{1}{d_{w}^{L-1} \din} \sum_{i = 1}^L \sigma_{\max}^2(W_{L:(i+1)}(t)) \sigma_{\max}^2(W_{(i-1):1}(t)A\Xlr) \\
                         &\lesssim
    \frac{1}{d_{w}^{L-1} \din} \sum_{i = 1}^L
    \left(d_w^{\frac{L - i}{2}}\right)^2
    \left(d_w^{\frac{i-1}{2}} \sigma_{\max}(\Xlr) \right)^2 \\ 
    &= \frac{L \opnorm{X}^2}{\din}.
  \end{align*}
  We also have
  \begin{align*}
      \opnorm{\Psmall(t)} 
 &\leq d_w^{-(L-1)}m^{-1} \sum_{i=1}^L \opnorm{
	 W_{i-1:1} A\Xsmall} \opnorm{W_{i-1:1}(t)A\Xlr}
	\opnorm{
	W_{L:i+1}}^2 \\
 & \lesssim d_w^{-(L-1)}m^{-1} \sum_{i=1}^L  \left(\dhid^{\frac{i-1}{2}}\right)^2 \sigma_{\max}(\Xsmall) \sigma_{\max}(\Xlr) \left(\dhid^{\frac{L-i}{2}}\right)^2\\
 &= \frac{L\sigma_{\max}(\Xsmall) \sigma_{\max}(\Xlr) }{m}.
  \end{align*}
  Finally, 
  \begin{equation*}
      \opnorm{P(t)} 
      \le \opnorm{\Plr(t)} + \opnorm{\Psmall(t)} 
      \lesssim \frac{L \opnorm{X}^2}{\din}
  \end{equation*}
  which completes the proof.
\end{proof}
\subsubsection{Proof of \texorpdfstring{\cref{lemma:Ct-implies-Bt-phase-1-main}}{B(t)}}
\label{proof:Ct-implies-Bt-supplement}
\begin{proof}[Continued proof of \cref{lemma:Ct-implies-Bt-phase-1-main}]
  Let us now bound $\opnorm{W_{j:i}(t)}$.
  \paragraph{Bounding $\opnorm{W_{j:i}(t)}$}
  Fix $1 < i \le j < L$.
  We consider the terms 
  \begin{equation}
  \label{eq:delta-prods}
      \widetilde{W}_{j:(k_{\myidx}+1)}
        \Delta_{k_{\myidx}} \dots \Delta_{k_1}
        \widetilde{W}_{(k_1 - 1):i}
  \end{equation}
  from the righthand side of of \cref{eq:Wi1-phase-1-derivation-B(t)}.
  There are $\myidx$ factors of the form $\Delta_{k_l}$ with $\frobnorm{\Delta_{k_l}} \lesssim R$ by $\cC(t)$. 
  There are at most $\myidx+1$ factors of the form $\widetilde{W}_{b:a}$ with $1 < a \le b < L$, each of which satisfies $\opnorm{\widetilde{W}_{b:a}} \lesssim \sqrt{L} \dhid^{\frac{b-a+1}{2}}$ by~\cref{lemma:norm-product-bounded}. 
  Combining these bounds and canceling out factors of $\dhid$ appropriately, we get
  \begin{equation}
     \opnorm{\widetilde{W}_{j:(k_{\myidx}+1)} \Delta_{k_{\myidx}}(t) \dots \Delta_{k_{1}}(t) \widetilde{W}_{(k_1 - 1):i}}  
     \lesssim
     \sqrt{L} \dhid^{\frac{j - i + 1}{2}} 
    \left( O\left(\frac{R \sqrt{L}}{\sqrt{d_w}} 
     \right)\right)^{\myidx}.
  \label{eq:term-bound-Bt2}
  \end{equation}
  Akin to \eqref{eq:sum-bound-Bt}, we sum over all choices of indices $\{k_{1}, \dots, k_{\myidx}\}$ and $\myidx = 1, \ldots, j-i+1$ to get
  \begin{equation}
      \sum_{\myidx = 1}^{j - i + 1} \binom{j - i + 1}{\myidx} 
    \left( O\left(\frac{R \sqrt{L}}{\sqrt{d_w}} 
     \right)\right)^{\myidx}
      \lesssim \frac{R L^{3/2}}{\sqrt{d_w}}.
  \label{eq:sum-bound-Bt2}
  \end{equation}
  By choosing
  $\dhid \gtrsim R^2 L^{3} 
  $
  and using \cref{eq:Wi1-phase-1-derivation-B(t),eq:term-bound-Bt2,eq:sum-bound-Bt2}, we see that
  \begin{equation}
  \label{eq:deviation-bound-Bt2-ji}
      \opnorm{W_{j:i}(t) - \widetilde{W}_{j:i}} \lesssim \sqrt{L} \dhid^{\frac{j - i + 1}{2}}.
  \end{equation}
  Since \cref{lemma:norm-product-bounded} implies that 
  $\opnorm{\widetilde{W}_{j:i}} \lesssim \sqrt{L} \dhid^{\frac{j - i + 1}{2}}$
  as well, we arrive at
  \begin{equation*}
    \opnorm{W_{j:i}(t)} \lesssim \sqrt{L} \dhid^{\frac{j - i + 1}{2}}.
  \end{equation*}
  This proves the
  first bound in the definition of $\cB(t)$.
  
  \paragraph{Bounding $\sigma_{\min}(W_{L:i})$ and $\opnorm{W_{L:i}}$}
  Fix $i > 1$.
  This time, we consider products of the form \eqref{eq:delta-prods} with $j = L$. Now, there is one factor of the form $\widetilde{W}_{L:(k_{\myidx}+1)}$ that satisfies $\opnorm{\widetilde{W}_{L:(k_{\myidx}+1)}} \lesssim \dhid^{\frac{L-k_{\myidx}}{2}}$ (\cref{lemma:restricted-WL_singular_values}), and at most $\myidx$ factors of the form $\widetilde{W}_{b:a}$ with $1 < a \le b < L$ satisfying $\opnorm{\widetilde{W}_{b:a}} \lesssim \sqrt{L} \dhid^{\frac{b-a+1}{2}}$ (\cref{lemma:norm-product-bounded}). 
  This means that we may save a factor of $\sqrt{L}$ in comparison to \eqref{eq:term-bound-Bt2} and obtain by similar arguments as before (e.g., this time we must ensure $\dhid \gtrsim \sfrac{R^2 L^3}{c_{\mathsf{b}}^2}$) that
  \begin{equation}
  \label{eq:deviation-bound-Bt2-Li}
      \opnorm{W_{L:i}(t) - \widetilde{W}_{L:i}} \le c_{\mathsf{b}} \dhid^{\frac{L - i + 1}{2}}
  \end{equation}
  where $c_{\mathsf{b}}$ is a small positive constant to be chosen later.
  By Weyl's inequality and \cref{lemma:restricted-WL_singular_values}, we see that
  \begin{equation*}
      \sigma_{\max}(W_{L:i}(t))
      \leq \sigma_{\max}(\widetilde{W}_{L:i}) +
      c_{\mathsf{b}} \cdot \dhid^{\frac{L - i + 1}{2}} \leq \left( \frac{6}{5} + c_{\mathsf{b}} \right) \cdot \dhid^{\frac{L - i + 1}{2}} \leq \frac{5}{4} \cdot \dhid^{\frac{L - i + 1}{2}},
  \end{equation*}
  by choosing $c_{\mathsf{b}}$ sufficiently small. This proves the third inequality
  in $\cB(t)$.
  Similarly, using \cref{lem:bracketed terms} we get the lower bound
  \begin{align*}
     \sigma_{\min}(W_{L:i}(t)) &\geq 
     \sigma_{\min}(W_{L:i}(0)) \cdot \prod_{k=i}^L \left(1 - \frac{\eta \lambda}{d_k}\right)^t
     - c_{\mathsf{b}} \cdot \dhid^{\frac{L - i + 1}{2}} \\
     &\geq
     \left[\frac{4}{5} \prod_{k=i}^L \left(1 - \frac{\eta \lambda}{d_k}\right)^t - c_{\mathsf{b}}\right] \cdot \dhid^{\frac{L - i + 1}{2}} \\
     &\geq \frac{3}{4} \cdot \dhid^{\frac{L - i + 1}{2}}.
  \end{align*}
  \paragraph{Bounding $\frobnorm{F(t) - \call^t F(0)}$}
  We consider products of the form
  \begin{equation*}
      \widetilde{W}_{L:(k_{\myidx}+1)}
        \Delta_{k_{\myidx}} \dots \Delta_{k_1}
        \widetilde{W}_{(k_1 - 1):1}.
  \end{equation*}
  There is one factor of the form $\widetilde{W}_{L:(k_{\myidx}+1)}$ that satisfies $\opnorm{\widetilde{W}_{L:(k_{\myidx}+1)}} \lesssim \dhid^{\frac{L-k_{\myidx}}{2}}$ (\cref{lemma:restricted-WL_singular_values}), and at most $\myidx-1$ factors of the form $\widetilde{W}_{b:a}$ with $1 < a \le b < L$ satisfying $\opnorm{\widetilde{W}_{b:a}} \lesssim \sqrt{L} \dhid^{\frac{b-a+1}{2}}$ (\cref{lemma:norm-product-bounded}). 
  Finally, there is one factor of the form $\widetilde{W}_{(k_1 - 1):1}$ satisfying $\opnorm{\widetilde{W}_{(k_1 - 1):1}} \lesssim \dhid^{\frac{k_1 - 1}{2}}$ (\cref{lemma:restricted-singular-values}). Thus
  \begin{equation}
  \label{eq:phismall factor bound}
        \frobnorm{\widetilde{W}_{L:(k_{\myidx}+1)}
        \Delta_{k_{\myidx}} \dots \Delta_{k_1}
        \widetilde{W}_{(k_1 - 1):1}}
        \lesssim \frac{\dhid^{\frac{L}{2}}}{\sqrt{L}}\left(O\left(\frac{R\sqrt{L}}{\sqrt{\dhid}}\right)\right)^\myidx
  \end{equation}
  Similar to \eqref{eq:sum-bound-Bt}, we sum over all choices of indices $\{k_{1}, \dots, k_{\myidx}\}$ and $\myidx = 1, \ldots, L$ to get
  \begin{equation}
  \label{eq:phismall sum over idx bd}
      \sum_{\myidx = 1}^{L} \binom{L}{\myidx}
        \left(O\left(R \sqrt{\frac{L}{\dhid}}\right)\right)^{\myidx} 
        \leq \sum_{\myidx = 1}^{L} \left(O\left( R  L^{3/2} \sqrt{\frac{1}{\dhid}}\right)\right)^{\myidx} 
        \lesssim
        RL^{3/2} \sqrt{\frac{1}{\dhid}}.
  \end{equation}
  Using the decomposition in \cref{eq:Wi1-phase-1-derivation-B(t)} and combining \cref{eq:phismall factor bound,eq:phismall sum over idx bd}, we get that
  \begin{equation*}
      \frobnorm{W_{L:1}(t) - \widetilde{W}_{L:1}}
      \lesssim RL\dhid^{\frac{L-1}{2}}.
  \end{equation*}
  Using the definition of 
  $
      F(t) = d_w^{-\frac{L-1}{2}}m^{-\frac{1}{2}} W_{L:1}(t)
  $
  from \cref{eq:F def} and $\call$ from \eqref{eq:c-prod-i}, it follows that
  \begin{equation}
  \label{eq:F deviation from init}
      \frobnorm{F(t) - \call^t F(0)}
      \lesssim m^{-\frac{1}{2}} RL.
  \end{equation}
  This proves the final inequality
  making up the event $\cB(t)$.
\end{proof}
\begin{lemma}
\label{lem:bracketed terms}
    There is a small positive constant $c_{\mathsf{b}} > 0$ such that for any $t \leq \tau$ and $1 \le i \le j \le L$, we have
    \begin{equation}
        \frac{4}{5}\left(\prod_{k=i}^j \left(1-\frac{\eta \lambda}{d_{k}}\right)^{t}\right) - c_{\mathsf{b}} \ge \frac{3}{4}.
    \end{equation}
\end{lemma}
\begin{proof}
    In \cref{lemma:contraction-factor-small-phase-1} we lower bound each factor in the product by $1-\frac{1}{20L}$. In combination with \cref{lemma:small-contraction-factor-lower-bound}, this implies that
    \begin{equation}
        \left(\prod_{k=i}^j \left(1-\frac{\eta \lambda}{d_{k}}\right)^{t}\right)
        \ge \left(1-\frac{1}{20L}\right)^{j-i+1}
        \ge \left(1-\frac{1}{20L}\right)^L
        \ge \frac{19}{20}.
    \end{equation}
    Hence
    \begin{equation}
        \frac{4}{5}\left(\prod_{k=i}^j \left(1-\frac{\eta \lambda}{d_{k}}\right)^{t}\right) - \frac{1}{100}
        \ge \frac{4}{5} \cdot \frac{19}{20} - \frac{1}{100}
        = \frac{3}{4}
    \end{equation}
    as desired.
\end{proof}
\begin{lemma}
  \label{lemma:contraction-factor-small-phase-1}
    For any $t \leq \tau$, we have
    \[
        \left(1 - \frac{\eta \lambda}{d_i} \right)^t \geq 1 - \frac{1}{20L}.
    \]
\end{lemma}
\begin{proof}
    From~\cref{thm:weierstrass}, it follows that
    \[
        \left( 1 - \frac{\eta \lambda}{d_i} \right)^{t} \geq
        1 - \frac{t \eta \lambda}{d_i}
        \geq 1 - \frac{\tau \eta \lambda}{d_i}.
    \]
    Using the bound on $\tau$ from~\cref{corollary:length-of-step-1}, the fact that $\din \le d_i$, and the bound $x \log(\nicefrac{1}{x}) \le \sqrt{x}$ shown in \cref{lem:xlog1/x le sqrt x}, we have
    \begin{equation}
        \frac{\tau \eta \lambda}{d_i} 
        \leq \frac{64 \din}{\eta L \sigma_{\min}^2(\Xlr) } \log\left(
          \frac{
          L \sigma_{\min}^2(\Xlr)}{\lambda}  
        \right) \cdot \frac{\eta \lambda}{d_i}
        \le
        64 \sqrt{
        \frac{\lambda}{L \sigma_{\min}^2(\Xlr)} 
        }.
        \label{eq:lem-contraction-factor-small-phase-1-bnd by sqrt x}
    \end{equation}
    The assumption that $\gamma \le 10^{-7}\cdot\frac{\sqrt{d}}{L\sqrt{m}}$ implies that 
    \begin{equation}
    \label{eq:lambda very small}
        \frac{\lambda}{L\sigma_{\min}^2(\Xlr)} 
        = \frac{\gamma}{L} \sqrt{\frac{\din}{\dout}}
        \le \frac{10^{-7}}{L^2}.
    \end{equation}    
    Using \cref{eq:lem-contraction-factor-small-phase-1-bnd by sqrt x,eq:lambda very small}, we obtain
    \begin{equation}
        \frac{\tau \eta \lambda}{d_i} 
        \leq 
        \frac{64 \cdot 10^{-3.5}}{L}
        \le \frac{1}{20L}.
    \end{equation}
    Therefore $\left(1 - \frac{\eta \lambda}{d_i} \right)^t \geq 1 - \frac{1}{20L}$ as desired.
\end{proof}
\begin{lemma}
    \label{lemma:small-contraction-factor-lower-bound}
    For any $L \geq 2$, we have that
    \[
        \left(1 - \frac{1}{20 L} \right)^{L}
        \geq 0.95.
    \]
\end{lemma}
\begin{proof}
    The function $x \mapsto \left(1 - \frac{1}{20x}\right)^{x}$ is monotone increasing for
    all $x \geq 1$. Therefore, 
    \[
        \left(1 - \frac{1}{20L}\right)^{L} \geq \left( 1 - \frac{1}{40} \right)^{2}
        = \left( \frac{39}{40} \right)^2 > 0.95.
    \]
\end{proof}

\subsection{Proof of \texorpdfstring{\cref{prop:step2-induction}}{Step 2}}
\label{sec:proof of induction step 2}
Similarly to the proof of \cref{prop:step1-induction}, we prove \cref{prop:step2-induction} by induction. 
For the base case, observe that $\cB(\tau)$ holds by \cref{prop:step1-induction}. 
The event $\cA(\tau)$ holds because of the definition of $\tau$ by choosing $C_2 \ge \sfrac{C_1}{L}$. 
The event $\cC(\tau)$ trivially holds.
For the inductive step, we will show that while $t \le T$,
\begin{itemize}
  \item $\cA(t)$ and $\cB(t)$ imply $\cA(t+1)$ (\cref{lemma:A proof part 2});
  \item $\set{\cA(j), \cB(j)}_{\tau \le j < t}$ imply $\cC(t)$ (\cref{lemma:C proof part 2});
  \item $\cC(t)$ implies $\cB(t)$ (\cref{lemma:B proof part 2}).
\end{itemize}

Throughout this section, we will need $\dhid$ to satisfy the following:
\begin{equation}
  \label{eq:dhid-lb-step-2}
  \dhid \gtrsim \Delta_{\infty}^2 L^3 \asymp L^3 \dout \kappa^4(X_{\tlr}) \sr(X_{\tlr}) \log(d_w)^2.
\end{equation}
\begin{lemma}
\label{lemma:A proof part 2}
    $\cA(t)$ and $\cB(t)$ imply $\cA(t+1)$.
\end{lemma} 
\begin{proof}
Using \cref{lemma:P-k-spectrum} and $\cB(t)$, an argument nearly identical to the proof of \cref{lemma:supplement-Pt-extreme-eigenvalues} gives us the bounds
\begin{equation*}
\opnorm{P(t)} \lesssim \frac{L  \opnorm{X}^2 }{m} ~~\text{and}~~
\opnorm{\Psmall(t)} \lesssim \frac{L\opnorm{\Xsmall} \opnorm{X} }{m}.
\end{equation*}
Because we no longer have control over $\sigma_{\min}(W_{i:1}(t) A\Xlr)$ for $t \ge \tau$, we must modify our bound on $\lambda_{\min}(\Plr(t))$. 
By discarding terms corresponding to $i>1$ in the bound on $\lambda_{\min}(\Plr(t))$ from \cref{lemma:P-k-spectrum}, we see that
\begin{equation*}
    \lambda_{\min}(\Plr(t))
    \ge \frac{\sigma_{\min}^2(A\Xlr) \sigma_{\min}^2(W_{L:2})}{4d_w^{L-1}m} 
    \ge \frac{\sigma_{\min}^2(\Xlr)}{4m} \cdot \frac{9}{10} \left(\frac{5}{7}\right)^2 
    \ge \frac{\sigma_{\min}^2(\Xlr)}{16m}
\end{equation*}
where the second inequality above comes from $\cB(t)$ and the bound $\sigma_{\min}^2(A\Xlr) \ge 0.9 \sigma_{\min}^2(\Xlr)$, which comes from the Restricted Isometry Property in Assumption~\ref{assumption:rip} (see \cref{lem:RIP bounds on sv}).
Notice that the lower bound on $\lambda_{\min}(\Plr(t))$ now is worse than that from Phase 1 by a factor of $L$.

Similarly to the first $\tau$ iterations, in \cref{lem:phase2:proof:e(t)} we bound the higher-order terms as 
$\frobnorm{E(t)} \leq \frac{ \eta \sigma_{\min}^2(\Xlr) \frobnorm{\Phi(t)}}{64 \din}$. Further, \cref{eq:Xsmall sufficiently small-appendix} implies that
$L\opnorm{\Xsmall} \opnorm{X} \lesssim \sigma_{\min}^2(\Xlr)$.
By $\cA(t)$, we see that 
\begin{equation*}
    \frobnorm{\Ulr(t)} \le \frobnorm{\Philr(t)} + \frobnorm{\Xlr} \lesssim \frobnorm{\Xlr}.
\end{equation*}
In \cref{lemma:Phismall bound phase 2}, we show that $\cB(t)$ implies that $\Phismall(t)$ is bounded as
\begin{equation*}
    \frobnorm{\Phismall(t)}
    \lesssim \lambda \sqrt{\frac{\dout}{\din}} \frac{\frobnorm{\Xlr}}{L \opnorm{X}^2}.
\end{equation*}
Akin to \cref{eq:loss-evolution-decomposition}, 
we plug in the above to the bound on $\Philr(t+1)$ in \cref{eq:error-evolution-II-AX-main} to see that
\begin{align*}
\frobnorm{\Philr(t+1)}
\leq
\left(1- \eta  \frac{\sigma_{\min}^2(\Xlr)}{32m}  \right) \frobnorm{\Philr(t)} + 
\frac{\eta \lambda}{\din} O\left(\sqrt{\frac{\dout}{\din}}
\frobnorm{\Xlr}\right).
\end{align*}
Finally, using 
$\frobnorm{\Philr(t)} \leq C_2 \cdot \gamma \frobnorm{\Xlr}$ from $\cA(t)$ and 
$\lambda = \gamma \sigma^2_{\min}(\Xlr) \sqrt{\frac{\din}{\dout}}$ from \cref{eq:main-thm-assumptions}, we see that
\begin{align*}
    \frobnorm{\Phi_{\tlr}(t+1)}
    &\leq \left(1- \eta  \frac{\sigma_{\min}^2(\Xlr)}{32\din} \right)
    C_2 \cdot \gamma \frobnorm{\Xlr}+ 
    O\left(\frac{\eta \gamma \sigma^2_{\min}(\Xlr)}{\din}
    \frobnorm{\Xlr}\right)  \\
    &= C_2 \gamma \frobnorm{X_{\tlr}} \left(
    1 - \frac{\eta \sigma_{\min}^2(X_{\tlr})}{\din} \cdot \left[
    \frac{1}{32} - O\left(\frac{1}{C_2}\right) \right]
    \right) \\
    & \leq C_2 \gamma \frobnorm{\Xlr},
\end{align*}
by choosing $C_2$ large enough.
This shows that the event $\cA(t+1)$ holds.
\end{proof}

\begin{lemma}
\label{lemma:C proof part 2}
    $\set{\cA(j), \cB(j)}_{\tau \le j < t}$ implies $\cC(t)$.
\end{lemma}
\begin{proof}
    Similar to \cref{eq:bound on deviation} in \cref{lemma:Bj-implies-Ct-phase-1}, 
  we obtain by $\cB(j)$ that
  \begin{align*}
     \frobnorm{W_{i}(t) - \left(1 - \frac{\eta \lambda}{d_i}\right)^{t - \tau}
    W_{i}(\tau)} 
    \lesssim
    \frac{ \eta \opnorm{X}}{\sqrt{\din}}
    \sum_{j = 0}^{t - \tau - 1}
    \frobnorm{\Phi(j + \tau)}. 
  \end{align*}
  Using the bound on $\Philr$ from $\cA(j)$ and the bound on $\Phismall$ from~\cref{lemma:Phismall bound phase 2},
  we see that 
    \begin{align*}
    \sum_{j = 0}^{t - \tau - 1}
    \frobnorm{\Phi(j + \tau)} 
     &\leq
    \sum_{j = 0}^{t - \tau - 1}
    \left(\frobnorm{\Philr(j + \tau)}+\frobnorm{\Phismall(j + \tau)} \right)
   \lesssim T \gamma \frobnorm{\Xlr} 
    \end{align*}
    Given the above and the definition of $T=\frac{2m \log(\dhid)}{\eta \lambda} \lesssim \frac{\sqrt{m} \sqrt{d} \log(\dhid)}{\eta \sigma_{\min}^2(\Xlr) \gamma}$ we can bound
    \begin{align*}
    \frac{ \eta \opnorm{X}}{\sqrt{\din}}
    \sum_{j = 0}^{t - \tau - 1}
    \frobnorm{\Phi(j + \tau)} 
    \lesssim \frac{\sqrt{d} \log(\dhid) \opnorm{X}\frobnorm{\Xlr}}{\sigma_{\min}^2(\Xlr)}
    = \log(d_w) \sqrt{\sr(\Xlr)}\Klr^2 \sqrt{d}.
    \end{align*}
\end{proof}

\begin{lemma}
\label{lemma:B proof part 2}
  Fix $t \in [\tau, T]$. Then $\cC(t) \implies \cB(t)$.
\end{lemma}
\begin{proof}
The proof is similar to that of \cref{lemma:Ct-implies-Bt-phase-1-main}, but now we define
\begin{equation*}
    \widetilde{W}_\ell := \left(1 - \frac{\eta \lambda}{d_{\ell}}\right)^{t - \tau} W_{\ell}(\tau)
    ~~\text{and}~~
    \Delta_\ell := W_\ell(t) - \widetilde{W}_\ell.
\end{equation*}
The event $\cC(t)$ implies that $\frobnorm{\Delta_\ell} \lesssim \Delta_{\infty}$, while \cref{prop:step1-induction} provides bounds on the spectrum of $\widetilde{W}_{j:i}$.

Fix $1 \le i < L$. 
An argument similar to the proof of \cref{lemma:Ct-implies-Bt-phase-1-main} implies that
\begin{equation*}
    \opnorm{W_{i:1}(t)A\Xlr - \widetilde{W}_{i:1}A\Xlr} \lesssim  \abs{\sigma_{j}(W_{i:1}(t) A\Xlr) - \sigma_{j}(\widetilde{W}_{i:1} A\Xlr)} 
        \lesssim
        \sigma_{\max}(\Xlr) \dhid^{\frac{i-1}{2}}
        \Delta_{\infty}L^{3/2},
\end{equation*}
akin to \cref{eq:bound on deviation from dampened}.
Let $c_{\mathsf{b}}$ be a small positive constant to be chosen later.
By choosing
\begin{equation*}
    \dhid \gtrsim \frac{\Delta_{\infty}^2 L^3}{c_{\mathsf{b}}^2} \asymp   \frac{L^3 \dout \kappa^4(X_{\tlr}) \sr(X_{\tlr}) \left(\log(d_w)\right)^2}{c_{\mathsf{b}}^2} 
\end{equation*} 
we see that
\begin{equation*}
    \opnorm{W_{i:1}(t)A\Xlr - \widetilde{W}_{i:1}A\Xlr} \le c_{\mathsf{b}}\sigma_{\max}(\Xlr) \dhid^{\frac{i}{2}},
\end{equation*}
similar to \cref{eq:weyl bound}.
It follows that
\begin{align*}
    \opnorm{W_{i:1}(t)A\Xlr} 
    &\leq  \opnorm{\widetilde{W}_{i:1}A\Xlr} + \opnorm{W_{i:1}(t)A\Xlr - \widetilde{W}_{i:1}A\Xlr} \\
    &\leq  \left(\frac{5}{4} + c_{\mathsf{b}}\right)\dhid^{\frac{i}{2}} \opnorm{X} \\
&\leq \frac{9}{7}  d_w^{\frac{i}{2}} \opnorm{X}
\end{align*}
for $c_{\mathsf{b}}$ sufficiently small.
We omit the argument for the bound on $\opnorm{W_{i:1}(t) A X_{\tsm}}$ since it is completely analogous.

Now fix $1 < i \le j < L$. Similar to \cref{eq:term-bound-Bt2,eq:sum-bound-Bt2,eq:deviation-bound-Bt2-ji}, we get
\begin{equation*}
     \opnorm{W_{j:i}(t) - \widetilde{W}_{j:i}} 
     \lesssim \sqrt{L} \dhid^{\frac{j-i+1}{2}} 
\end{equation*}
by choosing $\dhid \gtrsim \Delta_{\infty}^2 L^3$.
  Hence,
  \begin{equation*}
    \opnorm{W_{j:i}(t)} 
    \leq \opnorm{\widetilde{W}_{j:i}} + \opnorm{W_{j:i}(t) - \widetilde{W}_{j:i}} 
    \lesssim \sqrt{L} d_w^{\frac{j-i+1}{2}}.
  \end{equation*}
  This proves the first bound in the event $\cB(t)$. 

  Continuing with $W_{L:i}(t)$ for $1 < i \le L$, we get a bound similar to \cref{eq:deviation-bound-Bt2-Li} that
  \begin{equation*}
    \opnorm{W_{L:i}(t) - \widetilde{W}_{L:i}}
    \leq c_{\mathsf{b}}  d_w^{\frac{L-i+1}{2}}
  \end{equation*}
  by choosing $\dhid \gtrsim \sfrac{\Delta_{\infty}^2 L^3}{c_{\mathsf{b}}^2}$.
  By choosing $c_{\mathsf{b}}$ sufficiently small, it follows that
  \begin{equation*}
    \opnorm{W_{L:i}(t)} 
    \le \opnorm{\widetilde{W}_{L:i}} + \opnorm{W_{L:i}(t) - \widetilde{W}_{L:i}} 
    \le \left(\frac{5}{4} + c_{\mathsf{b}}\right)  d_w^{\frac{L-i+1}{2}}
    \le \frac{9}{7}  d_w^{\frac{L-i+1}{2}},
  \end{equation*}
  which proves the second bound from the event $\cB(t)$.
  In addition,
\begin{align*}
    \sigma_{\min}(W_{L:i}(t)) 
    &\geq  \sigma_{\min}(\widetilde{W}_{L:i}) - 
    \opnorm{W_{L:i}(t) - \widetilde{W}_{L:i}}  \\
    &\geq  \left[\frac{3}{4} \prod_{\ell = i}^L \left(1 - \frac{\eta \lambda}{d_{\ell}}\right)^{t - \tau}
    - c_{\mathsf{b}}\right] d_w^{\frac{L-i+1}{2}} \\
    &\ge \frac{5}{7} \dhid^{\frac{L - i + 1}{2}},
\end{align*}
where the lower bound on the bracketed terms above by $\sfrac{5}{7}$ while $\tau \le t \le T$ is shown in \cref{lem:bracketed terms-II}.
This proves the third bound from $\cB(t)$.

Finally, similar to \cref{eq:F deviation from init}, we get that
\begin{equation*}
    \frobnorm{F(t) - \call^{t-\tau} F(\tau)}
      \lesssim m^{-\frac{1}{2}} \Delta_\infty L.
\end{equation*}
By \cref{prop:step1-induction}, we know that 
\begin{equation*}
    \frobnorm{F(\tau) - \call^{\tau} F(0)}
      \lesssim m^{-\frac{1}{2}} R L
\end{equation*}
as well.
Thus
\begin{equation*}
    \frobnorm{F(t) - \call^{t} F(0)}
    \le \frobnorm{F(t) - \call^{t-\tau} F(\tau)}
    + \call^{t-\tau} \frobnorm{F(\tau) - \call^\tau F(0)} 
    \lesssim m^{-\frac{1}{2}} \Delta_\infty L
\end{equation*}
as desired since $R \le \Delta_\infty$.
\end{proof}
\subsubsection{Supplementary Lemmas for the Proof of \texorpdfstring{\cref{prop:step2-induction}}{Step 2}}
\begin{lemma}
\label{lem:bracketed terms-II}
    There is a small positive constant $c_{\mathsf{b}} > 0$ such that for any $\tau \le t \le T$ and $1 < i \le L$, we have
    \begin{equation*}
        \frac{3}{4}
        \prod_{\ell = i}^L \left(1 - \frac{\eta \lambda}{d_{\ell}}\right)^{t - \tau} 
        - c_{\mathsf{b}} 
        \ge \frac{5}{7}.
    \end{equation*}
\end{lemma}
\begin{proof}
Because $d_{\ell} = \dhid$ for $\ell > 1$ and $1 - x \geq \exp(-2x)$ for $x \le \frac{1}{2}$, we see that
\begin{equation*}
    \prod_{\ell = i}^L \left(1 - \frac{\eta \lambda}{d_{\ell}}\right)^{t - \tau} 
    = \left(1 - \frac{\eta \lambda}{\dhid}\right)^{(t - \tau)(L-i+1)} 
    \geq \exp\Big(-2  \cdot \frac{(t - \tau)(L-i+1)\eta \lambda}{\dhid}\Big).
\end{equation*}
Since $t - \tau \leq \nicefrac{ 2\log(\dhid) \cdot \din}{\eta \lambda}$ and $L-i+1 \le L$, it follows that
\begin{equation*}
    \exp\Big(-2  \cdot \frac{(t - \tau)(L-i+1)\eta \lambda}{\dhid}\Big)
    \ge \exp\Big(-\frac{4\din L \log(\dhid)}{\dhid}\Big).
\end{equation*}
Thus by choosing $\dhid \gtrsim \din L \log(\dhid)$, we may ensure that 
\begin{equation*}
    \prod_{\ell = i}^L \left(1 - \frac{\eta \lambda}{d_{\ell}}\right)^{t - \tau} \ge 0.99
\end{equation*}
and therefore that
\begin{equation*}
        \frac{3}{4}
        \prod_{\ell = i}^L \left(1 - \frac{\eta \lambda}{d_{\ell}}\right)^{t - \tau} - 0.01
        \ge \frac{5}{7}
    \end{equation*}
as desired.
\end{proof}
\begin{lemma}
\label{lemma:Phismall bound phase 2}
    Assume that $t \in [\tau,T]$ and $\cB(t)$ holds. Then 
    \begin{equation*}
        \frobnorm{\Phismall(t)}
        \lesssim \frac{\gamma \frobnorm{\Xlr}}{L \Klr^2}
        = \frac{\lambda \frobnorm{\Xlr}}{L \opnorm{X}^2} \sqrt{\frac{\dout}{\din}}
    \end{equation*}
\end{lemma}
\begin{proof}
    Similar to \cref{eq:Phismall bound}, we see that $B(t)$ implies that
    \begin{equation}
    \label{eq:Phismall bound phase II}
        \frobnorm{\Phismall(t)}
        \lesssim \sqrt{\frac{s-r}{\din}} \Delta_\infty L \opnorm{\Xsmall}
    \end{equation}
    for $t \in [\tau, T]$.
    From the definition of $\Delta_{\infty}$ in \cref{prop:step2-induction}, one can verify that
    \begin{equation*}
        \frac{\frobnorm{\Xlr}}{\Delta_{\infty}} = \frac{\sigma_{\min}(\Xlr)}{\sqrt{d}\Klr \log(\dhid)}.
    \end{equation*}
    Using the bound on $\Xsmall$ from \cref{eq:Xsmall sufficiently small-appendix}, we see that
    \begin{equation*}
        \opnorm{\Xsmall} \lesssim \gamma \cdot \frac{\sigma_{\min}(\Xlr)}{L^2 \Klr^3 \log(\dhid)} \sqrt{\frac{m}{d(s-r)}}
        = 
        \frac{\gamma}{L \Klr^2}
        \cdot
        \frac{\frobnorm{\Xlr}}{ \Delta_{\infty} L} \sqrt{\frac{m}{s-r}}.
    \end{equation*}
    Combining this with \eqref{eq:Phismall bound phase II} and $\lambda = \gamma \sigma_{\min}^2(\Xlr) \sqrt{\frac{\din}{\dout}}$ from \cref{eq:main-thm-assumptions} gives us
    \begin{equation*}
        \frobnorm{\Phismall(t)}
        \lesssim  \frac{\gamma \frobnorm{\Xlr}}{L \Klr^2}
        = 
        \frac{\lambda\frobnorm{\Xlr}}{L \opnorm{X}^2}
        \sqrt{\frac{\dout}{\din}}
    \end{equation*}
    as desired.
\end{proof}
\begin{lemma}\label{lem:phase2:proof:e(t)}
    Fix $\tau \leq t \leq T$ and suppose that $\cA(t)$ and $\cB(t)$ hold.
    Then
    \begin{equation}
        \frobnorm{\Elr(t)}
        \leq   \frac{ \eta \sigma_{\min}^2(\Xlr)\frobnorm{\Phi(t)}}{64 \din}  .
    \end{equation}
\end{lemma}
\begin{proof}
Similar to \eqref{eq:E sum bound} in the proof of \cref{lemma:Bt-implies-bound-on-E}, we use $\cB(t)$ to see that 
\begin{equation}
\label{eq:E sum bound II}
    \frobnorm{E(t)} 
    \lesssim  \frac{\eta L \opnorm{X}^2 \frobnorm{\Phi(t)}}{\din}
    \sum_{\ell = 1}^{L - 1} \left( O\left(
       \frac{\eta L^{\frac{3}{2}} \opnorm{X} \frobnorm{\Phi(t)}}{\sqrt{\din d_w}}
     \right)\right)^{\ell}.
\end{equation}
From $\cA(t)$ and \cref{lemma:Phismall bound phase 2}, we get that
\begin{equation*}
    \frobnorm{\Phi(t)} 
    \le \frobnorm{\Philr(t)} + \frobnorm{\Phismall(t)}
    \lesssim \frobnorm{\Xlr}.
\end{equation*}  
Using the bounds $\eta \leq \nicefrac{\din}{L \sigma_{\max}^2(X)})$ from \cref{eq:main-thm-assumptions}, it follows that
\begin{align*}
    \frac{\eta L^{3/2} \opnorm{X} \frobnorm{\Phi(t)}}{\sqrt{\din \dhid}}
    \lesssim
     \frac{L^{1/2} \sqrt{\din} \frobnorm{\Xlr}}{\sqrt{\dhid} \sigma_{\max}(X)}
    = \sqrt{\frac{L \din \sr(\Xlr)}{\dhid}}.
\end{align*}
Let $c_{\mathsf{e}}$ be a small positive constant to be chosen later. 
By choosing $\dhid \gtrsim \sfrac{L^3 \din \sr(\Xlr) \Klr^4}{c_{\mathsf{e}}^2}$ and applying the bound on geometric sums from \cref{lemma:truncated-geometric-series}, we may therefore ensure that
\begin{equation}
\label{eq:bound of sum for E}
    \sum_{\ell = 1}^{L - 1} \left( O\left(
       \frac{\eta L^{\frac{3}{2}} \opnorm{X} \frobnorm{\Phi(t)}}{\sqrt{\din d_w}}
     \right)\right)^{\ell}
     \lesssim \frac{c_{\mathsf{e}}}{L \Klr^2}.
\end{equation}
Thus by combining \cref{eq:E sum bound II,eq:bound of sum for E} and choosing $c_{\mathsf{e}}$ sufficiently small, we get
\begin{equation*}
    \frobnorm{\Elr(t)}
    \leq   \frac{ \eta \sigma_{\min}^2(\Xlr)\frobnorm{\Phi(t)}}{64 \din}
\end{equation*}
as desired.
\end{proof}

\section{Supplementary Lemmas for \texorpdfstring{\cref{cor:robustness}}{robustness corollary}}
\begin{lemma} \label{lem:init off subspace error lambda 0}
 With probability of at least $1-\exp{(-\Omega((m-s)^2))} - \exp{(-\Omega(d))}$ we have
      \begin{equation}
        \norm{W_{L:1}(0) P_{\range(Y)}^{\perp}\epsilon}
        \gtrsim \sigma \sqrt{\frac{\dout(\din - s)}{\din}}.
    \end{equation}
\end{lemma}
\begin{proof}
 Let us define $\hat{W}_{L:1}(0)=\dhid^{\frac{L-1}{2}} \din^{\frac{1}{2}}  W_{L:1}(0)$, the non-normalized version of the weight matrix product.
 Considering $P_{\range(Y)}^{\perp} \epsilon$ as a fixed vector and noticing that $\hat{W}_L,...,\hat{W}_1$ are all matrices with i.i.d. Gaussian elements and $ \dhid \gtrsim L \dout$, we get by \cref{lemma:wide-gaussian-prod-tail} and the same argument about the failure probability as in \cref{eq:fail prob} that
    \begin{equation}
        \prob{\left|\norm{\hat{W}_{L:1}(0)P_{\range(Y)}^{\perp} \epsilon}^2 - \dout \cdot \dhid^{L - 1} \norm{P_{\range(Y)}^{\perp} \epsilon}^2\right|
        \geq \frac{1}{10} \dout \cdot \dhid^{L-1} \norm{y}^2}
        \leq \exp(-\Omega(\dout)).
    \end{equation}
    Rewriting the above, we have
    \[
     \norm{\dhid^{-\frac{L-1}{2}} \din^{-\frac{1}{2}}  \hat{W}_{L:1}(0) P_{\range(Y)}^{\perp}\epsilon} \gtrsim \dhid^{-\frac{L-1}{2}} \din^{-\frac{1}{2}}   \dhid^{\frac{L-1}{2}} \dout^{\frac{1}{2}}  \norm{P_{\range(Y)}^{\perp} \epsilon} = \sqrt{\frac{d}{m}}   \norm{P_{\range(Y)}^{\perp} \epsilon}
    \] with probability of at least $1-\exp{(-\Omega(d))}$.
    Further we obtain, 
    \begin{equation}
        \label{eq:projected-noise}
        \norm{P_{\range(Y)}^{\perp} \epsilon}
        \gtrsim \sigma \sqrt{\din - s}, \quad
        \text{with probability at least $1 - \exp\left(-\Omega((\din - s)^2)\right)$.}
    \end{equation}
    To see~\eqref{eq:projected-noise}, let $V_{\perp} \in O(\din, \din - s)$ be a matrix whose columns span $\range(Y)^{\perp}$
    such that $P_{\range(Y)}^{\perp} = V_{\perp} V_{\perp}^{\T}$.
    By orthogonal invariance of the Gaussian distribution,
    \[
        V_{\perp}^{\T} \epsilon \overset{(d)}{=}
        \cN(0, \sigma^2 I_{\din - s}).
    \]
    Moreover, by orthogonal invariance of the Euclidean norm,
    \[
        \norm{P_{\range(Y)}^{\perp} \epsilon} =
        \norm{V_{\perp}^{\T} \epsilon}.
     \]
     Combining the two preceding displays with~\cite[Theorem 3.1.1]{Ver18} yields the inequality~\eqref{eq:projected-noise}, which completes the proof.
\end{proof}
\begin{lemma}[Gaussian concentration for $\|M\varepsilon\|$]
\label{lem:gaussian concentration}
Let $M\in\mathbb{R}^{m\times n}$ and $\varepsilon\sim\mathcal{N}(0,\sigma^2 I_n)$. 
  Then for any $\alpha \in (0,1)$ it holds with probability of at least $1-\alpha$ that
  \begin{align}\label{lem:gaussian-concentration-bound}
\norm{M\varepsilon} \lesssim
\sigma\left(1+\sqrt{\log(\sfrac{1}{\alpha})}\right)\frobnorm{M}.
\end{align}
\end{lemma}
\begin{proof}
For $\varepsilon\sim\mathcal{N}(0,\sigma^2 I_n)$ write $\varepsilon=\sigma X$ with $X\sim\mathcal{N}(0,I_n)$. Then
\[
\norm{M\varepsilon}=\sigma\norm{MX}.
\]
Define $f:\R^n \rightarrow \R$ by $f(x) = \norm{Mx}$.
For any $x,y\in\mathbb{R}^n$,
\[
\left|\norm{Mx}-\norm{My}\right| \le \norm{M(x-y)} \le \opnorm{M}\norm{x-y} .
\] 
Therefore $f$ is $L$-Lipschitz with Lipschitz constant $L=\opnorm{M}$.
From Theorem 5.2.2 in \cite{Ver18} it follows that if $f:\mathbb{R}^n\to\mathbb{R}$ is $L$-Lipschitz and $X\sim\mathcal{N}(0,I_n)$ then for all $t\ge0$,
\begin{equation}
    \prob{|f(X)-\mathbb{E}f(X)|\ge t}\le \exp\left(-\frac{t^2}{c^2L^2}\right)
\end{equation}
for some constant $c > 0$.
Applying this with $f(x)=\norm{Mx}$, $L=\opnorm{M}$, $t =  c \opnorm{M} \sqrt{\log(\sfrac{1}{\alpha})}$, we get that
\begin{equation*}
    \norm{MX} 
    \lesssim \mathbb{E}[\norm{MX}] + \opnorm{M} \sqrt{\log(\sfrac{1}{\alpha})}
    \le \mathbb{E}[\norm{MX}] + \frobnorm{M} \sqrt{\log(\sfrac{1}{\alpha})}
\end{equation*}
with probability at least $1-\alpha$.
Using Jensen's inequality and the expected value of a quadratic form, we see that 
\begin{equation*}
    \mathbb{E}[\norm{MX}]^2 \le \mathbb{E}[\norm{MX}^2] = \frobnorm{M}^2.
\end{equation*}
Therefore,
\begin{equation*}
    \norm{M\varepsilon} 
    = \sigma \norm{MX}
    \lesssim \sigma(1+\sqrt{\log(\sfrac{1}{\alpha})})\frobnorm{M}
\end{equation*}
with probability $1-\alpha$.
\end{proof}
\section{Auxiliary results}
In this section, we state and prove results used to prove the main result \cref{thm:mainresult-formal} or mentioned in the introduction. We start with a result showing that a global minimizer solution of the regularized optimization problem is zero on the orthogonal complement of the image.
\begin{lemma}\label{lem:robustsolutionofoptproblem}
    Suppose $f_{\{W_{\ell}\}_{\ell=1}^L}$ is a global minimizer of the regularized optimization problem~\eqref{eq:l2regprob1}.
    Then $f_{\{W_{\ell}\}_{\ell=1}^L}$ satisfies $W_1P_{\range(Y)}^{\perp}=0$, where $P_{\range(Y)}^{\perp}$ is the projection onto the orthogonal complement of $\range(Y)$.
\end{lemma}
\begin{proof}
Suppose that $f_{\{W_{\ell}\}_{\ell=1}^L}$ is a minimizer with $W_1P_{\range(Y)}^{\perp} \neq 0$. Then consider 
the neural network that coincides with $f_{\{W_{\ell}\}_{\ell=1}^L}$ except that its first-layer weights are right-multiplied by $P_{\range(Y)}$. We denote this new neural network by 
$f_{P_{\range(Y)}}$.
We have
\begin{align*}
\frobnorm{f_{P_{\range(Y)}}(Y)-X} 
= \frobnorm{f_{\{W_{\ell}\}_{\ell=1}^L}(P_{\range(Y)}Y)-X} 
= \frobnorm{f_{\{W_{\ell}\}_{\ell=1}^L}(Y)-X}. 
\end{align*}
Hence the first term in the objective in~\eqref{eq:l2regprob1} is the same for $f_{\{W_{\ell}\}_{\ell=1}^L}$ and $f_{P_{\range(Y)}}$.
By the Pythagorean theorem, we have that
\begin{align*}
    \frobnorm{W_1}^2 &=
    \frobnorm{W_1 P_{\range(Y)} }^2 +
    \frobnorm{W_1 P_{\range(Y)}^{\perp}}^2
    > \frobnorm{W_1 P_{\range(Y)} }^2
\end{align*}
since $W_1P_{\range(Y)}^{\perp} \neq 0$ by
assumption. Thus the regularization term in the objective~\eqref{eq:l2regprob1} is strictly larger for $f_{\{W_{\ell}\}_{\ell=1}^L}$ than for $f_{P_{\range(Y)}}$. Therefore $f_{\{W_{\ell}\}_{\ell=1}^L}$ cannot be the minimal-norm solution.
\end{proof}
The following lemma is a generalization of Lemma 6.1 from \cite{du2019width}.
\begin{lemma}
    \label{lemma:wide-gaussian-prod-tail}
    Let $A_1, A_2, \dots, A_{q}$ have i.i.d. Gaussian elements with $A_{i} \in \Rbb^{n_{i} \times n_{i-1}}$, $n_{0} = n$,
	and $n_{i} \gtrsim q$. Then
    \begin{align}
        \prob{\abs[\big]{\norm{A_{q} \cdots A_1 y}^2 - \norm{y}^2 \prod_{i=1}^q n_i}
        \geq 0.1 \norm{y}^2 \prod_{i=1}^q n_i}
        &\leq \exp\left(-\Omega\left(\frac{1}{\sum_{ i = 1 }^q n_i^{-1}}\right) \right),
        \label{eq:tail-bound-prod}
    \end{align}
    where $y$ is any fixed vector.
\end{lemma}
\begin{proof}
    Note that for any $A_{i}$, we have
	\begin{align*}
		\norm{A_{i} y}^2 & = \sum_{j = 1}^{n_i} \ip{(A_{i})_{j, :}, y}^2                                      \\
		                 & \overset{(d)}{=} \sum_{j = 1}^{n_i} \norm{y}^2 g_{i}^2 \qquad (g_i \sim \cN(0, 1)) \\
		                 & \overset{(d)}{=} \norm{y}^2 Z_{i},
	\end{align*}
	where $Z_i \sim \chi^2_{n_i}$, a $\chi^2$-random variable with $n_i$ degrees of freedom. As a result,
	\[
		\expec{\norm{A_i y}^2} = \norm{y}^2 \expec{Z_{i}} = \norm{y}^2 \cdot n_{i}.
	\]
	Moreover, since $A_{1}, \dots, A_{q}$ are independent, we have
	\begin{align*}
		\expec{\norm{A_{q} \dots A_1 y}^2} & =
		\expec{\expec{\norm{A_{q} (A_{q-1} \dots A_1 y)}^2 \mid A_{1}, \dots, A_{q-1}}} \\
		                                   & =
		n_{q} \expec{\norm{A_{q-1} \dots A_{1} y}^2}                                    \\
		                                   & =
		n_{q} \expec{\expec{\norm{A_{q-1} \dots A_1 y}^2 \mid A_{1}, \dots, A_{q-2}}}   \\
		                                   & =
		n_{q} \cdot n_{q-1} \expec{\norm{A_{q-2} \dots A_1 y}^2}                        \\
		                                   & = \dots                                    \\
		                                   & = \prod_{j = 1}^q n_{i} \cdot \norm{y}^2,
	\end{align*}
	by iterating the above construction.

	We now prove~\cref{eq:tail-bound-prod}. Let $\norm{y} = 1$ for simplicity;
	then $\norm{A_q \dots A_1 y}^2 \sim Z_{q} Z_{q-1} \dots Z_{1}$, where $Z_{i} \sim \chi^2_{n_i}$.
	The moments of a random variable $X \sim \chi^2_{k}$ satisfy
	\begin{align*}
		\mathbb{E}[X^{\lambda}] & = \frac{2^{\lambda} \Gamma(\frac{k}{2} + \lambda)}{\Gamma(\frac{k}{2})}
		=
		\frac{2^{\lambda} \sqrt{\frac{4 \pi}{k + 2 \lambda}} \left(\frac{k + 2\lambda}{2e}\right)^{\frac{k}{2} + \lambda}}{
		\sqrt{\frac{4 \pi}{k}} \left(\frac{k}{2e}\right)^{\frac{k}{2}}
		} \left(1 + O(1 / k)\right),
	\end{align*}
	for all $\lambda > -k/2$, with the second equality furnished by a Stirling approximation.
	Following Eq. (20) in \cite{du2019width}, we obtain the following upper bound:
	\begin{equation}
		\mathbb{E}[X^{\lambda}] \leq
		\exp\left(
		\frac{2 \lambda^2}{k} - \frac{1}{2} \log\left(1 + \frac{2\lambda}{k}\right)
		+ \lambda \log k
		\right) \cdot \left(1 + O\left(\frac{1}{k}\right)\right),
		\quad \forall \lambda \geq -\frac{k}{4}.
		\label{eq:stirling-ub-chi2}
	\end{equation}
	To bound the upper tail in~\cref{eq:tail-bound-prod}, we argue that for any $\lambda > 0$ and any small constant $c>0$,
	\begin{align}
		  \prob{Z_q \dots Z_1 \geq \exp(c) \prod_{i=1}^q n_i}                                                      
		 \leq
		\exp\left( -\lambda c + 2 \lambda^2 \sum_{i = 1}^q \frac{1}{n_i}\right)
        \prod_{j=1}^q\left(1 + O\left(\frac{1}{n_i}\right)\right),
		\label{eq:Z-chernoff-ub}
	\end{align}
    similar to Appendix C in \cite{du2019width} but replacing $m^q$ by $\prod_{i=1}^q n_i$ and adjusting the probability bound accordingly.
	Under our assumption that $n_i \gtrsim q$, we see that
	\begin{align}
		\prod_{j=1}^q \left(1 + O\left(\frac{1}{n_i}\right)\right)  
        \lesssim
		\left(1 + \frac{O(1)}{q}\right)^q                   
        \le O(1)
		\label{eq:exp-defn-ub}
	\end{align}
	using the definition of the exponential. 
    Optimizing the term $(-\lambda c + 2 \lambda^2 \sum_{i = 1}^q \frac{1}{n_i})$ over $\lambda \geq 0$ yields
    \[
        \lambda_{\star} = \frac{c}{4 \cdot \sum_{i = 1}^q \frac{1}{n_i}}.
    \]
    Plugging in value of $\lambda_{\star}$ 
    leads to
    \begin{align*}
         -\lambda_{\star} c + 2 \lambda_{\star}^2
        \sum_{i = 1}^q \frac{1}{n_i} =
        -\frac{c^2}{4} \frac{1}{\sum_{i = 1}^q n_i^{-1}} +
        \frac{c^2}{8} \cdot
        \frac{\sum_{i = 1}^q \frac{1}{n_i}}{
        \left(\sum_{i = 1}^q \frac{1}{n_i}\right)^2
        } =
        -\frac{c^2}{8} \cdot \frac{1}{\sum_{i = 1}^q n_i^{-1}}.
    \end{align*}
    Choosing $c = \log(1.1)$ completes the proof.
    
    We now derive the lower bound in~\cref{eq:tail-bound-prod}.
    Given $-\nicefrac{n_i}{4} <\lambda < 0$ for all $i \in \{1,...,q\}$, we have, analogous to the proof of Lemma 6.1 in \cite{du2019width} but replacing $m^q$ by $\prod_{i=1}^q n_i$, that 
	\begin{align*}
		 \prob{Z_{q} \dots Z_1 \leq \exp(-c) \prod_{i=1}^q n_i} 
		\leq \exp\left(\lambda c + 2 \lambda^2 \sum_{i = 1}^q 
        \frac{1}{n_i} - 2 \lambda \sum_{i=1}^q \frac{1}{n_i}\right)
	\end{align*}
    Setting $\lambda = -\frac{c}{4 \sum_{i=1}^q n_i^{-1}}$ yields
    \[
        -\frac{c^2}{4 \sum_{i=1}^q n_i^{-1}} +
        \frac{c^2}{8} \frac{\sum_{i=1}^q n_i^{-1}}{(\sum_{i=1}^q n_i^{-1})^2}
        +\frac{c}{2}
        =
        -\frac{c^2}{4 \sum_{i = 1}^q n_i^{-1}} + \frac{c}{2}.
    \]
    Setting $c = -\log(0.9)$ completes the proof.
\end{proof}
\begin{lemma}
    \label{lem:RIP bounds on sv}
    Given a matrix $M \in \R^{d \times s}$ whose columns lie in $\range(R)$ and a measurement matrix $A \in \R^{m \times d}$ that satisfies Assumption~\ref{assumption:rip}, we have
    \begin{equation}
        \sqrt{1-\delta} \sigma_{j}(M) \le \sigma_{j}(AM) \le \sqrt{1+\delta} \sigma_{j}(M) ~~\forall j = 1, \ldots, s.
    \end{equation}
    In particular, $\sigma_j(M) = 0$ if and only if $\sigma_j(AM) = 0$. As a consequence, we have
    \begin{equation}
        \sqrt{1-\delta} \sigma_{\min}(M) \le \sigma_{\min}(AM) \le \sqrt{1+\delta} \sigma_{\min}(M)
    \end{equation}
    where we use $\sigma_{\min}$ to denote the smallest nonzero singular value.
\end{lemma}
\begin{proof}
    By the min-max principle, the $j$-th largest singular value of $AM$ is
    \begin{align*}
        \sigma_j(AM) 
        = \max_{S: \dim(S) = j}\min_{\substack{v \in S \\ \norm{v}=1}} \norm{AMv}.
    \end{align*}
    Because the columns of $M$ lie in $\range(R)$, it follows that $Mv$ is in $\range(R)$ as well. By Assumption~\ref{assumption:rip}, it follows that
    \begin{align*}
        \sigma_j(AM) 
        \le \sqrt{1+\delta} \max_{S: \dim(S) = j}\min_{\substack{v \in S \\ \norm{v}=1}} \norm{Mv}
        = \sqrt{1+\delta} \sigma_j(M).
    \end{align*}
    A similar argument shows that 
    $
        \sigma_j(AM) 
        \ge \sqrt{1-\delta} \sigma_j(M).
    $
\end{proof}

\begin{lemma}\label{lem:orcal-pseudo-inverse}
The unique solution of the minimization problem 
\begin{equation}\label{eq:oracle-app}
    \Worc = \arg\min_W \|W\|_F^2 \text{ subject to } WY=X
\end{equation} is  given by $\Worc = XY^\dagger = R(AR)^\dagger$.
\end{lemma}
\begin{proof}
    All solutions of the linear equation $WY=X$ are of the form
    \begin{equation}
        W=XY^{\dagger} + M(I-YY^{\dagger})
    \end{equation}
    with $M \in \mathbb{R}^{d \times m}$ an arbitrary matrix.
Given the above and noticing that the two summands are a mapping onto the row space of $Y$ and a mapping onto the orthogonal complement of that row space. Therefore the two mappings are orthogonal two each other and  we can rewrite the minimization problem as
\begin{equation}
\arg\min_W \frobnorm{W}^2=\arg\min_M; \frobnorm{XY^{\dagger}+ M(I-YY^{\dagger})}^2 = \frobnorm{XY^{\dagger}}^2+\arg\min_M\frobnorm{M(I-YY^{\dagger}}^2
\end{equation}
Minimizing over $M$ we obtain $M=0$ and therefore the minimizer is $XY^{\dagger}$ as claimed. Plugging in the definitions of $Y$ and $X$ we obtain 
\[
XY^{\dagger} = RZ(ARZ)^{\dagger}= RZZ^{\dagger}(AR)^{\dagger} = R(AR)^{\dagger}
\]
since $Z$ has full row rank.
\end{proof}

\begin{lemma}\label{lem:oracle robustness}
        Given a test data point $(x, y)$ satisfying 
        $y=Ax+\epsilon$, where $x \in \range(R)$ and
        $\epsilon \sim \mathcal{N}(0, \sigma^2 I_m)$, the oracle mapping $\Worc$ satisfies
        \begin{equation}\label{eq:robustnessoracle-supplement}
            \norm{\Worc y - x} \lesssim \sigma \sqrt{s}
        \end{equation}
        with probability at least $1 - \exp\left(-\Omega(s^2)\right)$.
    \end{lemma}
    \begin{proof}
    Recall that 
    $x = Rz$ for some $z \in \Rbb^{s}$.
    By \cref{lem:orcal-pseudo-inverse}, we have
    \begin{align*}
        \Worc y 
        &= R(AR)^{\dag} (ARz+\epsilon) \\
        &= R (AR)^{\dag} (AR) z + R(AR)^{\dag} \epsilon \\
        &= Rz + R(AR)^{\dag} \epsilon \\
        &= x + R(AR)^{\dag} \epsilon.
    \end{align*}
    The third equality 
    follows from 
    the 
    full column rankness of 
    $AR$,
    which makes its pseudoinverse a one-sided inverse. Consequently,
    \begin{align*}
        \norm{\Worc y - x} &=
        \norm{R(AR)^{\dag} \epsilon} =
        \norm{(AR)^{\dag} \epsilon},
    \end{align*}
    since $R$ is a matrix with orthogonal columns.
    We now write
    \[
        AR = \bar{U} \bar{\Sigma} \bar{V}^{\T}, \quad \text{where} \;\;
        \bar{U} \in O(m, s), \;
        \bar{V} \in O(s), \;
        \sqrt{1 - \delta} \leq \bar{\Sigma}_{ii} \leq \sqrt{1 + \delta},
    \]
    for the economic SVD of $AR$, where the bounds on the singular
    values follow from Assumption~\ref{assumption:rip} (see \cref{lem:RIP bounds on sv}). 
    In particular,
    \begin{align*}
        \norm{(AR)^{\dag} \epsilon} &=
        \norm{\bar{V} \bar{\Sigma}^{-1} \bar{U}^{\T} \epsilon}
        \leq
        \frac{1}{\sigma_{\min}(\bar{\Sigma})}
        \norm{\bar{U}^{\T} \epsilon}
        \lesssim
        \norm{\bar{U}^{\T} \epsilon},
    \end{align*}
    treating $\delta$ as a constant in the last inequality.
    Finally, by standard properties of the multivariate normal distribution,
    \[
        \bar{U}^{\T} \epsilon \sim \cN(0, \sigma^2 I_{s}) \implies
        \norm{\bar{U}^{\T} \epsilon} \lesssim
        \sigma \sqrt{s},
    \]
    with probability at least $1 - \exp\left(-\Omega(s^2)\right)$~\cite[Theorem 3.1.1]{Ver18}.
    \end{proof}

\begin{theorem}[Weierstrass]
	\label{thm:weierstrass}
	The following inequality holds:
	\begin{equation}
		1 - \sum_{i=1}^n w_{i} x_i \leq \prod_{i = 1}^n \left(1 - x_i\right)^{w_i}, \quad
		\text{for all $x \in [0, 1]$ and $w_i \geq 1$}.
		\label{eq:weierstrass}
	\end{equation}
\end{theorem}
\begin{proof}
    We prove the inequality by induction on the number of
    terms. For the base case $n = 1$, consider the function
    \[
        h(w) = (1 - x_1)^w - (1 - w x_1), \quad
        \text{with} \;\;
        h'(w) = (1 - x_1)^{w} \log(1 - x_1) + x_1
    \]
    Clearly $h(1) = 0$, so it suffices to show $h$ is increasing on $[1, \infty)$. Starting from the inequality
    \(
        \log(1 - x_1) \geq \frac{x_1}{x_1 - 1},
    \)
    we have
    \begin{align*}
        h'(w) &\geq
        \frac{x_1 (1 - x_1)^{w}}{x_1 - 1} + x_1 \\
        &=
        \frac{x_1 (1 - x_1)^{w} + x_1 (x_1 - 1)}{x_1 - 1} \\
        &=
        \frac{x_1 \left[(1 - x_1) - (1 - x_1)^w\right]}{1 - x_1} \\
        &\geq 0, \quad \text{for all $w \geq 1$,}
    \end{align*}
    since $(1 - x_1) \in (0, 1)$. This proves the claim
    for $n = 1$.

    Now suppose the claim holds up to some $n \in \mathbb{N}$. We have
    \begin{align*}
        \prod_{j = 1}^{n+1} (1 - x_j)^{w_j} &=
        (1 - x_{n+1})^{w_{n+1}} \prod_{j = 1}^n (1 - x_j)^{w_j} \\
        &\geq
        (1 - w_{n+1} x_{n+1}) \prod_{j = 1}^{n}
        (1 - x_j)^{w_j} \\
        &\geq
        (1 - w_{n+1} x_{n+1}) \left(1 - \sum_{j = 1}^n w_j x_j \right) \\
        &=
        1 - \sum_{j = 1}^{n+1} w_j x_j +
        w_{n+1} x_{n+1} \cdot \sum_{j = 1}^{n} w_j x_j \\
        &\geq
        1 - \sum_{j = 1}^{n+1} w_j x_j,
    \end{align*}
    where the first inequality follows from the base
    case, the second inequality follows by the inductive
    hypothesis and the last inequality follows from
    nonnegativity of $\set{w_j}_{j \geq 1}$ and $\set{x_j}_{j \geq 1}$. This completes the proof.
\end{proof}
\begin{lemma}
	\label{lemma:one-minus-folded-product}
	Under the assumptions of \cref{thm:mainresult-formal}, we have that
	\begin{equation}
		|1 - \call|
		\leq \frac{2 \eta \lambda}{m}.
		\label{eq:one-minus-folded-product}
	\end{equation}
\end{lemma}
\begin{proof}
    Recall from \cref{eq:c-prod-i} that
    \begin{equation*}
        \call := \prod_{i = 1}^{L} \left(1 - \frac{\eta \lambda}{d_i}\right).
    \end{equation*} 
	Since $0 < \call < 1$, we have
	$|1 - \call| =
		1 - \prod_{i = 1}^{L}
		\left(1 - \frac{\eta \lambda}{d_i}\right)$.
	Now, let $x_i := \frac{\eta \lambda}{d_i}$ and $w_i := 1$ for $i = 1, \dots, L$. From~\cref{thm:weierstrass},
	it follows that
	\begin{align*}
		1 - \prod_{i = 1}^L
		\left(1 - \frac{\eta \lambda}{d_i}\right) & \leq
		\sum_{i = 1}^L \frac{\eta \lambda}{d_i} =
		\frac{(L - 1) \eta \lambda}{d_{w}} +
		\frac{\eta \lambda}{m}
		\leq \frac{2 \eta \lambda}{m},
	\end{align*}
    under the assumption that $\dhid \ge m(L-1)$.
\end{proof}
\begin{lemma}
    \label{lemma:cprod-i-lb}
    We have that
    $\frac{1}{4} \leq \cprod{i} \leq 1$ for all $1 \leq i \leq L$.
\end{lemma}
\begin{proof}
    Recall from \cref{eq:c-prod-i} that 
    \begin{equation*}
        \cprod{i} = \prod_{\substack{j = 1\\j \neq i}}^{L} \left(1 - \frac{\eta \lambda}{d_j} \right).
    \end{equation*}
    The bound $\cprod{i} \leq 1$ is immediate since each term in the product is bounded above by 1.
    From~\cref{thm:weierstrass}, we have 
    \begin{equation*}
        \cprod{i} = \prod_{\substack{j = 1\\j \neq i}}^{L} \left(1 - \frac{\eta \lambda}{d_j}\right) \geq
    1 - \sum_{\substack{j = 1\\j \neq i}}^{L} \frac{\eta \lambda}{d_j}. 
    \end{equation*}
    Moreover, we see that
    \begin{equation}
        \sum_{\substack{j = 2\\j \neq i}}^{L} \frac{1}{d_j} \le \frac{L}{\dhid} \le \frac{1}{\din L}
    \end{equation}
    because $\dhid \ge L^2 \din$.
    Adding in the $j=1$ term, multiplying by $\eta\lambda$, and using the assumption from \cref{eq:main-thm-assumptions} that $\eta \leq \frac{m}{L \sigma^2_{\max}(X)}$, we see that
    \begin{equation}
        \sum_{\substack{j = 1\\j \neq i}}^{L} \frac{\eta \lambda}{d_j} 
        \le \frac{\eta\lambda}{m}\left(1 + \frac{1}{L}\right)
        \le \frac{\lambda}{L \sigma^2_{\max}(X)} \left(1 + \frac{1}{L}\right).
    \end{equation}
    Notice that $L \mapsto \frac{1}{L} \left(1 + \frac{1}{L}\right)$ is
    decreasing in $L$ and equal to $\frac{3}{4}$ for $L = 2$. 
    Additionally, our choice of $\lambda$ in in \cref{eq:main-thm-assumptions} satisfies $\lambda \le \sigma^2_{\max}(X)$. 
    It follows that
    \begin{equation}
        \sum_{\substack{j = 1\\j \neq i}}^{L} \frac{\eta \lambda}{d_j} 
        \le \frac{\lambda}{L \sigma^2_{\max}(X)} \left(1 + \frac{1}{L}\right)
        \le \frac{3}{4}.
    \end{equation}
    Therefore, 
    \begin{equation}
        \cprod{i}
        \ge 1 - \sum_{\substack{j = 1\\j \neq i}}^{L} \frac{\eta \lambda}{d_j}
        \ge \frac{1}{4}
    \end{equation}
    as desired.
\end{proof}
\begin{lemma}
    \label{lemma:truncated-geometric-series}
    For any $0 \le \alpha \leq \frac{1}{2}$ and $j\le k \in \mathbb{N}$, it holds that
    $
        \sum_{i = j}^k \alpha^{i} 
        \lesssim \alpha^j.
    $
\end{lemma}
\begin{proof}
    The claim follows from the geometric series formula:
    \begin{align*}
       \sum_{i = j}^k \alpha^i &=
       \alpha^j \sum_{i = 0}^{k- j } \alpha^i
       =
       \alpha^j \cdot \frac{1 - \alpha^{k - j + 1}}{1 - \alpha}
       \leq
       2 \alpha^j (1 - \alpha^{k - j + 1})
       \leq 2 \alpha^j,
    \end{align*}
    where the penultimate inequality follows from  $\nicefrac{1}{(1 - \alpha)} \leq 2$.
\end{proof}
\begin{lemma}
\label{lem:xlog1/x le sqrt x}
    For any $x > 0$, $x \log(\nicefrac{1}{x}) \le \sqrt{x}$.
\end{lemma}
\begin{proof}
    Let $f(x) = \sqrt{x} \log(\nicefrac{1}{x}) = -\sqrt{x} \log(x)$. Then 
    \begin{equation*}
    \label{eq:f'}
        f'(x) = -\frac{2 + \log(x)}{2\sqrt{x}}.
    \end{equation*} 
    Notice that $f'(e^{-2}) = 0$. 
    Further, since the denominator of $f'$ is positive, we see that $f'(x) > 0$ when $x < e^{-2}$ and $f'(x) < 0$ when $x > e^{-2}$. Hence the global maximum of $f$ occurs at $x = e^{-2}$ and the maximum value of $f$ is $\nicefrac{2}{e}$. In particular, we have $\sqrt{x} \log(\nicefrac{1}{x}) < 1$. It follows that $x \log(\nicefrac{1}{x}) \le \sqrt{x}$.
\end{proof}

\section{Information on numerics and additional experiments}
\label{sec:appendix numerical description}

\subsection{Expanded versions of main text experiments}
\label{sec:expanded-numerics}
In this section, we provide expanded versions of the numerical experiments from the main manuscript.

\begin{figure}[ht]
    \centering
    \includegraphics[width=\linewidth]{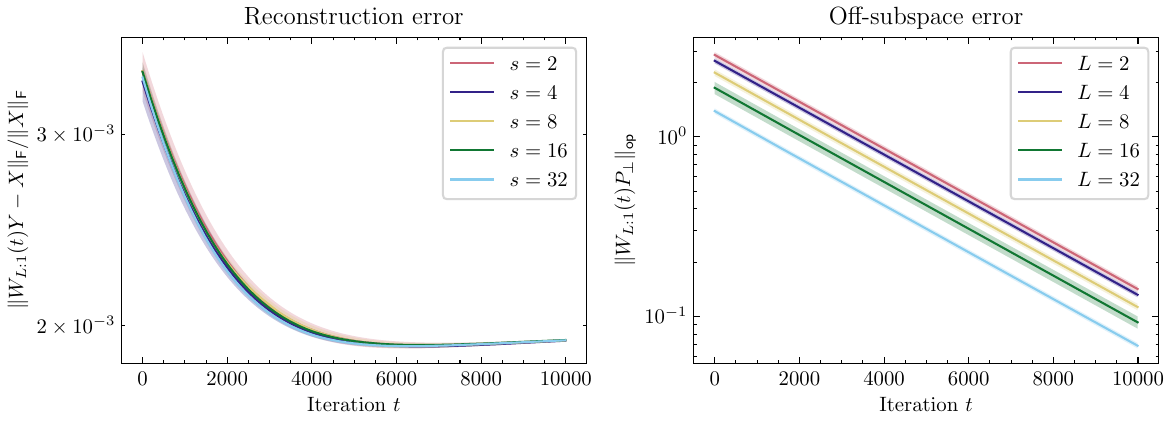}
    \caption{Comparing the training error of a deep linear neural network for data of varying subspace dimensions $s$ using constant stepsize $\eta = \nicefrac{1}{10}$ and
    weight decay $\lambda = 10^{-3}$. The lines are the median over $10$ runs with independently sampled training data and weight initializations. The shaded region indicates one standard deviation around the median. 
    See \cref{sec:impact of s} for details.
    }
    \label{fig:plot_s}
\end{figure}

\subsubsection{Impact of neural network depth \texorpdfstring{$L$}{L}}
\label{sec:subsec:depth}
\begin{figure*}[!ht]
    \centering
    \includegraphics[width=\linewidth]{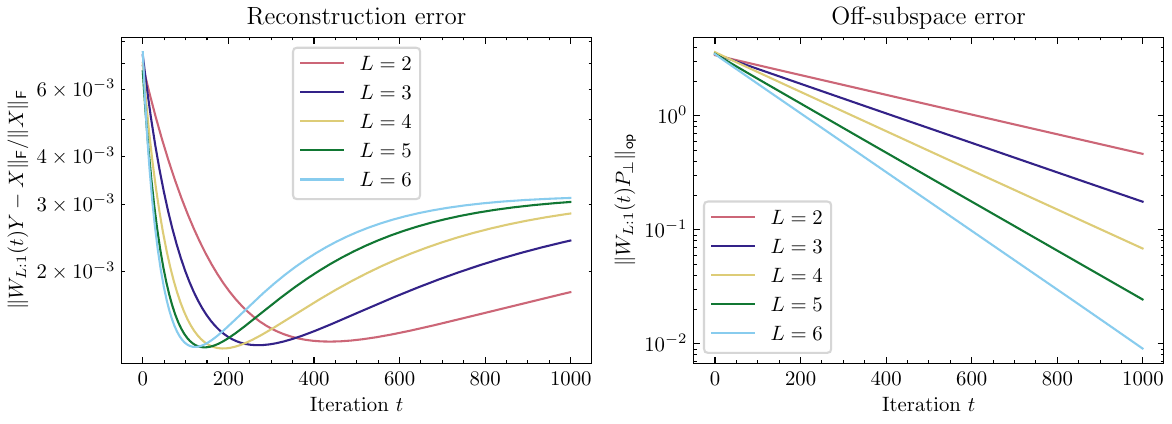}
    \caption{Normalized reconstruction error and off-subspace error for deep linear nets of varying
    depths $L$, trained with gradient descent using
    constant stepsize $\eta = \nicefrac{1}{10}$ and weight decay parameter $\lambda = 10^{-4}$. While the reconstruction error drops to similar levels for
    all depths, larger $L$ confers a clear advantage with respect to the off-subspace error. See~\cref{sec:subsec:depth} for details.}
    \label{fig:errors-by-depth}
\end{figure*}
\begin{figure}[!ht]
    \centering
    \includegraphics[width=\linewidth]{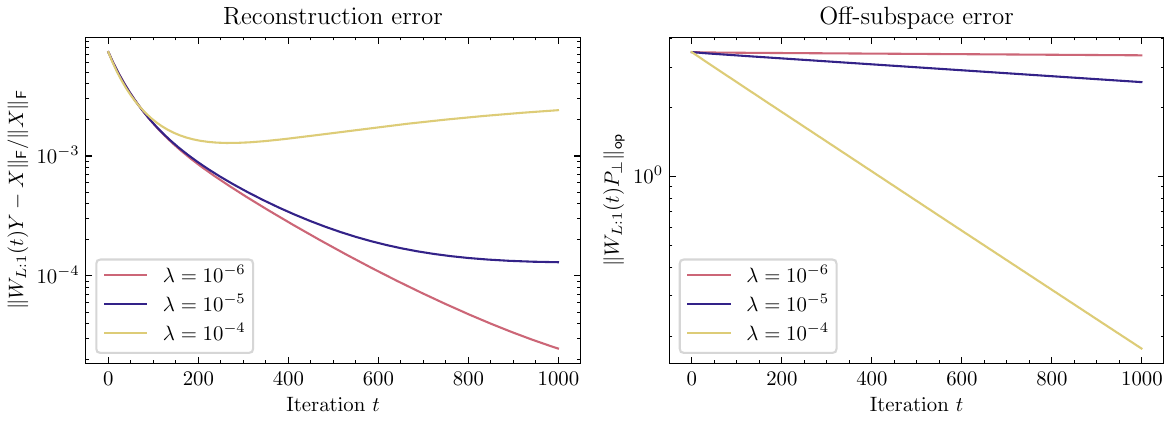}
    \caption{Normalized reconstruction error and off-subspace errors for deep linear nets trained with
    gradient descent with stepsize $\eta = \sfrac{1}{10}$ and varying levels of weight decay $\lambda$. While high levels of weight decay reduce
    the off-subspace error faster, they lead to
    larger reconstruction error. See~\cref{sec:subsec:wd} for details.}
     \label{fig:wd}
\end{figure}
In our next experiment, we examine how the neural network depth, $L$, affects convergence and generalization. We generate a dataset with subspace dimension $s = 4$,
measurement dimension $\din = 32$, signal dimension $\dout = 64$, and $n = 1000$ samples (using perfectly conditioned data; i.e., $\kappa = 1$) and train a deep linear network of width $\dhid = 1000$ using
gradient descent. We use the same stepsize $\eta = 10^{-1}$ and weight
decay parameter $\lambda = 10^{-4}$ across all configurations.

The results for both quantities of interest are depicted in~\cref{fig:errors-by-depth}.
The reconstruction error first drops to similar levels, for all depths, before it starts increasing and
plateauing at roughly $20 \lambda$. However, higher depth $L$ confers a clear advantage with respect to the off-subspace error.

\subsubsection{Impact of latent subspace dimension \texorpdfstring{$s$}{s}} \label{sec:impact of s}
The statement of~\cref{theorem:main-informal} suggests that the size of
the subspace $s$ does not affect the rate of (on-subspace) convergence
or the error achieved after $T$ iterations. To verify this numerically,
we generate several synthetic datasets with varying subspace dimension
$s \in \{2, 4, 8, 16, 32\}$, $\din = 128$, $\dout = 256$ and perfectly conditioned data (i.e., $\kappa = 1$). For each dataset, we train a deep linear network of width $\dhid = 512$ using $\eta = \nicefrac{1}{10}$ and $\lambda = 10^{-3}$ and compute the median reconstruction and off-subspace errors and standard deviation over $10$ independent runs, with each run using $n = 1000$ independently drawn samples. The results, depicted in~\cref{fig:plot_s},
suggest that the errors decay at the same rate; in the case of
the reconstruction error, the differences in magnitude are
negligible, while the off-subspace errors differ by a constant offset across subspace dimensions.

\subsubsection{Impact of weight decay parameter \texorpdfstring{$\lambda$}{lambda}}
\label{sec:subsec:wd}
Our next experiment examines the impact of the weight decay parameter $\lambda$.
We use a similar setup as in~\cref{sec:subsec:depth}, where $s = 4$, $\din = 32$ and $\dout = 64$ with $n = 1000$ samples,
and train neural networks of width $\dhid = 1000$ and depth $L = 3$;
see~\cref{fig:wd}.
As~\cref{theorem:main-informal} suggests, larger weight decay values lead to
larger reconstruction errors (approximately $10 \cdot \lambda$) but faster decaying off-subspace errors.

\subsection{Experimental setup}
\label{sec:subsec:experimental-setup}
\paragraph{Hardware information} We performed our numerical experiments on an internal cluster environment equipped with NVIDIA GPUs running CUDA 12.6.
Experiments were submitted to the
cluster scheduler requesting a single GPU and 8GB of RAM and took between 15 and 45
minutes to complete, depending on the hardware of the node they were scheduled on.

\paragraph{Data generation}
For each experiment shown in \cref{fig:linear,fig:stepsize-sweep,fig:wd,fig:errors-by-depth,fig:plot_s}, we generate the measurement matrix $A$ by sampling a random
Gaussian matrix $G \in \Rbb^{\din \times \dout}$ and setting 
$A := \frac{1}{\sqrt{\din}} G$; such matrices satisfy
Assumption \ref{assumption:rip} with high probability as long as
$\din \gtrsim s \log(\dout)$~\cite{foucart2013invitation}. 
To form the subspace basis matrix $R$, we
calculate the QR factorization of a $\dout \times s$ random Gaussian matrix and keep the orthogonal factor. 
Finally, we generate the signal matrix
$X \in \Rbb^{\dout \times n}$ as $X = RZ$, where $Z
\in \Rbb^{s \times n}$ is a full row-rank matrix of
subspace coefficients. Given a target condition number
$\kappa$ for $X$, we generate $Z$ via its SVD: we sample
the left and right singular factors at random and arrange
its singular values uniformly in the interval $[\frac{1}{\kappa}, 1]$. All our experiments use step sizes that
are covered by our theory but do not necessarily correspond to the value suggested by~\cref{theorem:main-informal} (see Remark~\ref{remark:small-eta}). Similarly, each experiment uses a number of iterations that is sufficiently large but not necessarily equal to $T$. Finally, all weight decay parameters used correspond to a valid
$\gamma \in (0, 1)$, but for the sake of simplicity we
specify $\lambda$ directly.

In the remainder of this section, we present expanded versions of the numerical
experiments from the main text (\cref{sec:expanded-numerics}), as well as additional numerical experiments that examine the sensitivity of the learned mapping to different parameters: the dimension of the latent subspace $s$ (\cref{sec:impact of s}, the neural network depth $L$ (\cref{sec:subsec:depth}),
and the regularization strength $\lambda$ (\cref{sec:subsec:wd}).
All experiments track the regression and ``off-subspace''
errors across $t$:
\[
    \frac{\frobnorm{W_{L:1}(t) Y - X}}{\frobnorm{X}}
    \quad \text{and} \quad
    \opnorm{W_{L:1}(t) P_{\range(Y)}^{\perp}}.
\]
Finally, \cref{sec:uos} formally introduces the union-of-subpaces (UoS) model used
to generate the training and test data for the experiment in~\cref{fig:wd-robustness-nonlinear}, as well as the associated nonlinear neural network architecture used in the same experiment.

\subsection{Union of Subspaces Experiment}\label{sec:uos}
In this subsection, we describe the union-of-subspaces (UoS) model used to generate the
synthetic training data in~\cref{fig:wd-robustness-nonlinear}.
The UoS model stipulates that each vector in the input
data belongs to one of $k$ subspaces. Formally, there exists a
collection $\mathcal{R} := \set{R_1, \dots, R_k}$, where $R_{i} \in O(\dout, s)$,
such that $x^{i} \in \bigcup_{j = 1}^k \range(R_{j})$ for all $i$.
In our experiments, we generate training samples from the union-of-subspaces model as follows:
\begin{itemize}
    \item Sample $Z \in \mathbb{R}^{\dout \times n}$ according to
    the procedure described in~\cref{sec:subsec:experimental-setup}.
    \item For each $j = 1, \dots, k$:
    \begin{enumerate}
        \item Sample a $d \times s$ random Gaussian matrix $R_j^{\mathrm{full}}$.
        \item Set $R_{j}$ equal to the $Q$ factor from the reduced QR factorization of
        $R_{j}^{\mathrm{full}}$.
    \end{enumerate}
    \item For each $i = 1, \dots, n$:
    \begin{enumerate}
        \item Sample $R \sim \mathrm{Unif}(\mathcal{R})$.
        \item Set $X_{:, i} = R Z_{:, i}$.
    \end{enumerate}
\end{itemize}

\subsubsection{Neural network architecture for the union-of-subspaces model}
The inverse mapping for linear inverse problems with data from a union-of-subspaces model is in general nonlinear for $k > 1$ --- as
a result, deep linear networks are not a suitable choice for learning
the inverse mapping.
Nevertheless, it is known that the inverse mapping is approximated to arbitrary accuracy
by a piecewise-linear mapping~(see~\cite{GOW20}), which can be realized as a multi-index model of the form $g(V^{\T} x)$ for
suitable $V$ and vector-valued mapping $g$. Guided by this, we use a
neural network architecture defined as follows:
\begin{equation}
    f_{\{W_{\ell}\}_{\ell=1}^L}(x) =
    W_L \left( W_{L-1} W_{L-2} \cdots W_{1} x\right)_+,
    \label{eq:mixed-linear-relu-network}
\end{equation}
where $W_1, \dots, W_{L}$ are learnable weight matrices and
$[\cdot]_+$ denotes the (elementwise) positive part, equivalent to using a ReLU activation at the $(L-1)^{\text{th}}$ hidden layer.
Indeed, recent results~\cite{parkinson2023linear} suggest that neural
networks of the form~\eqref{eq:mixed-linear-relu-network} are biased
towards multi-index models such as the one sought to approximate the
inverse mapping. Finally, all the networks from~\cref{fig:linear} were trained for $100,000$ iterations with learning rate $\eta = \nicefrac{\din}{2 L \sigma_{\max}^2(X)}$.

\bibliographystyle{unsrtnat}
\bibliography{references}

\end{document}